\documentclass{article} 
\usepackage{iclr2027_conference,times}
\usepackage{comment}

\usepackage{amsmath,amsfonts,bm}

\def\eqref#1{equation~\ref{#1}}
\def\Eqref#1{Equation~\ref{#1}}

\def\1{\bm{1}}

\DeclareMathAlphabet{\mathsfit}{\encodingdefault}{\sfdefault}{m}{sl}
\SetMathAlphabet{\mathsfit}{bold}{\encodingdefault}{\sfdefault}{bx}{n}

\newcommand{\E}{\mathbb{E}}

\usepackage{titletoc}
\usepackage{hyperref}
\usepackage{url}
\usepackage{xcolor}
\usepackage{amsmath}
\usepackage{amssymb}
\usepackage{amsthm}
\usepackage{mathtools}
\usepackage{graphicx}
\usepackage{subcaption}
\usepackage{algorithm}
\usepackage{algpseudocode}
\usepackage{longtable}
\usepackage{booktabs}
\usepackage{array}
\usepackage{enumitem}
\newtheorem{theorem}{Theorem}
\newtheorem{lemma}{Lemma}

\newtheorem{proposition}{Proposition}
\theoremstyle{remark}
\newtheorem{remark}{Remark}

\newcommand{\kl}{\operatorname{kl}}

\providecommand{\E}{\mathbb E}
\providecommand{\1}{\mathbf 1}

\title{To Defer or Not to Defer? 
Online Learning with LLMs and Congested Experts}

\author{
\textbf{Vasileios Saketos}\thanks{Equal contribution.} \,
\textbf{William R\'eveillard}\footnotemark[1] \,
Alexandre Prouti\`ere \,
Sebastian Dalleiger \,
Ming Xiao 
\\[0.6em]
KTH Royal Institute of Technology \\
Stockholm, Sweden \\[0.3em]
\texttt{\{saketos,wilrev,alepro,sdall,mingx\}@kth.se}
}

\iclrfinalcopy 
\begin{document}

\maketitle

\begin{abstract}
Modern AI systems increasingly combine large language models (LLMs) with human experts, using the LLM to handle routine queries while deferring uncertain cases for expert review. When expert capacity is limited, however, deferral decisions must account not only for the reliability of the LLM but also for congestion in the expert queue. We formulate this as an online sequential decision-making problem in which, for each query, the system observes the confidence signal of the LLM's answer and must decide whether to output the answer or defer the query to a capacity-constrained expert. The reliability of the confidence signal is initially unknown and must be learned from expert feedback, which is both selectively observed (only deferred queries are labeled) and delayed by queuing. We propose \texttt{KL-PBD}, a learning algorithm that uses Kullback--Leibler lower confidence bounds on the unknown reliability parameters to guide deferral decisions. For $K$ confidence levels and horizon $T$, \texttt{KL-PBD} achieves $O(K\log T)$ gap-dependent regret when the accuracy at every confidence level is bounded away from the deferral boundary, and $\widetilde{O}(\sqrt{KT})$ worst-case regret. We establish matching lower bounds, showing that these rates are optimal up to logarithmic factors. Finally, we validate our approach on synthetic instances and real-world question-answering datasets, namely GSM8K, NuminaMath, and MedMCQA, using Qwen LLMs.

\end{abstract}

\section{Introduction}

Large language models (LLMs) are increasingly deployed in settings where automated predictions can
be complemented by human expertise, including customer support and healthcare
\citep{huang-etal-2025-enable,feng2024large,yang-etal-2025-elaboration,10.1145/3706598.3713497}.
A natural design is to let the LLM handle routine queries while referring uncertain or high-risk cases
to a human expert. Expert attention, however, is limited: reviewing a query takes time, and queries
sent for review may have to wait behind previously deferred requests. Consequently, a deferral
decision must account not only for whether the LLM is likely to be correct, but also for the congestion
created by sending another query to the expert queue. Determining when the LLM is reliable is itself challenging: the correctness of a particular answer is
not known unless a ground-truth label is obtained, which in our setting requires expert review.
LLMs nevertheless provide signals that may be informative about their reliability: a prominent
example is self-consistency \citep{wang2022self}, which measures agreement across multiple generated
answers. Such signals can be correlated with correctness, but high confidence or agreement does not
guarantee that the answer is correct
\citep{aggarwal-etal-2023-lets,wang-etal-2024-self-consistency}.
Thus, a deferral system must learn how the observed confidence signal relates to the probability that
the LLM's answer is correct.

We are interested in doing so \emph{online}, without requiring a separate task-specific calibration
dataset or retraining the LLM before deployment. Instead, the system learns from expert feedback
generated during operation. This gives rise to a coupled learning-and-control problem with an unusual
feedback structure. If the system accepts an LLM response, it observes no ground-truth label. If it
defers, the label is revealed only after the query has been served by the expert. Feedback is therefore
both \emph{selective} and \emph{delayed}. Moreover, acquiring information is itself costly: every
deferred query consumes expert capacity, enters the queue, and may increase congestion for future
queries. Deferral is therefore simultaneously a prediction decision, an information-acquisition
decision, and a queue-control decision. We formalize this interaction as a finite-horizon online sequential decision problem. At each round,
a query arrives together with the LLM's answer and a confidence level, and the system decides
whether to accept the answer or defer the query to a capacity-constrained expert queue. The
conditional accuracy associated with each confidence level is initially unknown and must be learned
from completed expert reviews. Our goal is to design an online policy that minimizes the expected cumulative cost over the horizon, balancing the risk of incorrect LLM answers, the congestion created by expert review, and the value of deferrals as a source of information for future decisions.

Our contributions are as follows. 

(a) We propose \texttt{KL-PBD}, a queue-aware algorithm that learns the unknown conditional accuracies from delayed expert feedback. The algorithm acts optimally for a pessimistic estimate of the LLM's reliability, favoring deferral, and hence information acquisition, when the available evidence is still limited, while explicitly accounting for the congestion cost of doing so.

(b) We establish near-optimal regret guarantees for \texttt{KL-PBD}. For $K$ confidence levels, the algorithm achieves $O(K\log T/\gamma)$ regret when the accuracy at every confidence level is at least $\gamma$ away from the deferral boundary, and $\widetilde O(\sqrt{KT})$ regret in the worst case.

(c) We derive complementary lower bounds showing that logarithmic regret and inverse-gap
dependence are unavoidable in separated instances, while the worst-case $\sqrt{KT}$ dependence is optimal up to logarithmic factors.

(d) We evaluate \texttt{KL-PBD} on synthetic and real-world question-answering tasks based on GSM8K, NuminaMath, and MedMCQA. We demonstrate the importance of
pessimism under selective feedback and show that explicitly accounting for the expert queue is paramount when expert service is congested (see Fig.~\ref{fig:intro-candidates-row} (right) for an illustrative example).

\section{Related Work}
\paragraph{Learning to defer.} 
Learning to defer studies systems that either make an automated prediction or delegate the decision to a human expert
\citep{madras2018predict,mozannar2020consistent}.
These foundational formulations jointly learn a predictor and a deferral rule, with subsequent work developing calibrated deferral objectives
\citep{verma2022calibrated}.
Post-hoc approaches instead learn when to defer the predictions of a fixed model
\citep{narasimhan2022posthoc}.
These approaches are formulated as offline supervised learning problems, where the deferral policy is learned from a fixed labeled dataset and evaluated through test-time prediction risk.
In our problem, the LLM response rule is fixed, but the system learns a deferral policy online from delayed, selectively observed human feedback generated by its own deferral decisions. Moreover, expert intervention is capacity constrained, making the cost of deferral state dependent rather than fixed. Further supervised and multi-expert formulations are discussed in Appendix~\ref{app:related-deferral-extensions}.

\paragraph{Online learning and selective feedback.}
\citet{reid2024online} study online deferral under a total consultation budget. Their bandit feedback model reveals the reward of the selected action, including automated predictions.
Online selective classification is closer to our setting: labels are observed only when the learner abstains from answering, and  regret guarantees trade off prediction errors against selective label acquisition \citep{gangrade2021online,baharav2025good}.
In our setting, labels are similarly observed only on deferrals, but a deferred query first enters the expert queue, and its label becomes available only once it has been served. Broader active-learning and oracle-query settings are discussed in
Appendix~\ref{app:related-active-learning}.


\paragraph{Deferral under limited review capacity.}
\citet{alves2024tmlr-costsensitive} consider cost-sensitive deferral to multiple experts under individual workload constraints, optimizing assignments using offline estimates of expert error. 
More directly related, \citet{lykouris2025congested} learn a binary classifier online while controlling admission to human review and scheduling; as in our model, feedback is selective and becomes available only after a review is completed. 
In this work, by contrast, the LLM response rule is fixed and the order of service is not controlled by the learner; what is learned is the LLM's accuracy at each self-confidence level, which determines the deferral policy. Their benchmark allocates review capacity using expected arrival and service rates and assumes admitted jobs are reviewed immediately. We instead compare with the optimal finite-horizon policy subject to the same queue dynamics as the learner, with expert-review and holding costs included in the objective.

Additional related work on uncertainty signals and probabilistic calibration, LLM routing, structured admission-control policies, and human--AI decision support is discussed in Appendix~\ref{app:additional-related}.
\section{Problem formulation and optimal deferral}
\label{sec:problem-formulation}

\subsection{LLM--expert system model} 


\begin{figure}[htb!]
  \centering
  \includegraphics[width=0.532\linewidth]{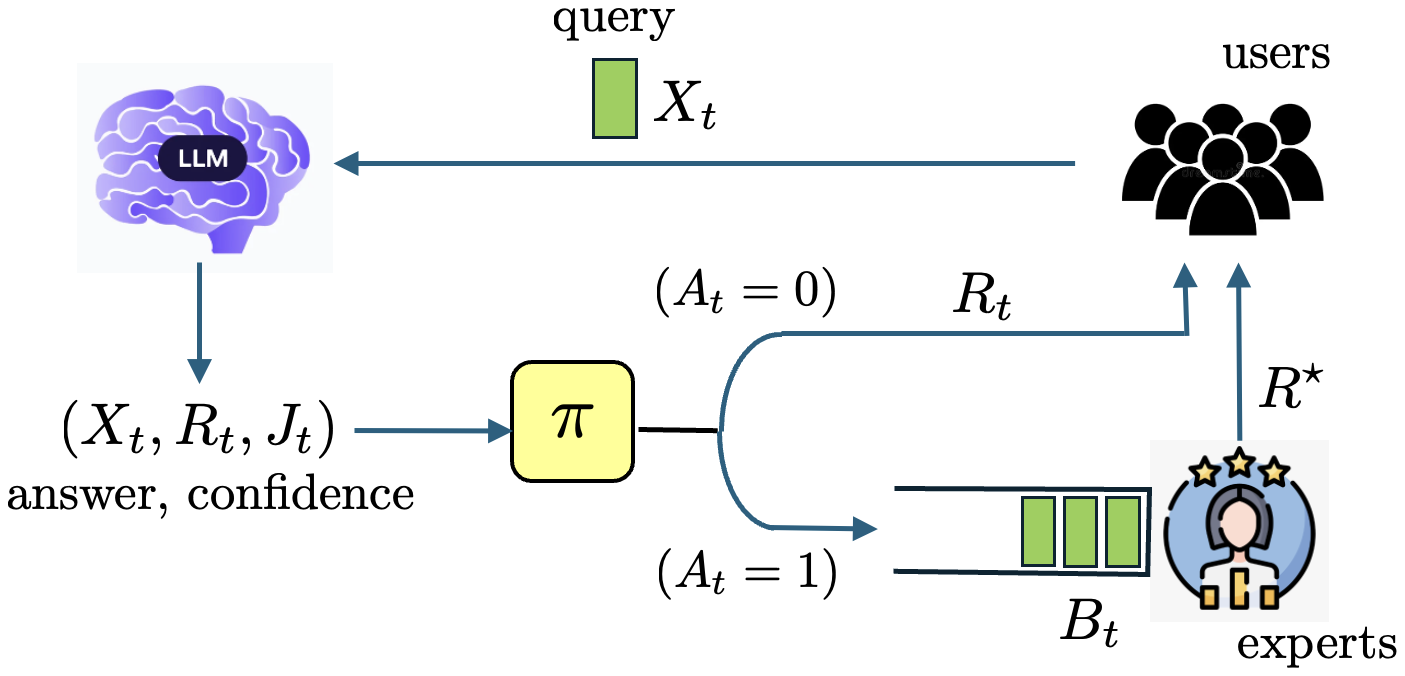}\hfill
  \includegraphics[width=0.428\linewidth]{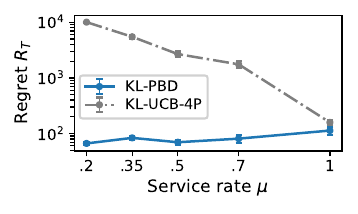}
  \caption{Left: expert-aided LLM system with deferral policy $\pi$ and delayed feedback from completed expert reviews.
Right: regret after $T=5\times 10^4$ queries on MedMCQA with Qwen3-4B against the expert service rate $\mu$, comparing the queue-aware \texttt{KL-PBD} with the queue-unaware \texttt{KL-UCB-4P}. For $\mu=1$, deferred queries are served before the next round and the queue does not build up, so queue awareness provides no benefit. For $\mu<1$, deferred queries accumulate and create congestion, making queue-aware deferral essential for low regret. Details in Appendix~\ref{app:service-rate}.}
  \label{fig:intro-candidates-row}
\end{figure}

We consider a sequential deferral problem over a known horizon $T$. At each round $t=1,\ldots,T$, a user query $X_t$ may arrive. Upon receiving a query, the LLM produces a candidate answer $R_t$ together with an index $J_t\in[K]:=\{1,\ldots,K\}$, which we call its \emph{self-confidence level}.
Our model is agnostic to how $(R_t,J_t)$ is generated. A canonical example is self-consistency \citep{wang2022self}: the LLM generates $K-1$ answers, $R_t$ is chosen from the largest answer-equivalence class, and $J_t$ is one plus the size of that class, so that $J_t=1$ when no answer can be extracted. Alternative constructions of $(R_t,J_t)$ are discussed in Appendix~\ref{app:uncertainty-scores}. Let $R_t^\star$ denote the ground-truth answer and define $Y_t := \mathbf{1}\{R_t \equiv R_t^\star\}$, where $R_1\equiv R_2$ means that the two answers are equivalent. We set $J_t=0$, $X_t=R_t=\emptyset$, $Y_t=1$ when no query arrives and assume that
$(J_t,Y_t)_{t\geq 1}$ is an i.i.d.\ sequence. For
$j\in[K]\cup\{0\}$, let $\omega_j := \mathbb{P}(J_t=j)$, and for $j\in[K]$, define\footnote{If $\omega_j=0$, $h_j^\star$
may be assigned arbitrarily.} $h_j^\star := \mathbb{P}(Y_t=1\mid J_t=j)$. Thus, $h_j^\star$ is the LLM's conditional accuracy at self-confidence level $j$, and we call $h^\star=(h_1^\star,\ldots,h_K^\star)$ the  \emph{conditional accuracy vector}. 

When a query arrives, the decision maker, who observes the self-confidence level $J_t$ but neither $X_t$ nor $R_t$, either accepts the LLM answer or defers the query to a queue served by a pool of experts, who always provide the correct answer. We encode this decision by $A_t\in\{0,1\}$, where $A_t=1$ denotes deferral. We set $A_t=0$ if $J_t=0$. Let $B_t$ be the number of queries in the queue at the beginning of round $t$, and let
$C_t$ be the number of queries that can be served by the experts between
rounds $t$ and $t+1$. We assume that $(C_t)_{t\geq1}$ is i.i.d.\ with distribution $\rho_C$ and independent of $(J_t,Y_t)_{t\ge1}$ and the algorithm's internal randomization, and that $\mu:=\mathbb{P}(C_t \ge 1)>0$. The order in which pending queries are reviewed may depend on the observations and the service capacity, but not on unrevealed ground-truth labels. The queue is initially empty, $B_1=0$, and has known capacity $\bar{B} \ge 1$, so deferral is infeasible when $B_t=\bar{B}$. Its dynamics are $B_{t+1} = \max\{0, B_t+A_t-C_t\}$. 

The cost incurred at round $t$ is $c_t:= c_{\mathrm{hold}} B_t+ c_{\mathrm{exp}} A_t+ c_e(1-A_t)(1-Y_t)$, where $c_{\mathrm{hold}} \ge 0$ is the per-round holding cost of each query in the
queue, $c_{\mathrm{exp}} \ge 0$ is the cost of deferring a query to the experts,
and $c_e>0$ is the cost of returning an incorrect LLM answer. At the end of the horizon, we account for the cost of serving the queries remaining in the queue. These cost parameters are known to the decision maker. Let $D(b)$ denote the expected holding cost of clearing a queue of size $b$. It satisfies $D(0)=0$ and, for $b\geq 1$, $D(b)= c_{\rm hold} b+ \mathbb{E}_{C\sim\rho_C}\!\left[D\bigl((b-C)^+\bigr)\right]$.
The terminal cost is therefore $D(B_{T+1})$. For example, if
$C\sim\operatorname{Bernoulli}(\mu)$, then $D(b)=c_{\rm hold}b(b+1)/(2\mu)$.

The resulting problem is a finite-horizon Markov decision process (MDP) with state $(B_t,J_t)$, action $A_t$, per-round cost $c_t$, and terminal cost $D(B_{T+1})$. A policy $\pi=(\pi_t)_{t=1}^T$ is a sequence of possibly randomized decision rules, where $\pi_t$ maps the current state $(B_t,J_t)$ to a distribution over feasible actions. Our goal is to learn, online, a policy that performs nearly as well as the optimal policy that knows the conditional accuracy vector $h^\star$. We formalize this objective through regret after first characterizing the optimal policy under known $h^\star$.

\subsection{Optimal deferral policy}
\label{sec:optimal-policy}
Fix $h\in[0,1]^K$ and set $h_0:=1$. For
$j\in[K]\cup\{0\}$, define $L^h(j):=c_e(1-h_j)$, which, for $j\in[K]$, is the expected error cost of accepting an LLM response with self-confidence level $j$. Let $V_t^h$ and $Q_t^h$ denote the finite-horizon value and
state-action value functions of the MDP induced by $h$, with terminal value $V_{T+1}^h((b,j))=D(b)$. In other words, $Q_t^h((b,j),a)$ is the expected cost-to-go from state $(b,j)$ when action $a$ is taken at round $t$, and $V_t^h((b,j))$ is its minimum over feasible actions.

To simplify the notation, let $P_a(b,b')$ denote the transition probability from queue state $b$ to $b'$ under action
$a\in\{0,1\}$. Since the queue dynamics depend on the current queue state and action, but not on the current self-confidence level, define,
for $b<\bar B$,
\begin{equation}
\label{eq:marginal_cost_def}
    M_t^h(b)
    :=
    c_{\rm exp}
    +
    \sum_{b'=0}^{\bar B}
    \bigl(P_1(b,b')-P_0(b,b')\bigr)
    \sum_{j=0}^K \omega_j V_{t+1}^h((b',j)).
\end{equation}
We refer to $M_t^h(b)$ as the \emph{marginal cost of deferral}: it
captures both the immediate expert-review cost and the expected increase
in future cost caused by adding a query to the expert queue. We adopt
the convention $M_t^h(\bar B):=+\infty$.

For any $b<\bar B$ and $j\in[K]$, the difference between the costs of
deferring and accepting is
\[
    Q_t^h((b,j),1)-Q_t^h((b,j),0)
    = M_t^h(b)-L^h(j).
\]
Hence, the optimal decision compares the expected error cost of the LLM
response with the marginal cost of deferral. The optimal policy has a threshold structure: a query is deferred whenever the LLM's conditional accuracy is below a threshold determined by the current time and expert backlog. As stated in the following lemma, the threshold is endogenous: through $M_t^h(b)$, it accounts for the future congestion created by the current deferral decision.


\begin{lemma}[Optimal deferral policy]
\label{lem:optimal-deferral}
Fix $h\in[0,1]^K$. Resolving ties in favor of acceptance, an optimal
action at round $t$ is $A_t^h=\mathbf 1\!\left\{J_t\in[K],\;L^h(J_t)>M_t^h(B_t)\right\}$.
Equivalently, defining the accuracy thresholds $\tau_t^h(b)
    :=
    1-M_t^h(b)/c_e$, for all $b\le \bar{B}$,
the optimal policy defers if and only if $h_{J_t}<\tau_t^h(B_t)$. In particular, $\tau_t^h(\bar B)=-\infty$, so deferral is infeasible when the queue is full.
\end{lemma}
\subsection{Learning problem and regret decomposition}
\label{sec:learning-problem}

We now turn to the setting in which the true conditional accuracy vector
$h^\star$ is unknown. The statistical challenge is to learn $h^\star$ while making deferral decisions. Feedback is both \emph{selective} and \emph{delayed}. If the learner accepts the LLM response ($A_t=0$), its ground-truth label $Y_t$ is not revealed, and hence the corresponding error cost is not observed. If the learner defers ($A_t=1$), the ground-truth label becomes available only once the query is processed by an expert. Thus, information about $h^\star$ can be acquired only through the capacity-constrained experts, and the resulting feedback is delayed by the same queue that governs the deferral decisions. In particular, exploration affects not only the information available to the learner but also the future state of the system through congestion.

To isolate this learning challenge, our main analysis assumes that the distributions of the confidence levels $\omega$ and of the service capacity $\rho_C$ are known. This assumption does not require analogous active exploration: the self-confidence level $J_t$ is observed whenever a query arrives, so $\omega$ can be estimated passively from arrival frequencies, while observed service capacities provide samples from $\rho_C$ independently of the deferral decisions. In Appendix~\ref{app:unknown-distributions}, we extend our analysis to the setting where $\omega$ and $\rho_C$ are unknown and show that our worst-case regret guarantees are preserved.



A learning algorithm $\pi$ is analogous to a policy, except that since it learns on the way, the decision taken at time $t$ may depend on {\it all} observations made until time $t$. The expected regret of an algorithm $\pi$ is defined by: 
\[
R_T(\pi) = \mathbb E^\pi\!\left[
        \sum_{t=1}^T c_t + D(B_{T+1})
    \right] - \mathbb E^{\pi^\star}\!\left[
        \sum_{t=1}^T c_t + D(B_{T+1})
    \right].
\]

where $\pi^\star$ is the optimal deferral policy for the conditional accuracy vector $h^\star$, identified in Lemma \ref{lem:optimal-deferral}. Note that the oracle, applying $\pi^\star$, and the learner use the same $K$ self-confidence levels; hence, regret measures only the additional cost incurred from learning the unknown conditional accuracy vector $h^\star$, rather than any approximation error induced by the choice of confidence discretization.

For brevity, a superscript $\star$ on a quantity denotes
evaluation at $h=h^\star$ (e.g. $L^{\star} = L^{h^\star}$). Define the instantaneous decision gap
$
    \Delta_t
    :=
    \bigl|
        L^\star(J_t)-M_t^\star(B_t)
    \bigr|
    \mathbf 1\{A_t\neq A_t^\star\}
$
when $J_t\in[K]$ and $B_t<\bar B$, and set $\Delta_t:=0$ otherwise. 

\begin{lemma}[Regret decomposition]
\label{lem:regret-decomposition}
For any learning algorithm $\pi$, $R_T(\pi)= \mathbb E^\pi \left[\sum_{t=1}^T \Delta_t\right].$
\end{lemma}


Lemma~\ref{lem:regret-decomposition} shows that regret is the expected cumulative decision gap over the rounds in which the learner takes a strictly suboptimal action. Consequently, the learning problem reduces to estimating $h^\star$ accurately enough to identify on which side of the time- and backlog-dependent oracle threshold
$\tau_t^\star(B_t)$ each self-confidence level lies.

\section{The \texttt{KL-PBD} algorithm and its regret}
\label{sec:kl-pbd}

\subsection{Pessimistic Bellman deferral}
\label{sec:learned-M}

We now introduce our algorithm \texttt{KL-PBD}, which learns the unknown conditional
accuracy vector $h^\star$ from delayed expert feedback while accounting
for the congestion induced by deferral. \texttt{KL-PBD} is summarized in
Algorithm~\ref{alg:kl-pbd}. 

For each self-confidence level $j\in[K]$, let $N_{j,t}$ denote the number of level-$j$ queries whose expert reviews have been \emph{completed} before the decision at round $t$. Importantly, $N_{j,t}$ need not equal the number of level-$j$ queries deferred before $t$: a deferred query contributes to $N_{j,t}$ only once it is served and its ground-truth label becomes available.
The selective and delayed nature of the feedback thus enters the algorithm through $N_{j,t}$. Let $Y_{j,n}$ denote the label revealed by the $n$th completed level-$j$ review, and, for $n\geq 1$, define \(\widehat h_{j,n}:=\frac{1}{n}\sum_{\ell=1}^n Y_{j,\ell}.\)
For $T\geq 3$, let $\phi(T):=\log T+\log\log T$ denote the confidence budget. Given $N_{j,t}\geq 1$ observed labels, we
form the KL lower confidence bound
\[
    \underline h_{j,t}
    :=
    \inf\left\{
        q\in[0,\widehat h_{j,N_{j,t}}]:
        N_{j,t}\,
        \kl\!\left(\widehat h_{j,N_{j,t}},q\right)
        \leq \phi(T)
    \right\}
\]
where $\kl$ denotes the Bernoulli relative entropy,
and set $\underline h_{j,t}=0$ when $N_{j,t}=0$. At round $t$, \texttt{KL-PBD} treats $\underline h_t$ as the current conditional accuracy vector and applies the corresponding optimal policy (Algorithm~\ref{alg:kl-pbd}, lines 5--6).
Equivalently, the learner defers whenever the pessimistic expected error cost of the LLM exceeds the Bellman marginal cost of expert review.

\begin{figure}[htb!]
\noindent
\begin{minipage}[t]{0.54\textwidth}
\vspace{0pt}
\setlength{\intextsep}{0pt}
\begin{algorithm}[H]
\caption{\texttt{KL-PBD}}
\label{alg:kl-pbd}
\begin{algorithmic}[1]
\State Set $\underline h_{j,1}=0$ for every $j\in[K]$.
\For{$t=1,\ldots,T$}
    \State Incorporate all expert labels received before round $t$ and update $N_{j,t}$, $\widehat h_{j,N_{j,t}}$, and $\underline h_{j,t}$.
    \State Observe $(B_t,J_t)$.
    \State Compute $\widehat\tau_t(B_t)=1-M_t^{\underline h_t}(B_t)/c_e$.
    \State Choose \( A_t = \mathbf 1\!\left\{ J_t\in[K],\; \underline h_{J_t,t}<\widehat\tau_t(B_t) \right\}.\) 
\EndFor
\end{algorithmic}
\end{algorithm}
\end{minipage}\hfill
\begin{minipage}[t]{0.43\textwidth}
\vspace{0pt}


The use of a lower confidence bound serves both learning and control. Because correctness is observed only for deferred queries, accepting a query provides no information about $h^\star$. A plug-in estimate can therefore become prematurely optimistic about a self-confidence level, causing the learner to accept subsequent queries from that level and thereby suppress the feedback needed to correct the estimate (see Appendix~\ref{app:gbd-pessimism}). 
\end{minipage}
\end{figure}

In contrast, \texttt{KL-PBD} treats uncertain levels pessimistically: before sufficient expert feedback has accumulated, a smaller $\underline h_{j,t}$ makes deferral more attractive and hence promotes the collection of additional labels. This exploration is nevertheless congestion-aware, since the decision is made through the threshold $\widehat\tau_t(B_t)$ and therefore accounts for the current queue and the future cost of expert review. 

\subsection{Regret guarantees}
\label{sec:kl-pbd-regret}

We first give a gap-dependent regret guarantee. Recall that the oracle defers a level-$j$ query at state $(t,b)$ when $h_j^\star < \tau_t^\star(b)=1-M_t^\star(b)/c_e$. The distance $\left|h_j^\star-\tau_t^\star(b)\right|$ therefore measures how far the conditional accuracy of level $j$ is from the oracle threshold at time $t$ and queue length $b$.
Equivalently, for $b<\bar B$, $\left|L^\star(j)-M_t^\star(b)\right|= c_e\left|h_j^\star-\tau_t^\star(b)\right|$ is the difference between the oracle $Q$-values of accepting and deferring. 
Thus, the farther $h_j^\star$ is from the threshold, the less accurately it needs to be estimated to identify the oracle action.

Importantly, states at which $h_j^\star=\tau_t^\star(b)$ should not be treated as having a problematic ``zero gap.'' At such a state,
$L^\star(j)=M_t^\star(b)$, so accepting and deferring have the same oracle $Q$-value. Choosing either action therefore incurs no regret. Accordingly, the relevant notion of separation is the smallest \emph{nonzero} distance from an oracle threshold. To formalize this, define the set of nonzero oracle gaps over the
horizon by
\[
    \mathcal G_T
    :=
    \left\{
        \left|h_j^\star-\tau_t^\star(b)\right|
        :
        j\in[K],\;
        t\leq T,\;
        b<\bar B,\;
        h_j^\star\neq\tau_t^\star(b)
    \right\}.
\]


If $\mathcal G_T=\emptyset$, every feasible oracle comparison is a tie. Hence the decision gap in Lemma~\ref{lem:regret-decomposition}
is zero at every round, and every algorithm has zero regret. We therefore restrict attention to the nontrivial case
$\mathcal G_T\neq\emptyset$ and define $\gamma_T:=\min\{1,\min \mathcal G_T\}$. 
This quantity captures the smallest separation that
the learner must resolve in order to distinguish a strict oracle preference between accepting and deferring. When this separation is
bounded away from zero, sufficiently accurate estimates of $h^\star$ identify the oracle action.
\begin{theorem}[Gap-dependent regret of \texttt{KL-PBD}]
\label{thm:pbd-gap-dependent}
Suppose that $\mathcal G_T\neq\emptyset$. For every $T\geq3$,
\[
R_T(\pi^{\texttt{KL-PBD}})\le
K\kappa\phi(T)
+2c_e\sum_{j=1}^K\min\biggl\{\omega_jT,\bar{B}
+\frac{4\phi(T)}{\gamma_T}\biggr\}
+2\bar{B}c_e\biggl[1+4\sum_{j=1}^K
\log\Bigl(1+\frac{\omega_jT}{\phi(T)}\Bigr)\biggr]
\]

where $\kappa:=4e(\max\{c_e,c_{\mathrm{exp}}\}
    +c_{\mathrm{hold}}\bar B/\mu)$.
\end{theorem}

In particular, if $\gamma_T\ge\gamma>0$ uniformly in $T$, then $R_T(\pi^{\texttt{KL-PBD}})=O(\frac{K\log T}{\gamma})$. Thus, when the conditional accuracies remain separated from the oracle thresholds, \texttt{KL-PBD} incurs only logarithmic regret. A stronger version of Theorem~\ref{thm:pbd-gap-dependent} is proven in Appendix~\ref{app:gap-dependent-ub-proof}.

We next give a gap-free guarantee, which holds without any separation from the oracle thresholds. A stronger version of Theorem~\ref{thm:pbd-regret} is proven in Appendix~\ref{app:worst-case-ub-proof}.

\begin{theorem}[Worst-case regret of \texttt{KL-PBD}]
\label{thm:pbd-regret}
For every $T\geq3$,
\[
R_T(\pi^{\texttt{KL-PBD}})\le
K\kappa\phi(T)
+c_e\sum_{j=1}^K\min\biggl\{\omega_jT,\bar{B}+2\sqrt{2\omega_jT\phi(T)}\biggr\}
+\bar{B}c_e\biggl[1+2\sum_{j=1}^K
\log\Bigl(1+\tfrac{\omega_jT}{\phi(T)}\Bigr)\biggr].
\]
\end{theorem}

Up to logarithmic factors, Theorem~\ref{thm:pbd-regret} gives $R_T(\pi^{\texttt{KL-PBD}})=\widetilde O(K+\sum_{j=1}^K\min\{\omega_jT,\sqrt{\omega_jT}\})$. In particular, the worst case over arrival distributions yields $\widetilde O(\sqrt{KT})$ for $K\le T$ (Remark~\ref{rem:ub-worst-case-omega}).

The main difficulty in the regret analysis is that $\underline h_t$ also enters the marginal cost of deferral, so pessimism may cause \texttt{KL-PBD} not only to defer queries that the oracle accepts, but also to accept queries that the oracle defers. 
For a high-level summary of our proof technique, see Appendix~\ref{app:ub-proof-sketch}.


\section{Lower bounds}
\label{sec:known-M}

We now show that the dependence on $K$, $T$, and the gaps in Theorems~\ref{thm:pbd-gap-dependent} and~\ref{thm:pbd-regret} is optimal in general (up to logarithmic factors for worst-case regret) by deriving matching lower bounds, which are proven in Appendix~\ref{app:known-M}. Write $R_T(\pi;h)$ for the regret when the conditional accuracy vector is $h$, and define $\tau:=1-[c_{\rm exp}+(1-\mu)c_{\rm hold}/\mu]/c_e$.
This is the largest accuracy threshold the oracle ever uses: $\tau_t^h(b)\le\tau$ for all $h$, $t$ and $b$, with equality at the empty queue in the last round (Lemma~\ref{lem:lb-one-job-floor}).

We first give an instance-dependent lower bound. We call an algorithm $\pi$ uniformly efficient if, run with horizon $T$, it satisfies $R_T(\pi;h)=o(T^\alpha)$ for every $h\in[0,1]^K$ and every $\alpha>0$. Let $\overline\tau^\star:=\lim_{T-t\to\infty}\tau_t^\star(0)$ be the long-horizon oracle threshold at the empty queue, which exists by Lemma~\ref{lem:lb-one-job-floor}. For every $t$ and $b$, $\tau_t^\star(b)\le\tau_t^\star(0)$ (Proposition~\ref{prop:queue-threshold}) and $\overline\tau^\star\le\tau_t^\star(0)\le\tau$ (Lemma~\ref{lem:lb-one-job-floor}). Thus, $\overline\tau^\star$ is the long-horizon limit of the largest threshold the oracle uses.

\begin{theorem}[Instance-dependent lower bound]
\label{thm:instance_dependent_lower_bound}
If $\overline\tau^\star\in(0,1)$, every uniformly efficient algorithm $\pi$ satisfies
\begin{equation}
\label{eq:main-instance}
\liminf_{T\to\infty}\frac{R_T(\pi;h^\star)}{\log T}
\ge
c_e\sum_{\substack{j\in[K]:\\ h_j^\star>\overline\tau^\star,\ \omega_j>0}}
\frac{h_j^\star-\overline\tau^\star}{\kl(h_j^\star,\overline\tau^\star)}.
\end{equation}
\end{theorem}


As the positive gaps $h_j^\star-\overline\tau^\star$ tend to zero, with $\overline\tau^\star$ bounded away from $0$ and $1$, each summand in the lower bound scales as $2\overline\tau^\star(1-\overline\tau^\star)/(h_j^\star-\overline\tau^\star)$. Hence, the logarithmic growth in $T$ and inverse-gap dependence of Theorem~\ref{thm:pbd-gap-dependent} are necessary in general. When $\mu=1$, $\overline\tau^\star=\tau$ and \texttt{KL-PBD} attains \eqref{eq:main-instance} with equality (Proposition~\ref{prop:kl-noqueue}). Additionally, the stronger version of Theorem~\ref{thm:pbd-gap-dependent} proved in Appendix~\ref{app:gap-dependent-ub-proof} matches the dependence of Theorem~\ref{thm:instance_dependent_lower_bound} on the conditional accuracies for small gaps.

Our proof relies on a change-of-measure argument, as in \citet{lai1985asymptotically}. The main difficulty in our setting is that the marginal cost of deferral depends on the round and the queue state, and that the oracle threshold changes with the instance. 
For a high-level summary of our proof technique, see Appendix~\ref{app:lb-proof-sketch}.

We next give a minimax lower bound.

\begin{theorem}[Minimax lower bound]
\label{thm:known-M-lower}
Assume $\tau\in(0,1)$. For every $T\ge1$,
\begin{equation}
\label{eq:main-minimax}
\inf_\pi\sup_{h\in[0,1]^K}R_T(\pi;h)
\ge
\frac1{256}\sum_{j=1}^K
\min\!\left\{\mu^3c_e\min\{\tau,1-\tau\}\min\{\omega_jT,\sqrt{\omega_jT}\},\ \mu^2c_{\rm hold}\,\omega_jT\right\}.
\end{equation}
When $\mu=1$, the last term can be dropped and the constant $1/256$ improved to $1/12$.
\end{theorem}

If $c_{\rm hold}>0$ or $\mu=1$, the lower bound scales as $\sum_{j=1}^K\min\{\omega_jT,\sqrt{\omega_jT}\}$. Thus, its dependence on $T$ and $\omega$ agrees with Theorem~\ref{thm:pbd-regret} up to logarithmic factors and the upper bound's additive $K$ term. The lower bound is largest for balanced levels, $\omega_0=0$ and $\omega_j=1/K$ (Remark~\ref{rem:worst_case_distribution}), giving $\min\{T,\sqrt{KT}\}$ and matching the upper bound up to logarithmic factors for $K\le T$.

The main difficulty in the minimax lower-bound construction is that hard instances cannot be obtained by perturbing the conditional accuracies around a fixed threshold, since the threshold itself depends on them, and that nearby instances are only guaranteed to have different optimal actions when the queue is empty. 
For a high-level summary of our proof technique, see Appendix~\ref{app:lb-proof-sketch}.

\section{Numerical experiments}\label{sec:experiments}

We now compare \texttt{KL-PBD} with baselines and examine the structure of the oracle policy on instances built from LLM outputs on GSM8K, NuminaMath-TIR and MedMCQA. Additional results and details are deferred to Appendices~\ref{app:synthetic-experiments} and~\ref{app:real-world-details}.

\paragraph{Instances.} We use two LLMs, Qwen3-4B-Instruct-2507 and Qwen3-0.6B, and build one instance per dataset--LLM pair; results for Qwen2.5-0.5B-Instruct are reported in Appendix~\ref{app:real-world-details}. Each instance uses the self-consistency construction of Section~\ref{sec:problem-formulation}: for every question, the LLM generates $20$ answers, so there are $K=21$ levels. For a query at level $j$, let $p_j$ denote its self-consistency score, i.e., the fraction of generated answers that agree with $R_t$, so that $p=(p_1,\ldots,p_K)=(0,.05,\ldots,1)$, where $p_1=0$ corresponds to queries for which no answer could be extracted. We take $\omega_j$ and $h_j^\star$ to be the empirical frequency and conditional accuracy of self-consistency level $j$ over the whole dataset, and sample queries independently with replacement from that dataset, so that a query arrives at every round ($\omega_0=0$). Regret is evaluated against the optimal policy for this empirical instance. Figure~\ref{fig:accuracy-2by3} shows three calibration regimes. On GSM8K and NuminaMath the score is informative: Qwen3-0.6B is especially close to calibrated on GSM8K, while Qwen3-4B is underconfident at low scores and overconfident at high scores. On MedMCQA, the score is nearly uninformative: accuracy barely varies with the score below $100\%$, and the $100\%$ level has accuracy only $.79$ (Qwen3-4B) and $.68$ (Qwen3-0.6B).

\begin{figure}[t]
  \centering
  \includegraphics[width=\textwidth]{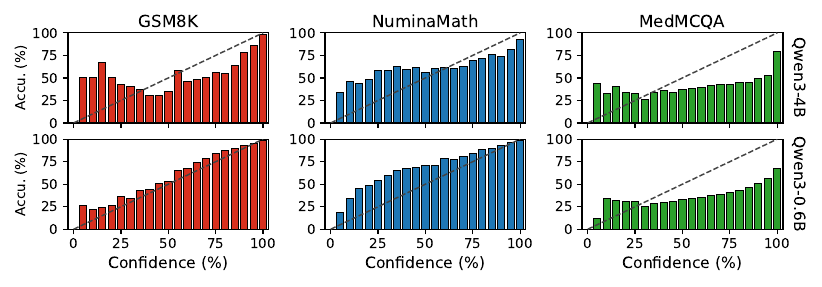}
    \vspace{-1em}
  \caption{Conditional accuracy of the self-consistency score (dashed line: perfect calibration).}
    \vspace{-1em}

  \label{fig:accuracy-2by3}
\end{figure}

\paragraph{Setup and baselines.} We set $(c_e,c_{\rm exp},c_{\rm hold})=(1,0.25,0.03)$, $\bar{B}=8$, $C_t\sim\operatorname{Bernoulli}(0.35)$, and $T=10^5$. The arrival distribution $\omega$ is unknown to the algorithms and is estimated as $\hat{\omega}_{j,t}=t^{-1}\sum_{s=1}^{t}\mathbf{1}\{J_s=j\}$ by all methods that require it, except \texttt{FQT}, which is tuned under a uniform $\omega$. We compare \texttt{KL-PBD} with: i) \texttt{SALMUT-SF}, a threshold-based actor--critic policy adapted from \cite{roy2021online} to handle unknown rewards and selective feedback; ii) Fixed Queue Threshold (\texttt{FQT}), a non-learning baseline that uses a queue threshold optimized under the surrogate $h=p$; iii) variants of \texttt{KL-PBD}, namely \texttt{KL-PBD+}, which uses a more aggressive confidence budget inspired by KL-UCB+ type rules \citep{garivierkl2011,achab2018profitable}, Greedy Bellman Deferral (\texttt{GBD}), which replaces the lower confidence bound $\underline h_t$ by the empirical mean $\widehat h_{j,N_{j,t}}$, and Naive Bellman Deferral (\texttt{NBD}), which freezes the estimate at $h=p$, i.e., treats the score as calibrated. These baselines are described in more detail in Appendix~\ref{app:policy-implementations}.

\paragraph{Regret of \texttt{KL-PBD} against baselines.} In Figure~\ref{fig:regret-2by3}, we report regret averaged over $100$ independent runs with pointwise $95\%$ confidence bands. Solid lines are algorithms that learn $h^\star$, and dashed lines are the baselines using the $h=p$ surrogate. The results illustrate the impact of score calibration: on GSM8K and NuminaMath, \texttt{NBD} is generally the best non-learning baseline and competes with \texttt{KL-PBD} over this horizon, whereas on MedMCQA it incurs $20$ to $65$ times the regret of \texttt{KL-PBD+}. Among learning algorithms, \texttt{GBD} suffers from the selective-feedback failure described in Appendix~\ref{app:gbd-pessimism}, and \texttt{SALMUT-SF} adapts too slowly to be competitive. \texttt{KL-PBD+} attains the lowest regret in all six instances, up to a factor $4$ below \texttt{KL-PBD}: its more aggressive level-dependent confidence budget reduces the conservatism of the deferral rule. Appendix~\ref{app:real-world-details} reports regret values at horizons $T\in\{10^4,5\cdot10^4,10^5,2\cdot10^5\}$; the gap between \texttt{KL-PBD+} and the $h=p$ baselines keeps widening with $T$. Comparisons to queue-unaware baselines, including adaptations of bandit algorithms from \citet{achab2018profitable}, are deferred to Appendix~\ref{app:service-rate}: \texttt{KL-PBD} and \texttt{KL-PBD+} generally beat them by a wide margin when $\mu<1$, and remain competitive under immediate service ($\mu=1$).

\begin{figure}[t]
  \centering
  \includegraphics[width=\textwidth]{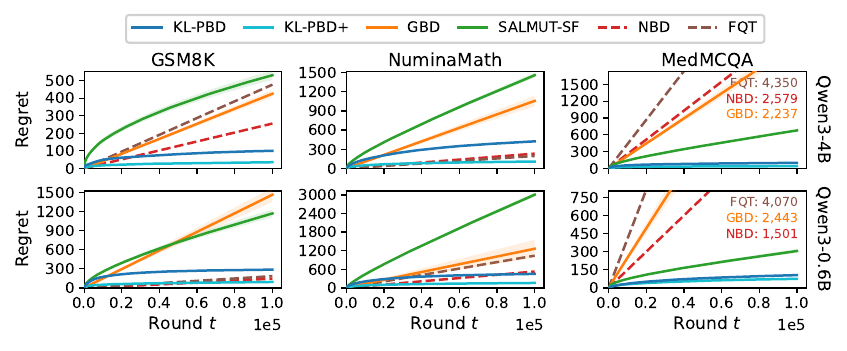}
  \vspace{-1em}
  \caption{Regret averaged over 100 runs. Curves leaving the frame are listed with their final regret.}
  \vspace{-1em}
  \label{fig:regret-2by3}
\end{figure}

\paragraph{Structure of the oracle policy.} Figure~\ref{fig:policy-2by3}  shows the oracle threshold and the resulting decisions for Qwen3-0.6B. For every $\mu$, the threshold decreases with the queue state, as established in Proposition~\ref{prop:queue-threshold}. Higher values of $\mu$ lead to higher deferral thresholds, as faster expert service reduces the congestion cost of deferral: the oracle then defers even queries on which the LLM is fairly accurate. The bottom row shows the oracle deferral decisions for the service rate $\mu=0.35$ used in the regret experiments. Accordingly, the oracle defers fewer levels as the queue becomes more congested. In contrast, it does not always defer fewer levels as self-consistency increases: since the oracle defers exactly the levels whose conditional accuracy lies below the threshold, this occurs because the conditional accuracy is non-monotone and the threshold falls between the accuracies of two levels whose order is reversed relative to their self-consistency. Similar plots for Qwen3-4B and comparisons with the thresholds and decisions learned by \texttt{KL-PBD} are deferred to Appendix~\ref{app:policy-diagnostics}.

\begin{figure}[t]
  \centering
  \includegraphics[width=\textwidth]{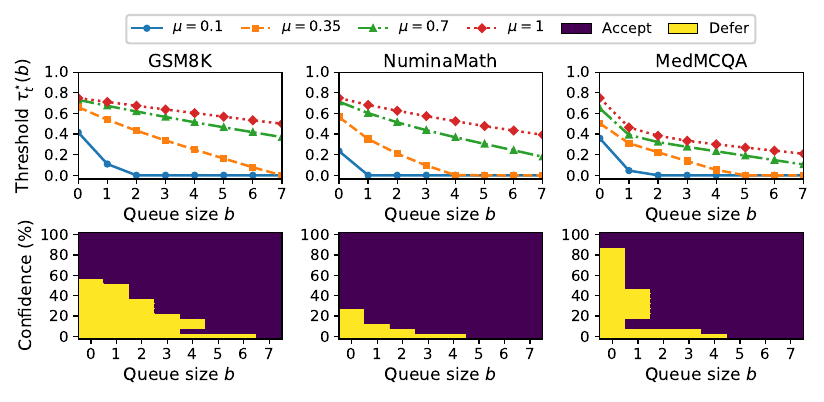}
  \vspace{-1em}
  \caption{Oracle thresholds for varying service rate $\mu$ (top) and oracle decisions for $\mu=0.35$ (bottom), for Qwen3-0.6B at round $t=T/2$.}
  \label{fig:policy-2by3}
  \vspace{-1em}
\end{figure}

\section{Conclusion}

We studied online deferral to a congested expert, where deferral costs depend on the queue state and expert feedback is both selective and delayed. This creates a feedback loop: learning when to trust the LLM requires expert labels, but acquiring those labels consumes the same limited resource that the learner is trying to manage. \texttt{KL-PBD} addresses this through pessimistic estimates of the unknown conditional accuracies, preventing premature acceptance from cutting off the feedback needed for learning. Its regret guarantees are near-optimal in both the worst-case and gap-dependent regimes.

Several extensions are natural. First, expert effort could depend on the query: low-confidence queries may also be harder to review, making service times confidence-level dependent. Second, our benchmark fixes a discretization of the confidence score; for continuous scores especially, choosing the number of levels would trade approximation error against statistical complexity. Third, the number of LLM generations could itself be treated as a design variable, trading the quality of the confidence signal against inference cost.

\subsection*{AI use statement}

In this work, we used generative AI tools to assist with:  
i) writing and editing, and in particular, to improve the presentation of the proofs; ii) coding, to help implement our algorithms and write data-processing and parsing scripts; iii) literature search, to help identify potentially relevant work.

The coding assistance was used strictly to implement our ideas and methods. We conducted our own literature review, and AI was used as a complementary tool for discovering potentially relevant references that we might have missed, not for synthesizing the literature on our behalf. We checked and revised all AI-assisted text and code, and independently verified all AI-suggested references and literature-related claims against the original sources. We take full responsibility for the final content of this work.

\subsection*{Reproducibility statement}

We provide a complete description of our problem setting and assumptions in Section~\ref{sec:problem-formulation}, and of our algorithm in Section~\ref{sec:kl-pbd}. Full proofs of our results are given in Appendices~\ref{app:problem-formulation}--\ref{app:known-M}, with accompanying proof sketches to help the reader navigate them. Our experimental setup is described in Section~\ref{sec:experiments} and Appendices~\ref{app:datasets}--\ref{app:real-world-details}, including the datasets used and how they were parsed, the baselines we compare our algorithm to, and our evaluation protocol. The code used to reproduce our experiments is included in the supplementary material.

\bibliography{references}
\bibliographystyle{iclr2027_conference}

\appendix
\startcontents[appendix]
\printcontents[appendix]{}{1}[2]{\section*{Appendix Content}}
\clearpage
\section{Proofs for Section~\ref{sec:problem-formulation}}
\label{app:problem-formulation}

Throughout the appendices, write
$\mathcal B:=\{0,\ldots,\bar B\}$ and $S_t:=(B_t,J_t)$. Let $\mathcal A(b,j):=\{0,1\}$ when $j\in[K]$ and $b<\bar B$, and $\mathcal A(b,j):=\{0\}$ otherwise. We recall that action $1$ corresponds to deferral, and action $0$ to acceptance. Let $\mathcal F_t$ be the sigma-field generated by the self-confidence levels,
actions, service capacities, completed-query identities and their
revealed ground-truth labels through the end of round $t$, together with past internal randomization. The algorithm chooses $A_t$ using only $\mathcal F_{t-1}$, $J_t$, and independent randomization. The queries completed in each round are selected using only the observed history and service capacity, before any of their ground-truth labels are revealed. Thus $B_t$ is $\mathcal F_{t-1}$-measurable, whereas $J_t$ is independent of $\mathcal F_{t-1}$ with distribution $\omega$.
For a real-valued function $f$ on a finite domain, we write $\operatorname{sp}(f):=\max_x f(x)-\min_x f(x)$.

\subsection{Optimal deferral and regret decomposition}

\begin{proof}[Proof of Lemma~\ref{lem:optimal-deferral}]
For $b<\bar B$, $j\in[K]$, and $s=(b,j)$,
\begin{equation}
\label{eq:bellman-margin}
Q_t^h(s,1)-Q_t^h(s,0)=M_t^h(b)-L^h(j).
\end{equation}
Hence deferral is optimal when $L^h(j)>M_t^h(b)$, acceptance is
optimal when $L^h(j)<M_t^h(b)$, and both actions are optimal at
equality. On no-query and full-queue states, acceptance is the
only feasible action. Resolving ties in favor of acceptance gives
the stated action rule and its equivalent threshold formulation.
\end{proof}

\begin{proof}[Proof of Lemma~\ref{lem:regret-decomposition}]
All expectations are under the true conditional accuracy vector
$h^\star$. By the Bellman definition and the observation model,
\[
Q_t^\star(S_t,A_t)
=
\mathbb E^\pi\!\left[
c_t+V_{t+1}^\star(S_{t+1})
\mid \mathcal F_{t-1},J_t,A_t
\right].
\]
Taking expectations and summing over $t$, the value terms telescope:
\[
\begin{aligned}
\mathbb E^\pi\!\left[
\sum_{t=1}^T
\bigl(Q_t^\star(S_t,A_t)-V_t^\star(S_t)\bigr)
\right]
={}&
\mathbb E^\pi\!\left[
\sum_{t=1}^T c_t+D(B_{T+1})
\right]
-\mathbb E^\pi[V_1^\star(S_1)].
\end{aligned}
\]
Since $B_1=0$ and $J_1\sim\omega$ under both the algorithm and the oracle policy,
\[
\mathbb E^\pi[V_1^\star(S_1)]
=
\sum_{j=0}^K\omega_jV_1^\star((0,j))
=
\mathbb E^{\pi^\star}\!\left[
\sum_{t=1}^T c_t+D(B_{T+1})
\right].
\]
Consequently,
\begin{equation}
\label{eq:regret_decomp}
R_T(\pi)
=
\mathbb E^\pi\!\left[
\sum_{t=1}^T
\bigl(Q_t^\star(S_t,A_t)-V_t^\star(S_t)\bigr)
\right].
\end{equation}

When $J_t\in[K]$ and $B_t<\bar B$,
\eqref{eq:bellman-margin} gives
\[
Q_t^\star(S_t,A_t)-V_t^\star(S_t)
=
\bigl|L^\star(J_t)-M_t^\star(B_t)\bigr|
\mathbf1\{A_t\ne A_t^\star\}
=
\Delta_t
\]
since the action space is binary. Otherwise, acceptance is the only feasible action, and both sides
are zero. Substitution into \eqref{eq:regret_decomp} proves the claim.
\end{proof}

\subsection{Service-capacity structure}

In this section, we prove several structural properties induced by our service-capacity model. The proposition below states that the marginal cost of deferral increases with the expert queue state. 

\begin{proposition}[Monotone marginal cost]
\label{prop:queue-threshold}
For every conditional accuracy vector \(h\) and every \(t\le T\),  the Bellman
marginal cost is nondecreasing in the queue state:
\[
c_{\mathrm{exp}}
\le
M_t^h(0)
\le
M_t^h(1)
\le
\cdots
\le
M_t^h(\bar{B}-1).
\]

Consequently, for every round $t$ and self-confidence level
$j\in[K]$, there exists a threshold
$\tau_{t,j}^{\rm queue}(h)\in\{0,\ldots,\bar B\}$ such that,
resolving ties in favor of acceptance, the optimal policy defers
a query with self-confidence level $j$ if and only if
\[
B_t<\tau_{t,j}^{\rm queue}(h).
\]

\end{proposition}

For a conditional accuracy vector \(h\), consider the queue value
\[
U_t^h(b)
:=
\sum_{j=0}^K
\omega_jV_t^h((b,j)),
\]
i.e., the value at queue state $b$ before the level of the incoming query is revealed. For \(t\le T\) and \(b\in\mathcal B\), let
\(F_t^h(b):=\mathbb E[U_{t+1}^h((b-C_t)^+)]\), and, for
\(b<\bar{B}\), define \(d_t^h(b):=F_t^h(b+1)-F_t^h(b)\).
Then the Bellman marginal cost can be written as
\begin{equation}
\label{eq:service-bellman-marginal}
M_t^h(b)
=
c_{\mathrm{exp}}+d_t^h(b),
\qquad b<\bar{B}.
\end{equation}

We first establish the structural properties of the queue value
function that will also be useful later. Define its adjacent increments
by
\[ x_t^h(q) := U_t^h(q)-U_t^h(q-1), \qquad q=1,\ldots,\bar{B}. \]

\begin{lemma}[Queue-value increments under service capacity]
\label{lem:service-queue-increments}
For every conditional accuracy vector \(h\) and every \(t\le T+1\),
\[ 0 \le x_t^h(1) \le x_t^h(2) \le \cdots \le x_t^h(\bar{B}). \]
Consequently, \(U_t^h\) is nondecreasing and discrete convex in the
queue state.

Moreover, for every \(t\le T\) and \(b<\bar{B}\),
\begin{equation}
\label{eq:service-smoothed-increment}
d_t^h(b)
=
\sum_{c=0}^{b}
\rho_C(c)\,
x_{t+1}^h(b+1-c).
\end{equation}
\end{lemma}

\begin{proof}
\Eqref{eq:service-smoothed-increment} follows directly from
the service-capacity dynamics. Indeed,
\[
d_t^h(b) = \mathbb E\!\left[ U_{t+1}^h((b+1-C_t)^+) - U_{t+1}^h((b-C_t)^+) \right].
\]
If \(C_t=c\le b\), the two queue states differ by one and the difference
equals \(x_{t+1}^h(b+1-c)\). If \(c>b\), both queue states are zero.
Hence \eqref{eq:service-smoothed-increment} follows.

We prove the monotonicity of the increments by backward induction. At $t=T+1$, we have $U_{T+1}^h=D$, independently of $h$. Let $S_0:=0$ and $S_k:=C_1+\cdots+C_k$ for $k\geq1$, where $C_1,C_2,\ldots$ are i.i.d.\ with law $\rho_C$. Iterating the definition of $D$ gives
\[
D(b)=c_{\rm hold}\sum_{k\geq0}\mathbb E[(b-S_k)^+].
\]
The series is finite since $\mu>0$. Let $g(q):=D(q)-D(q-1)$. Since
\((q-S_k)^+-(q-1-S_k)^+=\mathbf1\{S_k<q\},\)
we have
\[g(q)=c_{\rm hold}\sum_{k\geq0}\mathbb P(S_k<q).\]
Thus $g(q)\geq0$ and is nondecreasing in $q$, so
\[
0\leq g(1)\leq g(2)\leq\cdots\leq g(\bar B).
\]
Therefore $U_{T+1}^h=D$ is nondecreasing and discrete convex, establishing the base case.

Suppose now that
\[ 0 \le x_{t+1}^h(1) \le \cdots \le x_{t+1}^h(\bar{B}). \]
\Eqref{eq:service-smoothed-increment} immediately gives
\(d_t^h(b)\ge0\).
Moreover, for \(b<\bar{B}-1\),
\[
\begin{aligned}
d_t^h(b+1)-d_t^h(b)
={}&
\sum_{c=0}^{b}
\rho_C(c)
\left[
x_{t+1}^h(b+2-c)
-
x_{t+1}^h(b+1-c)
\right]
\\
&+
\rho_C(b+1)x_{t+1}^h(1)
\ge0.
\end{aligned}
\]
Thus
\begin{equation}
\label{eq:service-d-monotone}
0\le d_t^h(0) \le d_t^h(1) \le \cdots \le d_t^h(\bar{B}-1).
\end{equation}

For $j=0,\ldots,K$, write
\(
r_j^h:=L^h(j)-c_{\mathrm{exp}}.
\)
In particular, $r_0^h=-c_{\mathrm{exp}}$ because $L^h(0)=0$.

For $b<\bar{B}$, the Bellman equation for the queue value function gives
\[U_t^h(b)=c_{\mathrm{hold}}b+F_t^h(b)+c_{\mathrm{exp}}+\sum_{j=0}^K\omega_j\min\{r_j^h,d_t^h(b)\}.\]
Indeed, for $j\in[K]$ this represents the minimum between acceptance and
deferral costs. For $j=0$, action $0$ is forced; since
$r_0^h=-c_{\mathrm{exp}}$ and $d_t^h(b)\ge0$, the corresponding term
satisfies
$c_{\mathrm{exp}}+\min\{r_0^h,d_t^h(b)\}=0$, and therefore represents
the no-query state correctly.

Consequently, for $1\le q<\bar B$, subtracting the  Bellman equations for the queue value function at $q$ and $q-1$ gives
\begin{equation}\label{eq:service-increment-recursion}
\begin{aligned}
x_t^h(q)
={}&c_{\mathrm{hold}}+d_t^h(q-1)
+\sum_{j=0}^K\omega_j\left[\min\{r_j^h,d_t^h(q)\}-\min\{r_j^h,d_t^h(q-1)\}\right]\\
={}&c_{\mathrm{hold}}+\sum_{j=0}^K\omega_j\operatorname{med}\{d_t^h(q-1),d_t^h(q),r_j^h\},
\end{aligned}
\end{equation}
where the second equality uses $\sum_{j=0}^K\omega_j=1$ and $d_t^h(q-1)\le d_t^h(q)$, and $\operatorname{med}\{x,y,z\}$ denotes the median of its three arguments.

At $q=\bar{B}$, action $1$ is infeasible.
Hence
\[
U_t^h(\bar{B})
=
c_{\mathrm{hold}}\bar{B}
+
F_t^h(\bar{B})
+
\sum_{j=0}^K\omega_jL^h(j),
\]
which together with the Bellman equation at $\bar{B}-1$ yields
\begin{equation}
\label{eq:service-increment-boundary}
x_t^h(\bar{B})
=
c_{\mathrm{hold}}
+
\sum_{j=0}^K
\omega_j
\max\left\{
d_t^h(\bar{B}-1),
r_j^h
\right\}.
\end{equation}

By \eqref{eq:service-d-monotone}, for every fixed \(r\),
\(\operatorname{med}\{d_t^h(q-1),d_t^h(q),r\}\) is nonnegative and nondecreasing in \(q\). Therefore
\[
0
\le
x_t^h(1)
\le
\cdots
\le
x_t^h(\bar{B}-1).
\]
Finally,
\[
\max\left\{
d_t^h(\bar{B}-1),r
\right\}
\ge
\operatorname{med}
\left\{
d_t^h(\bar{B}-2),
d_t^h(\bar{B}-1),
r
\right\},
\]
whenever \(\bar{B}\ge2\). Combining
\eqref{eq:service-increment-recursion} and
\eqref{eq:service-increment-boundary} therefore gives
\(x_t^h(\bar{B})\ge x_t^h(\bar{B}-1)\).
For \(\bar{B}=1\), the conclusion follows directly from
\eqref{eq:service-increment-boundary}.
This completes the backward induction.
\end{proof}

\begin{proof}[Proof of Proposition~\ref{prop:queue-threshold}]
By Lemma~\ref{lem:service-queue-increments}, \(x_{t+1}^h(q)\) is nonnegative and nondecreasing in \(q\). Hence \eqref{eq:service-smoothed-increment} and \eqref{eq:service-bellman-marginal} give
\[
\begin{gathered}
0\le d_t^h(0)\le d_t^h(1)\le\cdots\le d_t^h(\bar B-1),\\
c_{\mathrm{exp}}\le M_t^h(0)\le M_t^h(1)\le\cdots\le M_t^h(\bar B-1).
\end{gathered}
\]
Fix \(t\) and self-confidence level \(j\), and define
\[
\tau_{t,j}^{\rm queue}(h):=
\min\left(\left\{b\in\{0,\ldots,\bar B-1\}:L^h(j)\le M_t^h(b)\right\}\cup\{\bar B\}\right).
\]
Since \(M_t^h(b)\) is nondecreasing in \(b\), \(L^h(j)>M_t^h(b)\) if and only if \(b<\tau_{t,j}^{\rm queue}(h)\). By \eqref{eq:bellman-margin}, the optimal policy with ties resolved in favor of acceptance therefore defers if and only if \(B_t<\tau_{t,j}^{\rm queue}(h)\).
\end{proof}

\clearpage
\section{Proofs for Section~\ref{sec:kl-pbd-regret}}
\label{app:learned-M}

\subsection{Proof sketch}\label{app:ub-proof-sketch}
To ease navigation within this appendix, we first provide a high-level proof sketch of our gap-dependent and worst-case upper bounds. They are proven in Appendices~\ref{app:gap-dependent-ub-proof} and ~\ref{app:worst-case-ub-proof} respectively.

\textit{Proof sketch.} The analysis exploits that pessimism makes \texttt{KL-PBD} almost one-sided. On the high-probability event $\underline h_t\le h^\star$, if the marginal cost of deferral were computed under the true model $h^\star$, the learner would accept only if the oracle does, and regret would come from deferrals only. This ordering is not exact, because $\underline h_t$ is also plugged into the marginal cost of deferral. To account for these additional errors, we track the difference between the value functions of the pessimistic and true models: Proposition~\ref{prop:bellman-stability} relates this misspecification to the estimation error. Updating $\underline h_t$ leaves an additional term controlled by its total variation, which is logarithmic. The remaining estimation error is charged only on deferred rounds, which are exactly the rounds that produce labels, so the confidence widths shrink as they are charged. Summing these widths incurs only an additive $\bar B$ per level from delayed feedback, since at most $\bar B$ previously deferred queries can still be unlabeled, and yields Theorem~\ref{thm:pbd-regret}. For Theorem~\ref{thm:pbd-gap-dependent}, we exploit that, by Proposition~\ref{prop:bellman-stability}(ii), estimation errors smaller than $\gamma_T/2$ cannot change strict oracle decisions. We therefore compare $\underline h_t$ with the intermediate model obtained by replacing $\underline h_{j,t}$ with $\max\{\underline h_{j,t},h_j^\star-\gamma_T/2\}$, which agrees with the oracle on every strict comparison. Only the confidence width beyond $\gamma_T/2$ needs to be charged, yielding the $1/\gamma_T$ dependence. \qed

\subsection{Notation}

All expectations in this appendix are under $h^\star$ and
$\pi^{\texttt{KL-PBD}}$. For $j\in[K]$ and $t\ge0$, define
$D_{j,t}:=\sum_{s=1}^t\mathbf1\{J_s=j,A_s=1\}$.
Recall that $N_{j,t}$ counts completed self-confidence level $j$ labels available before round $t$. In this appendix, we prove more general regret upper bounds from which the bounds of Theorems~\ref{thm:pbd-regret} and~\ref{thm:pbd-gap-dependent} directly follow. The technical lemmas used in the regret proofs are stated and proved in Section~\ref{app:ub-technical}.


We first establish bounds for a general confidence budget $\phi(T)>1$, then specialize to $\phi(T)=\log T+\log\log T$ to obtain the main-text bounds. For $T\ge2$, define
\begin{equation}
\eta_T:=T\min\left\{ 1,\, 2KT e^{-\phi(T)},\, 2eK\lceil\phi(T)\log T\rceil e^{-\phi(T)} \right\},
\label{eq:kl-failure-factor}
\end{equation} 
set $\eta_1:=\min\{1,2Ke^{-\phi(1)}\}$, and write
\[ \Lambda_T:=\sum_{j=1}^K \log\left(1+\frac{\omega_jT}{\phi(T)}\right).\]
These two quantities are used only to state the appendix bounds.

For a fixed budget $\phi(T)>1$, let $\widehat h_{j,n}$ denote the empirical mean of the first $n$ ground-truth labels of self-confidence level $j$, extended to an infinite i.i.d.\ $\operatorname{Bernoulli}(h_j^\star)$ sequence as in Lemma~\ref{lem:regional-concentration}, and define the concentration event
\begin{equation}
\mathcal E:=\bigcap_{j=1}^K\bigcap_{n=1}^T\left\{n\kl(\widehat h_{j,n},h_j^\star)\le\phi(T)\right\}.
\label{eq:kl-good-event}
\end{equation}
Let also
\begin{equation}
g_{j,0}:=0,
\qquad
g_{j,n}:=\inf\left\{
q\in[0,\widehat h_{j,n}]:
n\kl(\widehat h_{j,n},q)\le\phi(T)
\right\},\quad n\ge1.
\label{eq:kl-lcb}
\end{equation}
Thus $\underline h_{j,t}=g_{j,N_{j,t}}$, where $N_{j,t}$ counts
completed labels available before round $t$.

\subsection{Worst-case regret of \texttt{KL-PBD}}\label{app:worst-case-ub-proof}

We first prove a refined bound that retains the Bernoulli KL geometry.
The main-text theorem follows by bounding the KL confidence width
using Pinsker's inequality.

For a Bernoulli parameter $p\in[0,1]$ and $x\ge0$, define
\begin{equation}
w_p(x):=\sup\left\{
p-q: 0\le q\le p,\ \exists z\in[0,1],\
\kl(z,p)\le x,\ \kl(z,q)\le x
\right\}.
\label{eq:kl-diameter}
\end{equation}
This function is nondecreasing in $x$ and takes values in $[0,p]$.
For $\Gamma>0$ and $d\ge0$, let
\begin{equation}
F_{p,\Gamma}(d):=\int_0^dw_p(\Gamma/u)\,du,
\label{eq:kl-integrated-width}
\end{equation}
with the integrand at zero defined by its right limit.

\begin{theorem}[Refined regret bound]
\label{thm:kl-pbd-regret-refined}
For any fixed budget $\phi(T)>1$,
\begin{equation}
\begin{aligned}
R_T(\pi^{\texttt{KL-PBD}})
\le{}&\eta_T\Delta
+c_e\sum_{j=1}^K\min\left\{
h_j^\star\omega_jT,\,
\bar{B}h_j^\star+F_{h_j^\star,\phi(T)}(\omega_jT)
\right\}
\\
&+\bar{B}c_e(1+2\Lambda_T).
\end{aligned}
\label{eq:kl-main-bound}
\end{equation}
where $\Delta:=\max\{c_e,c_{\mathrm{exp}}\}
    +\frac{c_{\mathrm{hold}}\bar B}{\mu}.$
\end{theorem}

Recall the history $\mathcal F_t$ from
Appendix~\ref{app:problem-formulation}.
Immediately before $J_t$ is observed, $B_t$, the completed-label
counts, and $\underline h_t$ are $\mathcal F_{t-1}$-measurable,
while $J_t$ is independent of $\mathcal F_{t-1}$ with distribution
$\omega$. Define
\[\mathcal E_t:=\left\{
N_{j,t}\kl(\widehat h_{j,N_{j,t}},h_j^\star)\le\phi(T)
\text{ for every $j$ with $N_{j,t}\ge1$}
\right\},
\quad
\sigma:=\inf\{t\in[T]:\mathcal E_t^c\}\wedge(T+1).\]
Then $\{t<\sigma\}\in\mathcal F_{t-1}$. Since $N_{j,t}\le T$, we have $\mathcal E_t^c\subseteq\mathcal E^c$ for every $t$, so Lemma~\ref{lem:regional-concentration} gives
\begin{equation}
\mathbb{P}(\sigma\le T)\le\mathbb{P}(\mathcal E^c)\le\frac{\eta_T}{T}.
\label{eq:pbd-failure-prob}
\end{equation}

For $t<\sigma$, let $b_{0,t}:=0$ and for $j \in [K]$, define
\begin{equation}
b_{j,t}:=c_e(h_j^\star-\underline h_{j,t}).
\label{eq:pbd-current-b}
\end{equation}

For $j \in[K],$ Lemma~\ref{lem:kl-width} gives
\begin{equation}
0\le b_{j,t}\le
\begin{cases}
c_eh_j^\star,&N_{j,t}=0,\\[1mm]
c_ew_{h_j^\star}(\phi(T)/N_{j,t}),&N_{j,t}\ge1.
\end{cases}
\label{eq:pbd-current-b-bound}
\end{equation}

\paragraph{Regret decomposition.}
Define
\begin{align*}
\mathcal R_{\rm fail}
&:=
\E\!\left[
\sum_{t=\sigma}^T
\Delta_t\1\{\sigma\le T\}
\right],
\\
\mathcal R_{\rm defer}
&:=
\E\!\left[
\sum_{t<\sigma}
b_{J_t,t}\1\{A_t=1\}
\right],
\\
\mathcal R_{\rm update}
&:=
\bar{B}c_e
\left[
1+
\E\!\left[
\sum_{t=2}^T
\1\{t<\sigma\}
\left\|
\underline h_t-\underline h_{t-1}
\right\|_\infty
\right]
\right].
\end{align*}

\begin{lemma}[Regret decomposition]
\label{lem:pbd-regret-decomposition-fresh}
The regret satisfies
\begin{equation}
\label{eq:pbd-master-decomposition-fresh}
R_T(\pi^{\texttt{KL-PBD}})
\le
\underbrace{\mathcal R_{\rm fail}}_{\mathrm{(I)}}
+
\underbrace{\mathcal R_{\rm defer}}_{\mathrm{(II)}}
+
\underbrace{\mathcal R_{\rm update}}_{\mathrm{(III)}}.
\end{equation}
\end{lemma}

\begin{proof}[Proof of Lemma~\ref{lem:pbd-regret-decomposition-fresh}]
By Lemma~\ref{lem:regret-decomposition},
\[
R_T(\pi^{\texttt{KL-PBD}})
=\mathbb E\!\left[\sum_{t<\sigma}\Delta_t\right]+\mathcal R_{\rm fail}.
\]
Thus it remains to control the regret accumulated before \(\sigma\). On these rounds, \(\underline h_t\le h^\star\) coordinatewise. Apply Lemma~\ref{lem:pbd-model-comparison} with \(\widetilde h_t=h^\star\). Then \(\widetilde\Delta_t=\Delta_t\), \(e_{j,t}=b_{j,t}\), and the reference vector does not change with \(t\). Consequently,
\[
\begin{aligned}
\mathbb E\!\left[\sum_{t<\sigma}\Delta_t\right]
&\le\mathbb E\!\left[\sum_{t<\sigma}b_{J_t,t}\mathbf1\{A_t=1\}\right]
+\bar Bc_e\left(1+\mathbb E\!\left[\sum_{t=2}^{\sigma-1}\|\underline h_t-\underline h_{t-1}\|_\infty\right]\right)\\
&=\mathcal R_{\rm defer}+\mathcal R_{\rm update}.
\end{aligned}
\]
Adding \(\mathcal R_{\rm fail}\) proves \eqref{eq:pbd-master-decomposition-fresh}.
\end{proof}

\begin{lemma}[Term (I): concentration failure]
\label{lem:pbd-failure-fresh}
We have $\mathcal R_{\rm fail}\le\eta_T\Delta$.
\end{lemma}

\begin{proof}
Lemma~\ref{lem:uniform-bellman-gap} and
\eqref{eq:pbd-failure-prob} give
\[
\mathcal R_{\rm fail}
\le T\Delta\mathbb{P}(\sigma\le T)
\le\eta_T\Delta.
\]
\end{proof}

\begin{lemma}[Term (II): deferred-round contribution]
\label{lem:pbd-deferral-fresh}
We have
\begin{equation}
\mathcal R_{\rm defer}
\le c_e\sum_{j=1}^K\min\left\{
h_j^\star\omega_jT,\,
\bar{B}h_j^\star+F_{h_j^\star,\phi(T)}(\omega_jT)
\right\}.
\label{eq:pbd-deferral-total-fresh}
\end{equation}
\end{lemma}

\begin{proof}
Fix a self-confidence level \(j\). Define \(f_j(0):=h_j^\star\) and \(f_j(u):=w_{h_j^\star}(\phi(T)/u)\) for \(u>0\). Since \(w_{h_j^\star}\) is nondecreasing and takes values in \([0,h_j^\star]\), \(f_j\) is nonnegative and nonincreasing. By \eqref{eq:pbd-current-b-bound},
\[
b_{j,t}\le c_ef_j(N_{j,t}),\qquad t<\sigma.
\]
Moreover, the definition of \(F_{p,\Gamma}\) gives \(\int_0^{\omega_jT}f_j(u)\,du=F_{h_j^\star,\phi(T)}(\omega_jT)\). Applying Lemma~\ref{lem:pbd-deferred-sum}, we obtain
\[
\begin{aligned}
\mathbb E\!\left[\sum_{t<\sigma}b_{j,t}\mathbf1\{J_t=j,A_t=1\}\right]
&\le c_e\mathbb E\!\left[\sum_{t<\sigma}\mathbf1\{J_t=j,A_t=1\}f_j(N_{j,t})\right]\\
&\le c_e\min\left\{h_j^\star\omega_jT,\ \bar Bh_j^\star+F_{h_j^\star,\phi(T)}(\omega_jT)\right\}.
\end{aligned}
\]
Summing over \(j\in[K]\) proves \eqref{eq:pbd-deferral-total-fresh}.
\end{proof}

\begin{lemma}[Term (III): Bellman-model updates]
\label{lem:pbd-update-fresh}
We have
\[
\mathcal R_{\rm update}\le \bar{B}c_e(1+2\Lambda_T).
\]
\end{lemma}

\begin{proof}
Since \(\mathbf 1\{t<\sigma\}\le1\),
Lemma~\ref{lem:pbd-lcb-variation} gives
\[
\begin{aligned}
\mathcal R_{\rm update}
&\le
\bar{B} c_e
\left[
1+\E\!\left[
\sum_{t=2}^T
\left\|\underline h_t-\underline h_{t-1}\right\|_\infty
\right]
\right]
\\
&\le
\bar{B} c_e(1+2\Lambda_T).
\end{aligned}
\]
\end{proof}

\begin{proof}[Proof of Theorem~\ref{thm:kl-pbd-regret-refined}]
Combine Lemmas~\ref{lem:pbd-regret-decomposition-fresh},
\ref{lem:pbd-failure-fresh}, \ref{lem:pbd-deferral-fresh}, and
\ref{lem:pbd-update-fresh}.
\end{proof}

\begin{proof}[Proof of Theorem~\ref{thm:pbd-regret}]
For the main-text budget $\phi(T)=\log T+\log\log T$,
Lemma~\ref{lem:regional-concentration} gives $\eta_T\le4eK\phi(T)$.
Using $h_j^\star\le1$ and
$F_{h_j^\star,\phi(T)}(\omega_jT)\le2\sqrt{2\omega_jT\phi(T)}$
in Theorem~\ref{thm:kl-pbd-regret-refined} gives
\[
\begin{aligned}
R_T(\pi^{\texttt{KL-PBD}})
\le{}&4eK\phi(T)\Delta
+c_e\sum_{j=1}^K\min\left\{
\omega_jT,\,\bar{B}+2\sqrt{2\omega_jT\phi(T)}
\right\}
\\
&+\bar{B}c_e\left[
1+2\sum_{j=1}^K\log\left(1+\frac{\omega_jT}{\phi(T)}\right)
\right].
\end{aligned}
\]
Since $\kappa=4e\Delta$, this is the main-text bound.
\end{proof}

\begin{remark}[Worst case over $\omega$]
\label{rem:ub-worst-case-omega}
Since $\sum_{j=1}^K\omega_j\le1$, Jensen's inequality gives
\begin{align*}
\Lambda_T
&=
\sum_{j=1}^K
\log\left(1+\frac{\omega_jT}{\phi(T)}\right)\\
&\le
K\log\left(
1+\frac{T}{K\phi(T)}
\sum_{j=1}^K\omega_j
\right)\\
&\le
K\log\left(1+\frac{T}{K\phi(T)}\right).
\end{align*}

Moreover,
\[
\sum_{j=1}^K\sqrt{\omega_jT}
\le
\sqrt{KT\sum_{j=1}^K\omega_j}
\le
\sqrt{KT},
\qquad
\sum_{j=1}^K\omega_jT\le T.
\]

Hence
\[
\sum_{j=1}^K
\min\left\{
\omega_jT,\,
\bar{B}+2\sqrt{2\omega_jT\phi(T)}
\right\}
\le
\min\left\{
T,\,
K\bar{B}+2\sqrt{2KT\phi(T)}
\right\},
\] and therefore
\[
\begin{aligned}
R_T(\pi^{\texttt{KL-PBD}})
\le{}&4eK\phi(T)\Delta
+c_e\min\{T,\,K\bar{B}+2\sqrt{2KT\phi(T)}\}
\\
&+\bar{B}c_e\left[
1+2K\log\left(1+\frac{T}{K\phi(T)}\right)
\right].
\end{aligned}
\]

These inequalities hold with equality when $\omega_0=0$ and $\omega_j=1/K$ for every $j\in[K]$, so this bound is the largest value of the bound of Theorem~\ref{thm:pbd-regret} over $\omega$.
\end{remark}

Alternatively, retaining the variance-sensitive term in
\eqref{eq:kl-F-explicit} gives, for any fixed budget $\phi(T)>1$,
\[
\begin{aligned}
R_T(\pi^{\texttt{KL-PBD}})
\le{}&\eta_T\Delta+\bar{B}c_e(K+1)
\\
&+4c_e\sqrt{2T\phi(T)}
\sum_{j=1}^K\sqrt{h_j^\star(1-h_j^\star)\omega_j}
+c_e(7\phi(T)+2\bar{B})\Lambda_T.
\end{aligned}
\]
The term proportional to $\phi(T)\Lambda_T$ is a statistical
remainder; it is separate from the Bellman-model update contribution.

\subsection{Gap-dependent regret}\label{app:gap-dependent-ub-proof}

We use the gap definitions from the main text. For $p\in[0,1]$ and $a>0$,
define
\begin{equation}
I_p(a):=\int_0^\infty\bigl(w_p(1/u)-a\bigr)_+\,du,
\label{eq:kl-residual-integral}
\end{equation}
with the integrand at zero defined by its right limit.

\begin{theorem}[Refined gap-dependent regret bound]
\label{thm:pbd-gap-dependent-refined}
If $\mathcal G_T=\emptyset$, the regret is zero. Otherwise, let $a_T=\gamma_T/2$. For any budget $\phi(T)>1$ held fixed throughout
the run,
\begin{equation}
\begin{aligned}
R_T(\pi^{\texttt{KL-PBD}})
\le{}&\eta_T\Delta
+2c_e\sum_{j=1}^K\min\left\{
(h_j^\star-a_T)_+\omega_jT,\,
\bar{B}(h_j^\star-a_T)_++\phi(T) I_{h_j^\star}(a_T)
\right\}
\\
&+2\bar{B}c_e(1+4\Lambda_T).
\end{aligned}
\label{eq:kl-gap-bound}
\end{equation}
\end{theorem}

\begin{proof}[Proof of Theorem~\ref{thm:pbd-gap-dependent-refined}]
If \(\mathcal G_T=\emptyset\), then
\(h_j^\star=\tau^\star_t(b)\) for every \(j \in [K]\), \(t \le T\), and $b < \bar{B}$. Hence every feasible action contributes $0$ to regret and \(R_T(\pi^{\texttt{KL-PBD}})=0\).
We therefore assume below that \(\mathcal G_T \ne \emptyset\).

Set \( a:=a_T=\frac{\gamma_T}{2}\), \(\theta:=c_ea. \)
Recall from \eqref{eq:pbd-current-b} that, for \(t<\sigma\),
\( b_{j,t}=c_e(h_j^\star-\underline h_{j,t})\ge0. \)

\paragraph{Clipped intermediate model.}
For \(t<\sigma\) and \(j \in [K]\), define
\[ b_{j,t}^{\circ}:=\min\{b_{j,t},\theta\}, \qquad h_{j,t}^{\circ}:=h_j^\star-\frac{b_{j,t}^{\circ}}{c_e}, \]
and let
\[ x_{j,t}:=b_{j,t}-b_{j,t}^{\circ} =(b_{j,t}-\theta)_+. \]

For notational convenience, set \(x_{0,t}:=0.\) Thus
\[ \underline h_t\le h_t^\circ\le h^\star, \qquad \|h_t^\circ-h^\star\|_\infty\le a. \]

For $j\in[K]$, $t\le T$, and $b<\bar{B}$, define the action margin
\[ D_t^h(j,b):=M_t^h(b)-c_e(1-h_j). \]
Under $h^\star$, the margin satisfies
\(D_t^\star(j,b)=c_e\bigl(h_j^\star-\tau^\star_t(b)\bigr).\) Since \(h_t^\circ\le h^\star\), 
Proposition~\ref{prop:bellman-stability}(ii) gives, for every $b<\bar{B},$
\[ 0\le M_t^{h_t^\circ}(b)-M_t^\star(b) \le c_e\|h_t^\circ-h^\star\|_\infty \le\theta. \]
Moreover, for every $j \in [K],$
\[
c_e(1-h_{j,t}^\circ)-c_e(1-h_j^\star)
=b_{j,t}^\circ\in[0,\theta].
\]
Consequently, for every $j \in [K]$ and $b<\bar{B},$
\[ \left|D_t^{h_t^\circ}(j,b)-D_t^\star(j,b)\right| \le\theta. \]

Whenever $D_t^\star(j,b)\ne0$, the definition of $\gamma_T$ gives $|D_t^\star(j,b)|=c_e|h_j^\star-\tau_t^\star(b)|\ge c_e\gamma_T=2\theta$.

Hence \(D_t^{h_t^\circ}(j,b)\) has the same sign as
\(D_t^\star(j,b)\) and
\[ \left|D_t^{h_t^\circ}(j,b)\right| \ge\frac12\left|D_t^\star(j,b)\right|. \]

For \(t<\sigma\), define
\[ \Delta_t^\circ :=Q_t^{h_t^\circ}(S_t,A_t)-V_t^{h_t^\circ}(S_t). \]

If $A_t$ is optimal under $h^\star$, then
$\Delta_t=0\le2\Delta_t^\circ$.
Otherwise, necessarily $J_t \in [K]$ and $B_t<\bar{B}$, because action $0$ is forced on no-query and full-queue states. The sign comparison above then
shows that $A_t$ is also suboptimal under $h_t^\circ$, and
\[ \Delta_t^\circ = \left|D_t^{h_t^\circ}(J_t,B_t)\right| \ge \frac12\left|D_t^\star(J_t,B_t)\right| = \frac12\Delta_t. \]

Therefore
\begin{equation}
\Delta_t\le2\Delta_t^\circ,
\qquad t<\sigma.
\label{eq:kl-clipped-gap-comparison}
\end{equation}

\paragraph{Regret relative to the clipped model.}
For every \(j\in[K]\) and \(t<\sigma\), the conditional accuracy vectors \(\underline h_t\) and \(h_t^\circ\) satisfy
\[
c_e(h_{j,t}^\circ-\underline h_{j,t})
=c_e(h_j^\star-\underline h_{j,t})-b_{j,t}^\circ
=b_{j,t}-b_{j,t}^\circ
=x_{j,t}.
\]
Apply Lemma~\ref{lem:pbd-model-comparison} with \(\widetilde h_{j,t}:=\max\{\underline h_{j,t},h_j^\star-a\}\). This vector is \(\mathcal F_{t-1}\)-measurable, satisfies \(\underline h_t\le\widetilde h_t\), and equals \(h_t^\circ\) for \(t<\sigma\). Thus \(e_{j,t}=x_{j,t}\) and \(\widetilde\Delta_t=\Delta_t^\circ\) on these rounds, giving
\[
\begin{aligned}
\mathbb E\!\left[\sum_{t<\sigma}\Delta_t^\circ\right]
\le{}&\mathbb E\!\left[\sum_{t<\sigma}x_{J_t,t}\mathbf1\{A_t=1\}\right]\\
&+\bar Bc_e\left(1+\mathbb E\!\left[\sum_{t=2}^{\sigma-1}\left(\|\underline h_t-\underline h_{t-1}\|_\infty+\|h_t^\circ-h_{t-1}^\circ\|_\infty\right)\right]\right).
\end{aligned}
\]

It remains to control the variation of the clipped model. The clipping map is \(1\)-Lipschitz: for every \(j\in[K]\) and \(2\le t<\sigma\),
\[
\begin{aligned}
|h_{j,t}^\circ-h_{j,t-1}^\circ|
&=\left|\max\{\underline h_{j,t},h_j^\star-a\}-\max\{\underline h_{j,t-1},h_j^\star-a\}\right|\\
&\le|\underline h_{j,t}-\underline h_{j,t-1}|.
\end{aligned}
\]
Hence \(\|h_t^\circ-h_{t-1}^\circ\|_\infty\le\|\underline h_t-\underline h_{t-1}\|_\infty\). By Lemma~\ref{lem:pbd-lcb-variation},
\[
\begin{aligned}
&\mathbb E\!\left[\sum_{t=2}^{\sigma-1}\left(\|\underline h_t-\underline h_{t-1}\|_\infty+\|h_t^\circ-h_{t-1}^\circ\|_\infty\right)\right]\\
&\qquad\le2\mathbb E\!\left[\sum_{t=2}^T\|\underline h_t-\underline h_{t-1}\|_\infty\right]
\le4\Lambda_T.
\end{aligned}
\]
It follows that
\[
\mathbb E\!\left[\sum_{t<\sigma}\Delta_t^\circ\right]
\le\mathbb E\!\left[\sum_{t<\sigma}x_{J_t,t}\mathbf1\{A_t=1\}\right]+\bar Bc_e(1+4\Lambda_T).
\]
Combining this with \eqref{eq:kl-clipped-gap-comparison},
\begin{equation}
\mathbb E\!\left[\sum_{t<\sigma}\Delta_t\right]
\le2\mathbb E\!\left[\sum_{t<\sigma}x_{J_t,t}\mathbf1\{A_t=1\}\right]+2\bar Bc_e(1+4\Lambda_T).
\label{eq:kl-clipped-master}
\end{equation}

\paragraph{Residual error on deferred rounds.}
Fix a self-confidence level \(j\). Define \(f_j(0):=(h_j^\star-a)_+\) and \(f_j(u):=(w_{h_j^\star}(\phi(T)/u)-a)_+\) for \(u>0\). This function is nonnegative and nonincreasing. Since \(\theta=c_ea\), \eqref{eq:pbd-current-b-bound} gives
\[
x_{j,t}=(b_{j,t}-\theta)_+
\le
\begin{cases}
c_e(h_j^\star-a)_+,&N_{j,t}=0,\\
c_e\bigl(w_{h_j^\star}(\phi(T)/N_{j,t})-a\bigr)_+,&N_{j,t}\ge1,
\end{cases}
\qquad t<\sigma.
\]
Thus \(x_{j,t}\le c_ef_j(N_{j,t})\). Applying Lemma~\ref{lem:pbd-deferred-sum},
\[
\begin{aligned}
\mathbb E\!\left[\sum_{t<\sigma}x_{j,t}\mathbf1\{J_t=j,A_t=1\}\right]
&\le c_e\mathbb E\!\left[\sum_{t<\sigma}\mathbf1\{J_t=j,A_t=1\}f_j(N_{j,t})\right]\\
&\le c_e\min\left\{(h_j^\star-a)_+\omega_jT,\ \bar B(h_j^\star-a)_++\int_0^{\omega_jT}f_j(u)\,du\right\}.
\end{aligned}
\]
The integral is bounded using nonnegativity and the change of variables \(u=\phi(T)v\):
\[
\begin{aligned}
\int_0^{\omega_jT}f_j(u)\,du
&\le\int_0^\infty\bigl(w_{h_j^\star}(\phi(T)/u)-a\bigr)_+\,du\\
&=\phi(T)\int_0^\infty\bigl(w_{h_j^\star}(1/v)-a\bigr)_+\,dv\\
&=\phi(T) I_{h_j^\star}(a).
\end{aligned}
\]
Substituting this bound and summing over \(j\in[K]\) gives
\[
\mathbb E\!\left[\sum_{t<\sigma}x_{J_t,t}\mathbf1\{A_t=1\}\right]
\le c_e\sum_{j=1}^K\min\left\{(h_j^\star-a)_+\omega_jT,\ \bar B(h_j^\star-a)_++\phi(T) I_{h_j^\star}(a)\right\}.
\]
Substituting into \eqref{eq:kl-clipped-master} and adding the failure contribution \(\mathcal R_{\rm fail}\le\eta_T\Delta\) from Lemma~\ref{lem:pbd-failure-fresh} proves \eqref{eq:kl-gap-bound}.
\end{proof}

\begin{proof}[Proof of Theorem~\ref{thm:pbd-gap-dependent}]
If $\mathcal G_T=\emptyset$, the refined theorem gives zero regret. Otherwise, we apply Theorem~\ref{thm:pbd-gap-dependent-refined}
with $a_T=\gamma_T/2$. We have
$(h_j^\star-a_T)_+\le1$ and, by Lemma~\ref{lem:kl-residual-width},
$I_{h_j^\star}(a_T)\le2/a_T=4/\gamma_T$.
For $\phi(T)=\log T+\log\log T$,
Lemma~\ref{lem:regional-concentration} gives $\eta_T\le4eK\phi(T)$.
Consequently,
\[
\begin{aligned}
R_T(\pi^{\texttt{KL-PBD}})
\le{}&4eK\phi(T)\Delta
+2c_e\sum_{j=1}^K\min\left\{
\omega_jT,\,\bar{B}+\frac{4\phi(T)}{\gamma_T}
\right\}
\\
&+2\bar{B}c_e\left[
1+4\sum_{j=1}^K\log\left(1+\frac{\omega_jT}{\phi(T)}\right)
\right],
\end{aligned}
\]
which proves the main-text theorem since $\kappa=4e\Delta$.
\end{proof}

In particular, the main-text bound implies
\[
R_T(\pi^{\texttt{KL-PBD}})
\le 2\bar{B}c_e(K+1)
+K\phi(T)\left(\frac{8c_e}{\gamma_T}+4e\Delta\right)
+8\bar{B}c_e\Lambda_T.
\]
When the gap is bounded away from zero, the regret remains logarithmic in $T$. The failure contribution now has the same logarithmic order, but does not worsen the inverse-gap dependence. The refined theorem and \eqref{eq:kl-I-explicit} retain the variance-sensitive improvement.

\subsection{Technical lemmas}
\label{app:ub-technical}

\subsubsection{Concentration and KL confidence bounds}

\begin{lemma}[Concentration within each self-confidence level under delayed feedback]
\label{lem:regional-concentration}
For every $j\in[K]$, the observed-label sequence can be extended to an infinite i.i.d.\ $\operatorname{Bernoulli}(h_j^\star)$ sequence. For any fixed budget $\phi(T)>1$, the event $\mathcal E$ of \eqref{eq:kl-good-event} satisfies $\mathbb{P}(\mathcal E^c)\le\eta_T/T$. 
For \(T\ge3\) and \(\phi(T)=\log T+\log\log T\), we also have
\begin{equation}
\eta_T\le4eK\phi(T).
\label{eq:kl-default-failure}
\end{equation}
\end{lemma}

\begin{proof}
Conditional on the self-confidence level sequence, the ground-truth labels are independent, with $Y_t\sim\operatorname{Bernoulli}(h_j^\star)$
when $J_t=j$. Since deferral and completion decisions use only self-confidence levels, service information and previously revealed labels, conditional on the history immediately before a level-$j$ label is revealed, that label has distribution $\operatorname{Bernoulli}(h_j^\star)$. Revealing completed labels in submission order within each round, and appending independent $\operatorname{Bernoulli}(h_j^\star)$ labels after the observed sequence for level $j$, gives the claimed infinite i.i.d.\ sequence \citep[Section~4.6]{lattimore2020bandit}.

For a Bernoulli sample mean, Chernoff's inequality and a union bound give
\[
 \mathbb{P}(\mathcal E^c)\le2KT e^{-\phi(T)}.
\]
For $T\ge2$, applying the Bernoulli maximal inequality
\citep[equation (5)]{garivier2013informational}
to the ground-truth label sequence for each self-confidence level
and taking a union bound over levels gives
\begin{equation}
 \mathbb{P}(\mathcal E^c)
 \le2eK\lceil\phi(T)\log T\rceil e^{-\phi(T)},\qquad \phi(T)>1.
 \label{eq:kl-maximal}
\end{equation}
If $h_j^\star\in\{0,1\}$, then $\widehat h_{j,n}=h_j^\star$ almost surely for every $n \ge 1$, so the KL concentration event holds automatically for that $j$.

Combining these bounds with $\mathbb{P}(\mathcal E^c)\le1$ gives
$\mathbb{P}(\mathcal E^c)\le\eta_T/T$ for $T\ge2$. For $T=1$, the same conclusion follows from the Chernoff union bound and the definition of $\eta_1$. Picking $\phi(T)=\log{T}+\log\log{T}$ gives, using $T e^{-\phi(T)}=(\log{T})^{-1}$,
\[
\eta_T
\le 2eK\frac{\lceil\phi(T)\log{T}\rceil}{\log{T}}
\le 2eK\left(\phi(T)+\frac{1}{\log{T}}\right)
\le4eK\phi(T).
\]
The last inequality uses $\log{T}>1$ and $\phi(T)>1$.
\end{proof}
\begin{lemma}[KL diameter and its integral bounds]
\label{lem:kl-width}
For every $p\in[0,1]$ and $x\ge0$,
\begin{equation}
 w_p(x)\le\min\{p,\sqrt{2x}\}.
 \label{eq:kl-rho-bound}
\end{equation}
Whenever $n\kl(\widehat h_{j,n},h_j^\star)\le\phi(T)$ with $n\ge1$,
\begin{equation}
 0\le h_j^\star-g_{j,n}
 \le w_{h_j^\star}(\phi(T)/n).
 \label{eq:kl-width-good}
\end{equation}
Moreover, $F_{p,\Gamma}$ is nondecreasing and concave, and
\begin{equation}
F_{p,\Gamma}(d)\le
\min\left\{pd,\,2\sqrt{2\Gamma d}\right\}.
\label{eq:kl-F-explicit}
\end{equation}
\end{lemma}

\begin{proof}
The full confidence set
$\{q\in[0,1]:n\kl(\widehat h_{j,n},q)\le\phi(T)\}$
is nonempty and compact. Its minimum is $g_{j,n}$, since it contains
$\widehat h_{j,n}$. Under the stated KL constraint, it also contains
$h_j^\star$. Thus $g_{j,n}\le h_j^\star$, and the definition of
$w_{h_j^\star}$ gives \eqref{eq:kl-width-good}.

The bound $w_p(x)\le p$ is immediate. Pinsker's inequality applied to each of the two KL constraints gives
$|z-p|,|z-q|\le\sqrt{x/2}$, hence $p-q\le\sqrt{2x}$, proving
\eqref{eq:kl-rho-bound}.

As $w_p$ is nondecreasing, the integrand defining $F_{p,\Gamma}$ is nonincreasing in $u$, so $F_{p,\Gamma}$ is nondecreasing and concave. Integrating the two bounds in \eqref{eq:kl-rho-bound} gives \eqref{eq:kl-F-explicit}.
\end{proof}

\begin{lemma}[Pathwise total variation]
\label{lem:pbd-lcb-variation}
For any binary sequence and any fixed $\Gamma>0$, let
$g_0=0$ and let $g_n$ be its Bernoulli KL lower confidence bound with budget
$\Gamma$. For every integer $N\ge1$,
\begin{equation}
 \sum_{n=1}^N|g_n-g_{n-1}|
 \le-2\log(1-e^{-\Gamma/N})
 \le2\log(1+N/\Gamma).
 \label{eq:kl-scalar-tv}
\end{equation}
For \texttt{KL-PBD}, on every sample path,
\begin{equation}
 \sum_{t=2}^T\|\underline h_t-\underline h_{t-1}\|_\infty
 \le2\sum_{j=1}^K\log\left(1+\frac{N_{j,T}}{\phi(T)}\right).
 \label{eq:pbd-lcb-variation-pathwise}
\end{equation}
Consequently,
\begin{equation}
\E\left[\sum_{t=2}^T
\|\underline h_t-\underline h_{t-1}\|_\infty\right]
\le2\Lambda_T.
\label{eq:pbd-lcb-variation}
\end{equation}
\end{lemma}

\begin{proof}
We interpolate the number of successes $s$ and failures $f$ continuously.
For $s>0$, put $m=s+f$, $z=s/m$, and let $q=q(s,f)\in(0,z)$ solve
\[
 m\kl(z,q)=\Gamma.
\]
For $s=0$, set $q(0,f)=0$, and set $q(0,0)=0$. These definitions agree
continuously at the boundaries relevant below: $q\le z$ gives continuity
as $s\downarrow0$ with $f>0$, while
$\kl(z,q)\le\kl(1,q)=-\log q$ gives $q\le e^{-\Gamma/m}$ as
$m\downarrow0$. In particular, $q(s,0)=e^{-\Gamma/s}$ for $s>0$.

Consider first adding a success, keeping $f$ fixed. In the interior,
implicit differentiation gives
\[
 \frac{\partial}{\partial s}\big[m\kl(z,q)\big]_{q\text{ fixed}}
 =\log(z/q),
 \qquad
 \frac{\partial}{\partial q}\big[m\kl(z,q)\big]
 =-\frac{m(z-q)}{q(1-q)}.
\]
Writing $a=\log(z/q)>0$ and $x=\Gamma/m$, we obtain
\begin{equation}
 \frac{\partial q}{\partial s}
 =\frac{1-q}{m}\frac{a}{e^a-1}
 \le\frac{\Gamma}{m^2(e^{\Gamma/m}-1)}.
 \label{eq:kl-success-derivative}
\end{equation}
Indeed,
$x=za+(1-z)\log((1-z)/(1-q))\le a$, and
$a\mapsto a/(e^a-1)$ is decreasing on $(0,\infty)$.
The same inequality holds when $f=0$, either directly from
$q(s,0)=e^{-\Gamma/s}$ or by a limit.

Adding a failure while keeping $s$ fixed gives instead the partial
derivative $\log((1-z)/(1-q))\le0$ at fixed $q$. Implicit
differentiation therefore gives $\partial q/\partial f\le0$.
Continuity extends this monotonicity to $f=0$ and $s=0$.
Thus successes can only increase the lower confidence bound, and failures
can only decrease it.

We interpolate each observation over its own interval $m\in[n-1,n]$.
The total upward variation is at most the integral of the right-hand
side of \eqref{eq:kl-success-derivative} over all intervals:
\[
 \sum_{n=1}^N(g_n-g_{n-1})_+
 \le\int_0^N\frac{\Gamma}{m^2(e^{\Gamma/m}-1)}\,dm
 =-\log(1-e^{-\Gamma/N}).
\]
By continuity of the lower confidence bound at $m=0$ and $s=0$, the integrals may start at $0$.
Since $g_0=0$ and $g_N\ge0$, downward variation is no larger than upward variation.
Finally, $e^x\ge1+x$ implies
$-\log(1-e^{-x})\le\log(1+1/x)$, proving \eqref{eq:kl-scalar-tv}.

Fix a self-confidence level \(j\). Since \(N_{j,t}\) is nondecreasing and
\(\underline h_{j,t}=g_{j,N_{j,t}}\), the triangle inequality gives
\[
\sum_{t=2}^T
\left|\underline h_{j,t}-\underline h_{j,t-1}\right|
\le
\sum_{n=1}^{N_{j,T}}|g_{j,n}-g_{j,n-1}|.
\]
Applying \eqref{eq:kl-scalar-tv}, with an empty sum when
\(N_{j,T}=0\), and using
\[
\left\|\underline h_t-\underline h_{t-1}\right\|_\infty
\le
\sum_{j=1}^K
\left|\underline h_{j,t}-\underline h_{j,t-1}\right|
\]
proves \eqref{eq:pbd-lcb-variation-pathwise}.

It remains to take expectations. Since \(N_{j,T}\) counts completed
self-confidence level  \(j\) labels available before round \(T\), it is bounded by
the number of self-confidence level \(j\) arrivals before that round. Thus
\(\E[N_{j,T}]\le\omega_jT\). By concavity of
\(x\mapsto\log(1+x)\),
\[
\begin{aligned}
\E\!\left[
\sum_{t=2}^T
\left\|\underline h_t-\underline h_{t-1}\right\|_\infty
\right]
&\le
2\sum_{j=1}^K
\E\!\left[\log\left(1+\frac{N_{j,T}}{\phi(T)}\right)\right]
\\
&\le
2\sum_{j=1}^K
\log\left(1+\frac{\omega_jT}{\phi(T)}\right)
=2\Lambda_T.
\end{aligned}
\]
\end{proof}
\begin{lemma}[Integrated residual width]
\label{lem:kl-residual-width}
If $p\le a$, then $I_p(a)=0$. For $p>a>0$,
\begin{equation}
I_p(a)\le\min\left\{
\frac2a,\,
8p(1-p)\left(\frac1a-\frac1p\right)+12\log\frac pa
\right\}.
\label{eq:kl-I-explicit}
\end{equation}
\end{lemma}

\begin{proof}
For $p\le a$, $I_p(a)=0$. Pinsker's bound gives, for $a>0$,
\[  I_p(a)\le\int_0^\infty(\sqrt{2/u}-a)_+\,du=\frac2a. \]
For the other bound, assume $p>a$ and put $A=2\sqrt{2p(1-p)}$.
By Tonelli's theorem,
\begin{equation}
 I_p(a)=\int_a^p
 \operatorname{Leb}\{u>0:w_p(1/u)>r\}\,dr.
 \label{eq:kl-layer-cake}
\end{equation}
If \(w_p(1/u)>r\), then \eqref{eq:kl-rho-bound} gives \(r<A/\sqrt u+6/u\). By \(2xy\le x^2+y^2\), applied with \(x=\sqrt r\) and \(y=A/\sqrt{ru}\),
\[
\frac{A}{\sqrt u}\le\frac r2+\frac{A^2}{2ru}.
\]
Consequently,
\[
r<\frac r2+\frac{A^2}{2ru}+\frac6u,
\qquad\text{so}\qquad
u<\frac{A^2}{r^2}+\frac{12}{r}.
\]
Thus
\[
\operatorname{Leb}\{u>0:w_p(1/u)>r\}\le\frac{A^2}{r^2}+\frac{12}{r}.
\]
This argument also applies when \(A=0\). Substituting this length bound into \eqref{eq:kl-layer-cake} yields
\[  I_p(a)\le 8p(1-p)\left(\frac1a-\frac1p\right)+12\log\frac pa, \]
as claimed. In particular, all integrals used in the theorems are finite.
\end{proof}

\subsubsection{Service-capacity stability bounds}\label{app:service-capacity-bounds}

Recall from Appendix~\ref{app:problem-formulation} the queue value
\(U_t^h(b):=\sum_{j=0}^K\omega_j V_t^h((b,j)).
\)
Since $V_{T+1}^h((b,j))=D(b)$ for every $j=0,\ldots,K$ and
$\sum_{j=0}^K\omega_j=1$, we have $U_{T+1}^h=D$, independently of $h$, and the Bellman recursion can be written as
\[
U_t^h(b)=c_{\mathrm{hold}}b+\sum_{j=0}^K\omega_j\min_{a\in\mathcal A(b,j)}\left\{c_{\mathrm{exp}}a+L^h(j)(1-a)+\mathbb E\!\left[U_{t+1}^h((b+a-C_t)^+)\right]\right\}.
\]

We first show that the service-capacity structure yields a stability property useful to control the effect of a
pessimistic conditional accuracy vector on the Bellman recursion.

\begin{proposition}[Bellman stability]
\label{prop:bellman-stability}
Let \(h,\widetilde h\in[0,1]^K\). For every \(t\le T\):

\begin{enumerate}
    \item[(i)] For arbitrary \(h,\widetilde h\),
    \(\operatorname{sp} \left( U_t^h-U_t^{\widetilde h} \right) \le \bar{B}c_e\|h-\widetilde h\|_\infty\), and \(
    \left| M_t^h(b)-M_t^{\widetilde h}(b) \right|   \le  c_e\|h-\widetilde h\|_\infty \) for every \(b<\bar{B}\).

    \item[(ii)] If \(h\le\widetilde h\) coordinatewise, then
    \(U_t^h-U_t^{\widetilde h}\) is nondecreasing in the queue state and,
    for every \(b<\bar{B}\),
    \[0 \le  M_t^h(b)-M_t^{\widetilde h}(b)  \le c_e\|h-\widetilde h\|_\infty. \]
\end{enumerate}
\end{proposition}

Part~(i) controls changes in the Bellman quantities under arbitrary
conditional accuracy perturbations, while part~(ii) shows that a pessimistic
conditional accuracy vector increases the Bellman marginal deferral costs and
induces an ordered change in the queue values.

\begin{proof}[Proof of Proposition~\ref{prop:bellman-stability}]
Let \(h,\widetilde h\in[0,1]^K\) and set \(\varepsilon:=c_e\|h-\widetilde h\|_\infty\). Recall from Appendix~\ref{app:problem-formulation} the notations \(x_t^h(q)=U_t^h(q)-U_t^h(q-1)\) and \(d_t^h(b)=F_t^h(b+1)-F_t^h(b)\). We prove both statements by a common backward induction. For arbitrary \(h,\widetilde h\), we show that
\begin{equation}
\left|x_t^h(q)-x_t^{\widetilde h}(q)\right|\le\varepsilon,
\label{eq:two-sided-increment-stability}
\end{equation}
and, if \(h\le\widetilde h\) coordinatewise, that
\begin{equation}
0\le x_t^h(q)-x_t^{\widetilde h}(q)\le\varepsilon,
\label{eq:one-sided-increment-stability}
\end{equation}
for every \(t\le T+1\) and \(q=1,\ldots,\bar B\).

At \(t=T+1\), both queue-value functions are equal. Hence \(x_{T+1}^h(q)=x_{T+1}^{\widetilde h}(q)\) for every \(q=1,\ldots,\bar B\), so both claims hold. Suppose they hold at time \(t+1\). By \eqref{eq:service-smoothed-increment}, for every \(b<\bar B\),
\[
d_t^h(b)-d_t^{\widetilde h}(b)
=\sum_{c=0}^b\rho_C(c)\left[x_{t+1}^h(b+1-c)-x_{t+1}^{\widetilde h}(b+1-c)\right].
\]
The weights are nonnegative and sum to at most one. Therefore,
\[
\left|d_t^h(b)-d_t^{\widetilde h}(b)\right|\le\varepsilon,
\]
and, when \(h\le\widetilde h\),
\[
0\le d_t^h(b)-d_t^{\widetilde h}(b)\le\varepsilon.
\]

Recall the notation \(r_j^h=L^h(j)-c_{\mathrm{exp}}\). For \(j\in[K]\), we have \(r_j^h-r_j^{\widetilde h}=c_e(\widetilde h_j-h_j)\), while \(r_0^h=r_0^{\widetilde h}=-c_{\mathrm{exp}}\). Thus
\(|r_j^h-r_j^{\widetilde h}|\le\varepsilon\) for $j \in \{0\} \cup [K]$
and these differences are nonnegative when \(h\le\widetilde h\).

For \(1\le q<\bar B\), the preceding bounds and \eqref{eq:service-increment-recursion}  give
\[
\begin{aligned}
|x_t^h(q)-x_t^{\widetilde h}(q)|
&\le\sum_{j=0}^K\omega_j
\left|\operatorname{med}\{d_t^h(q-1),d_t^h(q),r_j^h\}
-\operatorname{med}\{d_t^{\widetilde h}(q-1),d_t^{\widetilde h}(q),r_j^{\widetilde h}\}\right|\\
&\le\sum_{j=0}^K\omega_j\varepsilon=\varepsilon,
\end{aligned}
\]
where we used that for any \(a,\widetilde a\in\mathbb R^3\),
\(
\left|\operatorname{med}\{a_1,a_2,a_3\}-\operatorname{med}\{\widetilde a_1,\widetilde a_2,\widetilde a_3\}\right|
\le\max_{i=1,2,3}|a_i-\widetilde a_i|.
\)
At \(q=\bar B\), the same conclusion follows from \eqref{eq:service-increment-boundary} and
\(|\max\{a_1,a_2\}-\max\{\widetilde a_1,\widetilde a_2\}|
\le\max_{i=1,2}|a_i-\widetilde a_i|.\)
This proves \eqref{eq:two-sided-increment-stability} at time \(t\).

If \(h\le\widetilde h\), each argument of the median or maximum under \(h\) is at least the corresponding argument under \(\widetilde h\). Since increasing any argument cannot decrease either its median or its maximum, \eqref{eq:service-increment-recursion} and \eqref{eq:service-increment-boundary} also give
\(x_t^h(q)\ge x_t^{\widetilde h}(q)\) for  $q=1,\ldots,\bar B$.
Combining this with the absolute bound proves \eqref{eq:one-sided-increment-stability} at time \(t\), completing the common induction.

Now define \(Z_t(b):=U_t^h(b)-U_t^{\widetilde h}(b)\). By \eqref{eq:two-sided-increment-stability},
\(|Z_t(b)-Z_t(b-1)|\le\varepsilon\)  for $b=1,\ldots,\bar B.$
Hence, \[\operatorname{sp}(Z_t)\le\bar B\varepsilon=\bar Bc_e\|h-\widetilde h\|_\infty.\]
Moreover, by \eqref{eq:service-bellman-marginal},
\(M_t^h(b)-M_t^{\widetilde h}(b)=d_t^h(b)-d_t^{\widetilde h}(b),\)
so
\[|M_t^h(b)-M_t^{\widetilde h}(b)|\le c_e\|h-\widetilde h\|_\infty,\qquad b<\bar B.\]
This proves part~(i).

If \(h\le\widetilde h\), \eqref{eq:one-sided-increment-stability} shows that \(Z_t\) is nondecreasing in the queue state. The corresponding bound on the smoothed increments also gives
\[
0\le M_t^h(b)-M_t^{\widetilde h}(b)\le c_e\|h-\widetilde h\|_\infty,\qquad b<\bar B.
\]
This proves part~(ii).
\end{proof}
We next record a uniform bound on the gaps $\Delta_t$.

\begin{lemma}[Uniform gap bound]
\label{lem:uniform-bellman-gap}
For every conditional accuracy vector $h$, every $t\le T+1$,
and every $q=1,\ldots,\bar B$,
\begin{equation}
\label{eq:service-adjacent-value-bound}
0
\le
U_t^h(q)-U_t^h(q-1)
\le
\frac{q\,c_{\mathrm{hold}}}{\mu}
+
(c_e-c_{\mathrm{exp}})_+.
\end{equation}
Consequently, for every $t\le T$ and every $b<\bar B$,
\begin{equation}
\label{eq:service-M-uniform-bound}
c_{\mathrm{exp}}
\le
M_t^h(b)
\le
\max\{c_e,c_{\mathrm{exp}}\}
+
\frac{c_{\mathrm{hold}}\bar{B}}{\mu}.
\end{equation}
In particular, \(0\le\Delta_t\le\Delta\) on every sample path.
\end{lemma}

\begin{proof}
The lower bound in \eqref{eq:service-adjacent-value-bound} follows from
Lemma~\ref{lem:service-queue-increments}. For the upper bound, fix
\(t\) and \(q\ge1\), and couple two systems starting from queue states
\(q\) and \(q-1\), respectively, using the same self-confidence level and service realizations through round \(T\), and the same service realizations
thereafter. Let the lower system follow an optimal policy.

To account for the terminal value, we regard both systems as continuing
after round \(T\) with no further arrivals and incurring holding costs
until their queues are empty; by definition, the expected cost of this
continuation from queue state \(b\) is \(D(b)\).

While the queues differ, they differ by exactly one job. If the lower system takes action $0$, let the upper system also take action $0$. If the lower system defers, let the upper system take action $0$ instead. In the latter case, the two pre-service queue states after the deferral decision coincide, and hence the queues couple at the end of that round. Once they have coupled, let the upper system follow the same actions as the lower system.

Before coupling, the difference in holding cost is \(c_{\mathrm{hold}}\)
per round. If the lower system never triggers coupling by deferring,
the queues couple once \(q\) rounds with \(C_t\ge1\) have occurred.
The expected time required for \(q\) such service opportunities is
\(q/\mu\). If coupling instead occurs because the lower system defers,
the only additional non-holding cost incurred by the upper system on
that round is \(c_e(1-h_{J_t})-c_{\mathrm{exp}} \le (c_e-c_{\mathrm{exp}})_+. \)
This can occur at most once. Thus the constructed algorithm for the upper
system has expected cost at most
\[U_t^h(q-1) + \frac{q\,c_{\mathrm{hold}}}{\mu} + (c_e-c_{\mathrm{exp}})_+. \] Optimality of \(U_t^h(q)\) proves
\eqref{eq:service-adjacent-value-bound}.

By \eqref{eq:service-smoothed-increment},
\[ d_t^h(b) = \sum_{c=0}^b \rho_C(c)\, x_{t+1}^h(b+1-c). \]
Using \eqref{eq:service-adjacent-value-bound} and \(b+1\le \bar{B}\),
\[ 0 \le d_t^h(b) \le \frac{c_{\mathrm{hold}}\bar{B}}{\mu} + (c_e-c_{\mathrm{exp}})_+. \]
Together with \eqref{eq:service-bellman-marginal}, this yields
\eqref{eq:service-M-uniform-bound}.

Finally, for $j\in[K]$, $c_e(1-h_j^\star)\in[0,c_e]$, while
\eqref{eq:service-M-uniform-bound} gives
$0\le M_t^\star(b)\le\Delta$. Hence, for $b<\bar{B}$,
\[ \left| c_e(1-h_j^\star)-M_t^\star(b) \right| \le \Delta. \]
On no-query or full-queue rounds $\Delta_t=0$.
Thus the definition of $\Delta_t$ gives
$0\le\Delta_t\le\Delta$ on every sample path.

\end{proof}

\subsubsection{Model comparison and pending labels}

\begin{lemma}[Bellman comparison]
\label{lem:pbd-model-comparison}
Let \(\widetilde h_t\in[0,1]^K\) be \(\mathcal F_{t-1}\)-measurable and satisfy \(\underline h_t\le\widetilde h_t\) coordinatewise for \(t<\sigma\). For \(j\in[K]\), define \(e_{j,t}:=c_e(\widetilde h_{j,t}-\underline h_{j,t})\), set \(e_{0,t}:=0\), and let
\[
\widetilde\Delta_t:=Q_t^{\widetilde h_t}(S_t,A_t)-V_t^{\widetilde h_t}(S_t)
\]
be the decision gap of the learner's action under \(\widetilde h_t\). Then
\begin{align}
\mathbb E\!\left[\sum_{t<\sigma}\widetilde\Delta_t\right]
\le{}&\mathbb E\!\left[\sum_{t<\sigma}e_{J_t,t}\mathbf1\{A_t=1\}\right]\nonumber\\
&+\bar Bc_e\left(1+\mathbb E\!\left[\sum_{t=2}^{\sigma-1}\left(\|\underline h_t-\underline h_{t-1}\|_\infty+\|\widetilde h_t-\widetilde h_{t-1}\|_\infty\right)\right]\right).
\label{eq:pbd-model-comparison}
\end{align}
\end{lemma}

\begin{proof}
Fix \(t<\sigma\). For \(t\le s\le T+1\), define \(Z_s^t(b):=U_s^{\underline h_t}(b)-U_s^{\widetilde h_t}(b)\). By Proposition~\ref{prop:bellman-stability}(ii), \(Z_s^t\) is nondecreasing in \(b\). Define the centered potential \(\Phi_s^t(b):=Z_s^t(\bar B)-Z_s^t(b)\ge0\), and write \(\bar e_t:=\sum_{j=1}^K\omega_je_{j,t}\).

At queue state \(\bar B\), action \(0\) is the only feasible action. Therefore, for \(t\le s\le T\),
\[
Z_s^t(\bar B)
=\bar e_t+\mathbb E\!\left[Z_{s+1}^t((\bar B-C_s)^+)\mid\mathcal F_{t-1}\right]
\le\bar e_t+Z_{s+1}^t(\bar B),
\]
and hence
\begin{equation}
Z_s^t(\bar B)-\bar e_t\le Z_{s+1}^t(\bar B).
\label{eq:pbd-potential-top}
\end{equation}

Since the \texttt{KL-PBD} action at round \(t\) is optimal under \(\underline h_t\), the definitions of \(Q_t^h\) and \(V_t^h\) give
\begin{align}
\widetilde\Delta_t
={}&V_t^{\underline h_t}(S_t)-V_t^{\widetilde h_t}(S_t)-e_{J_t,t}(1-A_t)\nonumber\\
&-\mathbb E\!\left[Z_{t+1}^t(B_{t+1})\mid\mathcal F_{t-1},J_t,A_t\right].
\label{eq:pbd-fresh-one-step}
\end{align}
Because \(J_t\) is independent of \(\mathcal F_{t-1}\) with distribution \(\omega\),
\[
\mathbb E[e_{J_t,t}\mid\mathcal F_{t-1}]
=\sum_{j=0}^K\omega_je_{j,t}=\bar e_t,
\]
and
\[
\mathbb E\!\left[V_t^{\underline h_t}(S_t)-V_t^{\widetilde h_t}(S_t)\mid\mathcal F_{t-1}\right]
=Z_t^t(B_t).
\]
Taking conditional expectation in \eqref{eq:pbd-fresh-one-step} therefore yields
\[
\begin{aligned}
\mathbb E[\widetilde\Delta_t\mid\mathcal F_{t-1}]
={}&Z_t^t(B_t)-\bar e_t
+\mathbb E\!\left[e_{J_t,t}\mathbf1\{A_t=1\}\mid\mathcal F_{t-1}\right]\\
&-\mathbb E\!\left[Z_{t+1}^t(B_{t+1})\mid\mathcal F_{t-1}\right].
\end{aligned}
\]
Using \(Z_t^t(B_t)=Z_t^t(\bar B)-\Phi_t^t(B_t)\) and \eqref{eq:pbd-potential-top}, we obtain
\begin{align}
\mathbb E[\widetilde\Delta_t\mid\mathcal F_{t-1}]
\le{}&\mathbb E\!\left[e_{J_t,t}\mathbf1\{A_t=1\}\mid\mathcal F_{t-1}\right]\nonumber\\
&+\mathbb E\!\left[\Phi_{t+1}^t(B_{t+1})\mid\mathcal F_{t-1}\right]-\Phi_t^t(B_t).
\label{eq:pbd-fresh-one-step-potential}
\end{align}

Since \(\{t<\sigma\}\in\mathcal F_{t-1}\), we may multiply \eqref{eq:pbd-fresh-one-step-potential} by \(\mathbf1\{t<\sigma\}\), take expectations, and sum over \(t\). Let \(m:=\min\{T,\sigma-1\}\). On paths with \(m\ge1\),
\[
\begin{aligned}
\sum_{t=1}^m\left[\Phi_{t+1}^t(B_{t+1})-\Phi_t^t(B_t)\right]
={}&-\Phi_1^1(B_1)+\Phi_{m+1}^m(B_{m+1})\\
&+\sum_{t=2}^m\left[\Phi_t^{t-1}(B_t)-\Phi_t^t(B_t)\right].
\end{aligned}
\]
The first term is nonpositive. If \(m=T\), then \(Z_{T+1}^T=D-D=0\), and therefore \(\Phi_{T+1}^T\equiv0\). Otherwise, \(m+1=\sigma\), and Proposition~\ref{prop:bellman-stability}(i) gives
\[
0\le\Phi_\sigma^{\sigma-1}(B_\sigma)
\le\bar Bc_e\|\underline h_{\sigma-1}-\widetilde h_{\sigma-1}\|_\infty
\le\bar Bc_e.
\]

It remains to control the change of potential when the conditional accuracy vectors are updated. Successive vectors need not be coordinatewise ordered, so here we use the two-sided stability in Proposition~\ref{prop:bellman-stability}(i). For \(2\le t<\sigma\),
\[
\begin{aligned}
|\Phi_t^{t-1}(B_t)-\Phi_t^t(B_t)|
&=\left|(Z_t^{t-1}-Z_t^t)(\bar B)-(Z_t^{t-1}-Z_t^t)(B_t)\right|\\
&\le\operatorname{sp}(Z_t^{t-1}-Z_t^t)\\
&\le\operatorname{sp}\left(U_t^{\underline h_{t-1}}-U_t^{\underline h_t}\right)
+\operatorname{sp}\left(U_t^{\widetilde h_{t-1}}-U_t^{\widetilde h_t}\right)\\
&\le\bar Bc_e\left(\|\underline h_t-\underline h_{t-1}\|_\infty+\|\widetilde h_t-\widetilde h_{t-1}\|_\infty\right).
\end{aligned}
\]
Here the third line uses subadditivity of the span. Combining these bounds controls the potential sum by
\[
\bar Bc_e\left(1+\sum_{t=2}^{\sigma-1}\left(\|\underline h_t-\underline h_{t-1}\|_\infty+\|\widetilde h_t-\widetilde h_{t-1}\|_\infty\right)\right).
\]
On paths with \(m=0\), the potential sum is empty and the same upper bound holds. Taking expectations and substituting into the summed one-step inequality proves \eqref{eq:pbd-model-comparison}.
\end{proof}

\begin{lemma}[Pending-label bound for deferred queries]
\label{lem:pbd-deferred-sum}
For \(j\in[K]\) and any nonnegative nonincreasing function \(f:[0,\infty)\to[0,\infty)\),
\[
\mathbb E\!\left[\sum_{t<\sigma}\mathbf1\{J_t=j,A_t=1\}f(N_{j,t})\right]
\le\min\left\{f(0)\omega_jT,\ \bar Bf(0)+\int_0^{\omega_jT}f(u)\,du\right\}.
\]
\end{lemma}

\begin{proof}
Fix a self-confidence level \(j\), write \(D_j:=D_{j,\sigma-1}\), and enumerate its deferrals before \(\sigma\) as \(t_{j,1},\ldots,t_{j,D_j}\). Immediately before the \(\ell\)-th such deferral, \(\ell-1\) previous self-confidence level \(j\) jobs have been submitted. Since deferral is feasible at \(t_{j,\ell}\), at most \(\bar B-1\) of them can still be pending. Hence
\begin{equation}
N_{j,t_{j,\ell}}\ge(\ell-\bar B)_+.
\label{eq:kl-delayed-count}
\end{equation}

For the first at most \(\bar B\) deferrals, use \(f(N_{j,t_{j,\ell}})\le f(0)\). For \(\ell>\bar B\), monotonicity of \(f\) and \eqref{eq:kl-delayed-count} give \(f(N_{j,t_{j,\ell}})\le f(\ell-\bar B)\). Therefore,
\[
\sum_{\ell=1}^{D_j}f(N_{j,t_{j,\ell}})
\le f(0)\min\{D_j,\bar B\}+\sum_{n=1}^{(D_j-\bar B)_+}f(n).
\]
Since \(f\) is nonincreasing, \(f(n)\le\int_{n-1}^nf(u)\,du\). Thus
\[
\begin{aligned}
\sum_{\ell=1}^{D_j}f(N_{j,t_{j,\ell}})
&\le f(0)\min\{D_j,\bar B\}+\int_0^{(D_j-\bar B)_+}f(u)\,du\\
&\le\bar Bf(0)+\int_0^{D_j}f(u)\,du.
\end{aligned}
\]

Every such deferral occurs on a self-confidence level \(j\) arrival, so \(\mathbb E[D_j]\le\omega_jT\). The map \(d\mapsto\int_0^df(u)\,du\) is nondecreasing and concave because \(f\) is nonnegative and nonincreasing. Consequently, Jensen's inequality gives
\[
\begin{aligned}
\mathbb E\!\left[\sum_{\ell=1}^{D_j}f(N_{j,t_{j,\ell}})\right]
&\le\bar Bf(0)+\mathbb E\!\left[\int_0^{D_j}f(u)\,du\right]\\
&\le\bar Bf(0)+\int_0^{\mathbb E[D_j]}f(u)\,du\\
&\le\bar Bf(0)+\int_0^{\omega_jT}f(u)\,du.
\end{aligned}
\]
On the other hand, \(f(N_{j,t})\le f(0)\) gives
\[
\mathbb E\!\left[\sum_{\ell=1}^{D_j}f(N_{j,t_{j,\ell}})\right]
\le f(0)\mathbb E[D_j]\le f(0)\omega_jT.
\]
Taking the better of these two bounds proves the claim.
\end{proof}

\subsection{Extension to unknown arrival and service distributions}\label{app:unknown-distributions}

The arrival distribution $\omega$ and the truncated service distribution defined below can be learned without additional deferrals when their realizations are observed independently of the deferral decisions. In this section, we show that using \texttt{KL-PBD} with plug-in empirical estimates of these distributions preserves the worst-case regret rate of Theorem~\ref{thm:pbd-regret}, including its dependence on the arrival distribution $\omega$. Throughout this section, assume $c_{\mathrm{hold}}>0$. For unknown $\rho_C$, additionally assume that
$\bar C_t:=\min\{C_t,\bar B\}$ is observed at the end of every round,
including rounds with an empty queue. Write $\bar\rho_c:=\mathbb{P}(\bar C_t=c)$ for
$c=0,\ldots,\bar B$, so $\mu=1-\bar\rho_0>0$, and set
$\Omega:=\mathbb{P}(J_t\neq0)=1-\omega_0$. Only this truncated service distribution affects the queue dynamics and terminal cost. We write $Q_t^{h,\omega,\bar\rho}$, $V_t^{h,\omega,\bar\rho}$ and $M_t^{h,\omega,\bar\rho}$ for the corresponding Bellman quantities.

\paragraph{Estimators of $\omega$ and $\bar\rho$.}
For $t\geq1$ and $t\geq2$ respectively, we form the estimates
\begin{equation}\label{eq:passive-estimates}
 \widehat\omega_{j,t}=\frac{1}{t}\sum_{s\leq t}\mathbf1\{J_s=j\},\quad j=0,\ldots,K,
 \qquad
 \widehat\rho_{c,t}=\frac{1}{t-1}\sum_{s<t}\mathbf1\{\bar C_s=c\},\quad c=0,\ldots,\bar B.
\end{equation}
The arrival estimate uses the current level $J_t$, which is observed before the decision; the capacity $\bar C_t$ is realized after it.

When $\bar\rho$ is known, set $\widetilde\rho_t:=\bar\rho$. When service is estimated and $t\geq2$, set $\widehat\mu_t:=1-\widehat\rho_{0,t}$ and, if $\widehat\mu_t>0$, set $\widetilde\rho_t:=\widehat\rho_t$. Whenever $\widetilde\rho_t$ is defined, we solve the remaining-horizon Bellman recursion using $\widehat\omega_t$ and $\widetilde\rho_t$, with terminal cost $D_{\widetilde\rho_t}$, and choose
\begin{equation}\label{eq:passive-action}
 A_t=\mathbf1\bigl\{J_t\in[K],\
 \underline h_{J_t,t}<\widehat\tau_t(B_t)\bigr\},
 \qquad
 \widehat\tau_t(b):=1-\frac{M_t^{\underline h_t,\widehat\omega_t,\widetilde\rho_t}(b)}{c_e},
\end{equation}
with the convention $M_t^{\cdot}(\bar B)=+\infty$ of Section~\ref{sec:problem-formulation}.
If service is estimated and either $t=1$ or $\widehat\mu_t=0$, we set $A_t=0$ instead.

Define
\[
 \Psi_T(\bar\rho):=\sum_{c=0}^{\bar B}\min\{\bar\rho_cT,\sqrt{\bar\rho_c(1-\bar\rho_c)T}\},
 \qquad S_T(\omega):=\sum_{j=1}^K\min\{\omega_jT,\sqrt{\omega_jT}\}.
\]
Denote the known-parameter worst-case bound by
\begin{equation}\label{eq:passive-base-bound}
 \mathcal B_T:=K\kappa\phi(T)
 +c_e\sum_{j=1}^K\min\{\omega_jT,\bar B+2\sqrt{2\omega_jT\phi(T)}\}
 +\bar Bc_e\left[1+2\sum_{j=1}^K\log\left(1+\frac{\omega_jT}{\phi(T)}\right)\right]
\end{equation}
where we recall $\kappa=4e\Delta=4e(\max\{c_e,c_{\mathrm{exp}}\}+c_{\mathrm{hold}}\bar B/\mu)$.
For $\mu_0>0$, set
\begin{equation}\label{eq:passive-constants}
 L_{\mu_0}:=\frac{\bar Bc_{\mathrm{hold}}}{\mu_0}+(c_e-c_{\mathrm{exp}})_+,
 \qquad a_\omega(\mu_0):=\frac{c_eL_{\mu_0}}{c_{\mathrm{hold}}},
 \qquad a_C(\mu_0):=\frac{L_{\mu_0}^2}{c_{\mathrm{hold}}}+L_{\mu_0}+
 \frac{c_{\mathrm{hold}}\bar B^2}{\mu_0^2}.
\end{equation}

In the rest of this section, we prove the following theorem. 

\begin{theorem}[Regret of \texttt{KL-PBD} for unknown $\omega$ and $\rho_C$]\label{thm:passive-estimation}
For $T\geq3$, the plug-in variants of \texttt{KL-PBD} defined above satisfy
\begin{align}
 R_T&\leq\mathcal B_T+4\,a_\omega(\mu)S_T(\omega),
 &&\text{if $\bar\rho$ is known},\label{eq:passive-known-service}\\
 R_T&\leq\mathcal B_T+\Omega\Delta(1+8/\mu)
 +4\,a_\omega(\mu/2)S_T(\omega)
 +2\Omega\,a_C(\mu/2)\Psi_T(\bar\rho),
 &&\text{if $\bar\rho$ is estimated}.\label{eq:passive-unknown-service}
\end{align}
Both bounds are
$\widetilde O(K+S_T(\omega))$, uniformly over arrival distributions.
\end{theorem}

Theorem~\ref{thm:passive-estimation} shows that learning these distributions requires no additional exploration under the stated observation assumptions and does not worsen the worst-case rate of Theorem~\ref{thm:pbd-regret}.


\begin{proof}[Proof of Theorem~\ref{thm:passive-estimation}]
All Bellman quantities used for comparison below are evaluated at the fixed
true distributions $(\omega,\bar\rho)$. We use the history from
Appendix~\ref{app:problem-formulation} and set
\[
 e_t:=Q_t^{\underline h_t,\omega,\bar\rho}(S_t,A_t)
       -V_t^{\underline h_t,\omega,\bar\rho}(S_t)\geq0.
\]
By Lemma~\ref{lem:uniform-bellman-gap},
$e_t\leq\Delta\mathbf1\{J_t\neq0\}$; also $e_t=0$ when $B_t=\bar B$.
Whenever $A_t$ is computed from \eqref{eq:passive-action}, \eqref{eq:bellman-margin} gives
\begin{equation}\label{eq:passive-action-error}
 e_t\leq\mathbf1\{J_t\neq0\}
 \sup_{b<\bar B}\left|M_t^{\underline h_t,\widehat\omega_t,\widetilde\rho_t}(b)
                    -M_t^{\underline h_t,\omega,\bar\rho}(b)\right|.
\end{equation}
The proof of Theorem~\ref{thm:pbd-regret} then extends as
\begin{equation}\label{eq:passive-transfer}
 R_T\leq\mathcal B_T+\sum_{t=1}^T\mathbb E[e_t].
\end{equation}
To see this, we use the proof of Lemma~\ref{lem:pbd-model-comparison}
with $\widetilde h_t=h^\star$.

The only use of exact optimality under
$\underline h_t$ is in \eqref{eq:pbd-fresh-one-step}. Here, by definition of $e_t$,
\[ Q_t^{\underline h_t,\omega,\bar\rho}(S_t,A_t)=V_t^{\underline h_t,\omega,\bar\rho}(S_t)+e_t, \]
so \eqref{eq:pbd-fresh-one-step} has the additional term $e_t$, and \eqref{eq:pbd-fresh-one-step-potential} has the additional term $\mathbb E[e_t\mid\mathcal F_{t-1}]$.

The potential bounds from Proposition~\ref{prop:bellman-stability} and
the subsequent telescoping argument are unchanged. Therefore, \eqref{eq:pbd-model-comparison}, and hence the decomposition in \eqref{eq:pbd-master-decomposition-fresh}, hold with the additional term $\mathbb E[\sum_{t<\sigma}e_t]$.

The bounds in Lemmas~\ref{lem:pbd-failure-fresh},
\ref{lem:pbd-deferral-fresh} and~\ref{lem:pbd-update-fresh} still apply:
the concentration, pending-label counting and lower-confidence-bound variation arguments in
Lemmas~\ref{lem:regional-concentration}, \ref{lem:pbd-deferred-sum}
and~\ref{lem:pbd-lcb-variation} hold for adaptive feasible actions.
The same simplification as in the proof of Theorem~\ref{thm:pbd-regret},
together with $e_t\geq0$, proves \eqref{eq:passive-transfer}.
No variation bound for $\widehat\omega_t$ or $\widehat\rho_t$ is needed, because the Bellman comparison in the proof is carried out under the fixed true distributions \((\omega,\bar\rho)\); the effect of using the empirical distributions to choose \(A_t\) is isolated in \(e_t\).

For an empirical distribution $\widehat v_n$ based on $n$ i.i.d.\
samples from $v$, we have
$\mathbb E[\widehat v_{i,n}]=v_i$ and
$\operatorname{Var}(\widehat v_{i,n})=v_i(1-v_i)/n$.
Cauchy--Schwarz and
$|\widehat v_{i,n}-v_i|\leq\widehat v_{i,n}+v_i$ therefore give
\begin{equation}\label{eq:passive-empirical-error}
 \mathbb E[|\widehat v_{i,n}-v_i|]
 \leq\min\{2v_i,\sqrt{v_i(1-v_i)/n}\},\qquad
 \sum_{n=1}^{T}\mathbb E[|\widehat v_{i,n}-v_i|]
 \leq2\min\{v_iT,\sqrt{v_i(1-v_i)T}\},
\end{equation}
where the second inequality uses $\sum_{n=1}^{T}n^{-1/2}\leq2\sqrt T$.
Since $\widehat\omega_t$ and $\omega$ are probability vectors,
$|\widehat\omega_{0,t}-\omega_0|
=\bigl|\sum_{j=1}^K(\widehat\omega_{j,t}-\omega_j)\bigr|$, so
\begin{equation}\label{eq:passive-arrival-error}
 \|\widehat\omega_t-\omega\|_1\leq2\sum_{j=1}^K|\widehat\omega_{j,t}-\omega_j|,
 \qquad
 \sum_{t=1}^T\mathbb E[\|\widehat\omega_t-\omega\|_1]\leq4S_T(\omega).
\end{equation}

Similarly, since $\widehat\rho_t$ is based on $t-1$ samples, \eqref{eq:passive-empirical-error} gives
\begin{equation}\label{eq:passive-service-error}
\sum_{t=2}^T\mathbb E[\|\widehat\rho_t-\bar\rho\|_1]\leq2\Psi_T(\bar\rho).
\end{equation}
For known $\bar\rho$, applying Lemma~\ref{lem:passive-stability} with $\mu_0=\mu$
in \eqref{eq:passive-action-error} and bounding the indicator by one gives
$\mathbb E[e_t]\leq a_\omega(\mu)
\mathbb E[\|\widehat\omega_t-\omega\|_1]$ for every $t\geq1$.
\Eqref{eq:passive-transfer} and \eqref{eq:passive-arrival-error}
prove \eqref{eq:passive-known-service}.

For estimated $\bar\rho$, let
$G_t:=\{\widehat\mu_t\geq\mu/2\}$ for $t\geq2$.
On $G_t$, Lemma~\ref{lem:passive-stability} applies with $\mu_0=\mu/2$;
on $G_t^c$, we have $e_t\leq\Delta\mathbf1\{J_t\neq0\}$.

Both $G_t$ and $\widehat\rho_t$ are $\mathcal F_{t-1}$-measurable, whereas $J_t$ is independent of $\mathcal F_{t-1}$. Bounding $\mathbf1\{J_t\neq0\}$ by one in the $\widehat\omega_t$ term and using independence in the other two terms therefore gives, for $t\geq2$,
\[
 \mathbb E[e_t]\leq a_\omega(\mu/2)
 \mathbb E[\|\widehat\omega_t-\omega\|_1]
 +\Omega\,a_C(\mu/2)\mathbb E[\|\widehat\rho_t-\bar\rho\|_1]
 +\Omega\Delta\mathbb{P}(G_t^c).
\]
The binomial lower-tail bound gives
$\mathbb{P}(G_t^c)\leq e^{-(t-1)\mu/8}$, so that
$\sum_{t=2}^T\mathbb{P}(G_t^c)\leq8/\mu$.
Together with $\mathbb E[e_1]\leq\Omega\Delta$,
\eqref{eq:passive-arrival-error}, \eqref{eq:passive-service-error} and \eqref{eq:passive-transfer} prove \eqref{eq:passive-unknown-service}.

Finally, for $T\geq1$,
\[ S_T(\omega)\geq\Omega\sqrt T,\qquad
 \Omega\Psi_T(\bar\rho)\leq\Omega\sqrt{\bar B T}
 \leq\sqrt{\bar B}\,S_T(\omega).\]

Indeed, $S_T(\omega)\geq\Omega\sqrt T$ follows from $\min\{\omega_jT,\sqrt{\omega_jT}\}\geq\omega_j\sqrt T$ for every $j\in[K]$. Moreover, Cauchy--Schwarz gives $\sum_{c=0}^{\bar B}\sqrt{\bar\rho_c(1-\bar\rho_c)}\leq\sqrt{\bar B}$ and hence $\Psi_T(\bar\rho)\leq\sqrt{\bar B T}$.
Thus the additional terms are $O(S_T(\omega))$, with constants
independent of $\omega$. If $\Omega=0$, all actions are forced and regret
is zero.
\end{proof}

\begin{lemma}[Uniform stability of the marginal cost]\label{lem:passive-stability}
Fix $h\in[0,1]^K$, with $h_0=1$. For two models\footnote{Each model uses its own terminal draining cost.} $(\omega,\rho)$ and $(\omega',\rho')$ satisfying $1-\rho_0,1-\rho'_0\geq\mu_0>0$, for every $t\leq T$ and $b<\bar B$,
\begin{equation}\label{eq:passive-stability}
\left|M_t^{h,\omega,\rho}(b)-M_t^{h,\omega',\rho'}(b)\right| \leq a_\omega(\mu_0)\|\omega-\omega'\|_1+a_C(\mu_0)\|\rho-\rho'\|_1.
\end{equation}
\end{lemma}
\begin{proof}
We use the notation $x_t$, $d_t$ and $r_j$ from the proof of Lemma~\ref{lem:service-queue-increments}, suppressing model superscripts. The recursions in \eqref{eq:service-smoothed-increment}, \eqref{eq:service-increment-recursion} and \eqref{eq:service-increment-boundary} apply with $\rho_C(c)=\rho_c$ for $c<\bar B$ and terminal cost $D_\rho$.

We interpolate linearly between the two models:
\[
(\omega^\lambda,\rho^\lambda):=(1-\lambda)(\omega,\rho)+\lambda(\omega',\rho'),\qquad \lambda\in[0,1].
\]
Since $1-\rho_0^\lambda\geq\mu_0$, \eqref{eq:service-adjacent-value-bound} gives $0\leq x_t(i)\leq L_{\mu_0}$ for every $\lambda\in[0,1]$. We bound $\mathrm d M_t(b)/\mathrm d\lambda$ uniformly in $\lambda$ and then integrate over $[0,1]$.

For the model with service law $\rho^\lambda$, define $g(i):=D_{\rho^\lambda}(i)-D_{\rho^\lambda}(i-1)$. Subtracting the defining recursions for $D_{\rho^\lambda}(i)$ and $D_{\rho^\lambda}(i-1)$ gives, for $1\leq i\leq\bar B$,
\[
(1-\rho_0^\lambda)g(i)=c_{\mathrm{hold}}+\sum_{c=1}^{i-1}\rho_c^\lambda g(i-c),
\]
where the sum is empty for $i=1$. Hence each $g(i)$ is absolutely continuous in $\lambda$. The recursions in \eqref{eq:service-smoothed-increment}, \eqref{eq:service-increment-recursion} and \eqref{eq:service-increment-boundary} then imply by backward induction that all $d_t(b)$ and $x_t(i)$ are absolutely continuous in $\lambda$, since these recursions involve only finite sums, products, and the Lipschitz maps $\min$, $\max$ and $\operatorname{med}$. Below, derivatives are taken at values of $\lambda$ where these functions are differentiable, which holds almost everywhere. Let
\[
F:=c_e\|\omega'-\omega\|_1+L_{\mu_0}\|\rho'-\rho\|_1,\qquad G:=\frac{c_{\mathrm{hold}}\bar B^2}{\mu_0^2}\|\rho'-\rho\|_1.
\]

We first bound $\mathrm d g(i)/\mathrm d\lambda$. Since $\sum_{c=1}^{i-1}\rho_c^\lambda\leq1-\rho_0^\lambda$ and $1-\rho_0^\lambda\geq\mu_0$, the preceding recursion gives $g(i)\leq c_{\mathrm{hold}}/\mu_0+\max_{k<i}g(k)$, where the maximum over an empty set is zero, so that $0\leq g(i)\leq i c_{\mathrm{hold}}/\mu_0$. Differentiating the recursion gives
\[
(1-\rho_0^\lambda)\frac{\mathrm d g(i)}{\mathrm d\lambda}=\sum_{c=1}^{i-1}\rho_c^\lambda\frac{\mathrm d g(i-c)}{\mathrm d\lambda}+\sum_{c=0}^{i-1}(\rho'_c-\rho_c)g(i-c),
\]
so that
\[
\left|\frac{\mathrm d g(i)}{\mathrm d\lambda}\right|\leq\max_{k<i}\left|\frac{\mathrm d g(k)}{\mathrm d\lambda}\right|+\frac{c_{\mathrm{hold}}\bar B}{\mu_0^2}\|\rho'-\rho\|_1.
\]
A second induction on $i\leq\bar B$ gives $\left|\mathrm d g(i)/\mathrm d\lambda\right|\leq G$.

We next bound $\mathrm d x_t(i)/\mathrm d\lambda$ recursively in terms of $\mathrm d x_{t+1}(k)/\mathrm d\lambda$. Define
\[
\psi_{j,t}(i):=\begin{cases}
\operatorname{med}\{d_t(i-1),d_t(i),r_j\}, & i<\bar B,\\[1mm]
\max\{d_t(\bar B-1),r_j\}, & i=\bar B.
\end{cases}
\]
By \eqref{eq:service-increment-recursion} and \eqref{eq:service-increment-boundary},
\begin{equation}\label{eq:passive-x-derivative-expansion}
x_t(i)=c_{\mathrm{hold}}+\sum_{j=0}^K\omega_j^\lambda\psi_{j,t}(i),
\qquad
\frac{\mathrm d x_t(i)}{\mathrm d\lambda}=\sum_{j=0}^K(\omega'_j-\omega_j)\psi_{j,t}(i)+\sum_{j=0}^K\omega_j^\lambda\frac{\mathrm d\psi_{j,t}(i)}{\mathrm d\lambda}.
\end{equation}

At almost every $\lambda$, the derivative of each median or maximum equals the derivative of one of the arguments attaining its value (when several arguments attain it on a set of positive measure, their derivatives agree almost everywhere there). For each $j$, choose such an argument, and let $\mathcal J_{t,i}$ be the set of $j$ for which it is some $d_t(b_j)$. Since $h$ is fixed, $\mathrm d r_j/\mathrm d\lambda=0$. Hence $\psi_{j,t}(i)=d_t(b_j)$ and $\mathrm d\psi_{j,t}(i)/\mathrm d\lambda=\mathrm d d_t(b_j)/\mathrm d\lambda$ for $j\in\mathcal J_{t,i}$, while $\psi_{j,t}(i)=r_j$ and $\mathrm d\psi_{j,t}(i)/\mathrm d\lambda=0$ for $j\notin\mathcal J_{t,i}$. By \eqref{eq:service-smoothed-increment},
\begin{equation}\label{eq:passive-d-derivative}
\frac{\mathrm d d_t(b)}{\mathrm d\lambda}=\sum_{c=0}^b \rho_c^\lambda\frac{\mathrm d x_{t+1}(b+1-c)}{\mathrm d\lambda}+\sum_{c=0}^b(\rho'_c-\rho_c)x_{t+1}(b+1-c).
\end{equation}
Substituting \eqref{eq:passive-d-derivative} into \eqref{eq:passive-x-derivative-expansion} for $j\in\mathcal J_{t,i}$ gives
\begin{equation}\label{eq:passive-derivative}
\frac{\mathrm d x_t(i)}{\mathrm d\lambda}=\sum_{k=1}^{\bar B}a_{t,i,k}\frac{\mathrm d x_{t+1}(k)}{\mathrm d\lambda}+f_t(i),
\end{equation}
where
\[a_{t,i,k}:=\sum_{j\in\mathcal J_{t,i}}\sum_{c=0}^{b_j}\omega_j^\lambda \rho_c^\lambda\mathbf 1\{k=b_j+1-c\}\geq0,
\qquad \sum_{k=1}^{\bar B}a_{t,i,k}=\sum_{j\in\mathcal J_{t,i}}\omega_j^\lambda\sum_{c=0}^{b_j}\rho_c^\lambda\leq1,\]
and
\[f_t(i):=\sum_{j=0}^K(\omega'_j-\omega_j)\psi_{j,t}(i)+\sum_{j\in\mathcal J_{t,i}}\omega_j^\lambda\sum_{c=0}^{b_j}(\rho'_c-\rho_c)x_{t+1}(b_j+1-c).\]

The second term of $f_t(i)$ is at most $L_{\mu_0}\|\rho'-\rho\|_1$ in absolute value, since $0\leq x_{t+1}(k)\leq L_{\mu_0}$ and $\sum_{j}\omega_j^\lambda\leq1$. For the first, both $\operatorname{med}\{d_t(i-1),d_t(i),r\}$ and $\max\{d_t(\bar B-1),r\}$ are $1$-Lipschitz in $r$ for fixed $d_t$, and $r_j=L^h(j)-c_{\mathrm{exp}}$ with $L^h(j)\in[0,c_e]$, so the values $\psi_{j,t}(i)$ lie in an interval of length at most $c_e$. Since $\sum_j(\omega'_j-\omega_j)=0$, the first term is at most $c_e\|\omega'-\omega\|_1$ in absolute value. Hence $|f_t(i)|\leq F$.

The same coefficients satisfy
\begin{equation}\label{eq:passive-x-lower}
x_t(i)\geq c_{\mathrm{hold}}+\sum_{k=1}^{\bar B}a_{t,i,k}x_{t+1}(k).
\end{equation}
Indeed, for $j\in\mathcal J_{t,i}$, $\psi_{j,t}(i)=d_t(b_j)=\sum_{c=0}^{b_j}\rho_c^\lambda x_{t+1}(b_j+1-c)$, so these terms of \eqref{eq:passive-x-derivative-expansion} sum to $\sum_ka_{t,i,k}x_{t+1}(k)$. For $j\notin\mathcal J_{t,i}$, $\psi_{j,t}(i)=r_j\geq0$: $r_j$ attaining the median (resp. maximum) implies $r_j\geq d_t(i-1)$ (resp. $r_j\geq d_t(\bar B-1)$), which is nonnegative by \eqref{eq:service-d-monotone}.

We now prove by backward induction that
\begin{equation}\label{eq:passive-propagation}
\left|\frac{\mathrm d x_t(i)}{\mathrm d\lambda}\right|\leq\frac{F}{c_{\mathrm{hold}}}x_t(i)+G.
\end{equation}
At $t=T+1$, this follows from $x_{T+1}(i)=g(i)\geq0$ and $\left|\mathrm d g(i)/\mathrm d\lambda\right|\leq G$. If \eqref{eq:passive-propagation} holds at $t+1$, then \eqref{eq:passive-derivative} and \eqref{eq:passive-x-lower} give
\[
\left|\frac{\mathrm d x_t(i)}{\mathrm d\lambda}\right|\leq\sum_{k=1}^{\bar B}a_{t,i,k}\left(\frac{F}{c_{\mathrm{hold}}}x_{t+1}(k)+G\right)+F
\leq\frac{F}{c_{\mathrm{hold}}}\bigl(x_t(i)-c_{\mathrm{hold}}\bigr)+G+F
=\frac{F}{c_{\mathrm{hold}}}x_t(i)+G.
\]
Thus $\left|\mathrm d x_t(i)/\mathrm d\lambda\right|\leq FL_{\mu_0}/c_{\mathrm{hold}}+G$.

Finally, $M_t(b)=c_{\mathrm{exp}}+d_t(b)$ by \eqref{eq:service-bellman-marginal}, so \eqref{eq:passive-d-derivative} gives
\[
\left|\frac{\mathrm d M_t(b)}{\mathrm d\lambda}\right|\leq\frac{FL_{\mu_0}}{c_{\mathrm{hold}}}+G+L_{\mu_0}\|\rho'-\rho\|_1
=a_\omega(\mu_0)\|\omega'-\omega\|_1+a_C(\mu_0)\|\rho'-\rho\|_1.
\]
Integrating over $\lambda\in[0,1]$ proves \eqref{eq:passive-stability}.
\end{proof}

\clearpage
\section{Proofs for Section~\ref{sec:known-M}}
\label{app:known-M}

\subsection{Proof sketches}\label{app:lb-proof-sketch}

To ease navigation within this appendix, we first provide high-level proof sketches of our instance-dependent and minimax lower bounds. They are proven in Appendices~\ref{app:instance-lb-proof} and ~\ref{app:worst-case-lb-proof} respectively.

\textit{Proof sketch of Theorem~\ref{thm:instance_dependent_lower_bound}.} 
We first reduce to the long-horizon problem, in which the oracle threshold no longer depends on the round; passing from long-horizon to finite-horizon regret costs only an $O(1)$ additive term (Lemma~\ref{lem:lb-average-cost}). For a fixed level $j$ with $h_j^\star>\overline\tau^\star$, the long-horizon optimal policy accepts level $j$ at every queue state, whereas deferring incurs an extra cost of at least $c_e(h_j^\star-\overline\tau^\star)$. We compare $h^\star$ with an alternative instance in which only $h_j^\star$ is moved slightly below $\overline\tau^\star$. Although the deferral threshold changes with the instance, we show that deferring level $j$ is strictly beneficial in the long run under the alternative (\eqref{eq:lb-aeps}), so a uniformly efficient learner must behave differently on the two instances. Since only completed level-$j$ reviews distinguish them (Lemma~\ref{lem:lower-bound-information}),
a change-of-measure argument \citep{lai1985asymptotically} forces order $\log T/\kl(h_j^\star,\overline\tau^\star)$ level-$j$ deferrals under $h^\star$. Multiplying by their extra cost $c_e(h_j^\star-\overline\tau^\star)$ yields the corresponding term in Theorem~\ref{thm:instance_dependent_lower_bound}. Levels with $h_j^\star<\overline\tau^\star$ are instead always deferred by the oracle at the empty queue, so they can be reviewed without taking a suboptimal action. \qed

\textit{Proof sketch of Theorem~\ref{thm:known-M-lower}.} When $\mu=1$, we move each conditional accuracy $h_j$ a distance $\delta_j$ above or below the oracle's accuracy threshold, so that two neighboring instances differing only at level $j$ have different optimal actions for level-$j$ queries. With $\delta_j\asymp\min\{1,(\omega_jT)^{-1/2}\}$, Lemma~\ref{lem:lower-bound-information} implies that the corresponding histories remain statistically indistinguishable, so on one of the two instances the learner must take the suboptimal action on a constant fraction of these queries, each costing order $c_e\delta_j$. When $\mu<1$, perturbing the accuracies moves the threshold they are compared with. We therefore shift all coordinates of each perturbed instance by a common amount until the long-horizon empty-queue threshold coincides with the center of the perturbation (Lemma~\ref{lem:lb-self-centered-family}). An additional difficulty is that the two neighboring instances are only guaranteed to have different optimal actions when the queue is empty; by keeping the queue occupied, the learner could reduce the number of rounds on which indistinguishability translates into regret. Lemmas~\ref{lem:lb-local-queue} and~\ref{lem:lb-empty-occupancy} show that this is not free: if the queue is rarely empty, the learner must defer sufficiently often while it is congested, at a cost of order $c_{\rm hold}$. Capping $\delta_j$ at the holding-cost scale keeps this cost uniformly positive and produces the second term in the minimum, yielding Theorem~\ref{thm:known-M-lower}. \qed

\subsection{Common ingredients}\label{app:common-ingredients}
We first collect the average-cost and information identities used by both lower-bound proofs. Their proofs and the supporting technical lemmas are given in Subsection~\ref{app:lb-technical}.

Throughout this appendix, $\omega$, $\rho_C$, $\bar B$, $(c_e,c_{\rm exp},c_{\rm hold})$ and the review order are fixed; only the conditional accuracy vector $h$ varies. Write $\mathbb{P}^\pi_h$ and $\mathbb{E}^\pi_h$ for probability and expectation under $h$ and algorithm $\pi$.

For a function $u:\mathcal B\to\mathbb R$ and $a\in\{0,1\}$ with $b+a\le\bar B$,
write $(P_au)(b):=\sum_{b'}P_a(b,b')u(b')=\mathbb E[u((b+a-C_t)^+)]$, and define the expected
one-step cost
$\ell_h(b,j,a):=c_{\rm hold}b+c_{\rm exp}a+L^h(j)(1-a)$.
In this notation, the terminal cost satisfies $D(0)=0$ and
\begin{equation} \label{eq:lb-drain} D(b)=c_{\rm hold}b+P_0D(b),\qquad b\ge1. \end{equation}

We recall that a stationary policy $\varphi$ is a policy whose action distribution $\varphi(\cdot\mid b,j)$ depends only on the current state $(b,j)$, not on time or past history.

A pair $(g_h,u_h)\in\mathbb R\times\mathbb R^{\mathcal B}$ with $u_h(0)=0$ solves the average-cost optimality equation (ACOE) for $h$ if
\begin{equation} \label{eq:lb-acoe} g_h+u_h(b) = \sum_{j=0}^K\omega_j \min_{a\in\mathcal A(b,j)} \{\ell_h(b,j,a)+(P_a u_h)(b)\}, \qquad b\in\mathcal B. \end{equation}

Given such a pair, write $x_h(q):=u_h(q)-u_h(q-1)$ for $q\in[\bar B]$, $d_h(b):=(P_1-P_0)u_h(b)$ and $\overline M^h(b):=c_{\rm exp}+d_h(b)$ for $b<\bar B$; recall $r_j^h=L^h(j)-c_{\rm exp}$. For a feasible action, define its average-cost disadvantage by
\begin{equation} \label{eq:lb-disadvantage} \mathfrak d_h(b,j,a) := \ell_h(b,j,a)+(P_a u_h)(b) - \min_{a'\in\mathcal A(b,j)} \{\ell_h(b,j,a')+P_{a'}u_h(b)\}\ \ge0 . \end{equation}

For $s\ge0$, let $W_s^h(b)$ be the queue value function with $s$ decision rounds remaining and terminal cost $D$.
Thus $W_0^h=D$ and $W_{T+1-t}^h=U_t^h$.
Define
$\xi_s^h(q):=W_s^h(q)-W_s^h(q-1)$, so that
$x_t^h(q)=\xi_{T+1-t}^h(q)$.

\begin{lemma}[Average-cost comparison]
\label{lem:lb-average-cost}
For every $h\in[0,1]^K$, the limits
$u_h:=\lim_{s\to\infty}[W_s^h-W_s^h(0)]$ and
$g_h:=\lim_{s\to\infty}[W_{s+1}^h(0)-W_s^h(0)]$ exist.
The pair $(g_h,u_h)$ solves \eqref{eq:lb-acoe};
$u_h$ is nondecreasing and discrete convex, and $u_h\ge D$.
Moreover,
\begin{enumerate}
\item[(i)] the following increment identities hold:
\begin{align}
d_h(b)&=\sum_{c=0}^{b}\rho_C(c)\,x_h(b+1-c), \label{eq:lb-smoothed}\\
x_h(q)&=c_{\rm hold}+\sum_{j=0}^K\omega_j\operatorname{med}\{d_h(q-1),d_h(q),r^h_j\}, \qquad 1\le q<\bar B, \label{eq:lb-increment}\\
x_h(\bar B)&=c_{\rm hold}+\sum_{j=0}^K\omega_j\max\{d_h(\bar B-1),r^h_j\}; \label{eq:lb-increment-boundary}
\end{align}
\item[(ii)] $c_{\rm exp}\le\overline M^h(0)\le\cdots\le\overline M^h(\bar B-1)$, and
\begin{equation} \label{eq:lb-empty-marginal} \overline M^h(0)=c_{\rm exp}+(1-\mu)\,u_h(1); \end{equation}
\item[(iii)] for every algorithm $\pi$,
\begin{equation} \label{eq:lb-regret-average} \Big| R_T(\pi;h) - \mathbb E_h^\pi\Big[\sum_{t=1}^T\mathfrak d_h(B_t,J_t,A_t)\Big] \Big| \le \operatorname{sp}(D-u_h); \end{equation}
\item[(iv)] for every stationary policy $\varphi$ with a stationary distribution $\lambda$,
\begin{equation} \label{eq:lb-gain-optimal} g_h\le\sum_{s,a}\lambda(s)\varphi(a\mid s)\,\ell_h(s,a). \end{equation}
\end{enumerate}
\end{lemma}

Recall $D_{j,t}:=\sum_{s=1}^t\mathbf 1\{J_s=j,A_s=1\}$, the number of self-confidence level $j$ deferrals submitted through round $t$, and that $N_{j,t}$ is the number of completed self-confidence level $j$ labels available before round $t$. Let $P_h^{\pi,t}$ denote the law of $\mathcal F_t$ under $h$ and algorithm $\pi$.

\begin{lemma}[Information from completed reviews]
\label{lem:lower-bound-information}
For every $t\ge0$, any algorithm $\pi$, and conditional accuracy vectors $h,\widetilde h$ such that $\widetilde h_j\in(0,1)$ whenever $h_j\ne\widetilde h_j$,
\begin{equation} \label{eq:lb-information} \mathrm{KL}(P_h^{\pi,t}\|P_{\widetilde h}^{\pi,t}) = \sum_{j=1}^K\mathbb E_h^\pi[N_{j,t+1}]\kl(h_j,\widetilde h_j) \le \sum_{j=1}^K\mathbb E_h^\pi[D_{j,t}]\kl(h_j,\widetilde h_j). \end{equation}
Consequently,
\begin{equation} \label{eq:lb-information-arrivals} \mathrm{KL}(P_h^{\pi,t}\|P_{\widetilde h}^{\pi,t}) \le t\sum_{j=1}^K\omega_j\kl(h_j,\widetilde h_j). \end{equation}
\end{lemma}

Define
\begin{equation} \label{eq:lb-Mq} M_\tau:=c_e(1-\tau)=c_{\rm exp}+\frac{1-\mu}{\mu}c_{\rm hold}, \qquad h^0:=\tau\mathbf1. \end{equation}

\begin{lemma}[Empty-queue threshold bounds]
\label{lem:lb-one-job-floor}
For every $h\in[0,1]^K$ and $1\le t\le T$,
\begin{equation} \label{eq:lb-finite-empty-marginal} M_t^h(0)=c_{\rm exp}+(1-\mu)x_{t+1}^h(1), \end{equation}
and
\begin{equation} \label{eq:lb-threshold-sandwich} M_\tau=M_T^h(0)\le M_t^h(0)\le\overline M^h(0), \qquad M_t^h(0)\longrightarrow\overline M^h(0) \quad\text{as }T-t\to\infty. \end{equation}
Moreover, $M_t^h(0)\le M_{t-1}^h(0)$ for $2\le t\le T$.
If $\tau\in[0,1]$, then at $h^0$ all marginal costs in
\eqref{eq:lb-threshold-sandwich} equal $M_\tau$,
always accepting is optimal, and $u_{h^0}=D$.
\end{lemma}

In particular, writing $\overline\tau(h):=1-\overline M^h(0)/c_e$, we have $\overline\tau(h)\le1-M^h_t(0)/c_e\le\tau$ for every $h$ and $t$, with equality at $h^0$ when $\tau \in [0,1]$.

\subsection{Instance-dependent regret lower bound}\label{app:instance-lb-proof}
We show that every uniformly efficient algorithm must collect order $\log T$ expert labels at each self-confidence level where the oracle accepts, by comparing the true instance with an alternative on which deferral is strictly beneficial.

\begin{proof}
For each horizon $T$, write $\mathbb{P}_h := \mathbb{P}^{\pi}_h$ and $\mathbb{E}_h := \mathbb{E}^{\pi}_h$ for probability and expectation under $h$ and $\pi$ run with horizon $T$, and use the subscript $\star$ when $h=h^\star$. Let $u^\star:=u_{h^\star}$ be the bias from Lemma~\ref{lem:lb-average-cost}, set $\mathfrak d_\star:=\mathfrak d_{h^\star}$, so that $\overline M^\star(0)=c_{\rm exp}+(P_1-P_0)u^\star(0)$ and $\overline\tau^\star=1-\overline M^\star(0)/c_e$ by Lemma~\ref{lem:lb-one-job-floor}. Finally, let $I:=\{j\in[K]:\omega_j>0,\ h_j^\star>\overline\tau^\star\}$.

For $j\in I$, Lemma~\ref{lem:lb-average-cost}(ii) implies
that acceptance is the unique average-cost minimizer at every
queue state. For $j\in I$ and $b<\bar B$,
\[ \mathfrak d_{h^\star}(b,j,1) = \overline M^\star(b)-L^{h^\star}(j) \ge c_e(h_j^\star-\overline\tau^\star). \]
Thus \eqref{eq:lb-regret-average} gives
\begin{equation} \label{eq:lb-instance-count} R_T(\pi;h^\star) \ge c_e\sum_{j\in I}(h_j^\star-\overline\tau^\star)\mathbb E_\star [D_{j,T}]-C_\star, \qquad C_\star:=\operatorname{sp}(D-u^\star). \end{equation}
By \eqref{eq:lb-instance-count}, uniform efficiency implies $\mathbb E_\star [D_{j,T}]\le[R_T(\pi;h^\star)+C_\star]/[c_e(h_j^\star-\overline\tau^\star)]=o(T^\alpha)$ for every $\alpha>0$ and $j\in I$.

\paragraph{A confusing alternative for self-confidence level $j$.} Fix $j\in I$ and $\varepsilon\in(0,\overline\tau^\star)$. We construct an alternative on which deferral at level $j$ is strictly beneficial while all other conditional accuracies are unchanged: set $\widetilde h_j:=\overline\tau^\star-\varepsilon$ and $\widetilde h_k:=h_k^\star$ for $k\ne j$.
We use the transformed costs $k_h$, gains $\bar g_h$, and state values $v_h$ of Lemma~\ref{lem:lb-transformed-disadvantage}, which isolate the effect of deferring at level $j$, and write $\bar g_\star:=\bar g_{h^\star}$ and $v_\star:=v_{h^\star}$.
Let $\varphi^\star$ be a deterministic stationary minimizing policy
for $h^\star$; since $j\in I$, it accepts at level $j$ at every
queue state. Define a stationary policy $\varphi'$ that agrees with $\varphi^\star$ everywhere except that $\varphi'(0,j)=1$. From any positive queue state, $\bar B$ consecutive rounds with $J_t=j$ and $C_t\ge1$ empty the queue because $\varphi'$ accepts at level $j$ whenever the queue state is positive. Thus every state reaches the recurrent class containing queue zero with positive probability at least $(\omega_j\mu)^{\bar B}$, so this class is unique. Let $\lambda$ be its stationary distribution; since $J_t$ is drawn independently with law $\omega$ and $\omega_j>0$, $\lambda(0,j)>0$.

Let $\bar g_{\widetilde h}(\varphi')$ denote the stationary average transformed cost of $\varphi'$ under $\widetilde h$. Stationarity of $\lambda$ and \eqref{eq:lb-state-disadvantage} give
\[\bar g_{\widetilde h}(\varphi')-\bar g_\star =\sum_{s,a}\lambda(s)\varphi'(a\mid s) \bigl[k_{\widetilde h}(s,a)-\bar g_\star +\mathbb E[v_\star(S_{t+1})\mid S_t=s,A_t=a] -v_\star(s) \bigr]. \]
For every action used by $\varphi^\star$, the bracket is zero. The only additional action used by $\varphi'$ is deferral at $(0,j)$, where the bracket equals $\overline M^\star(0)-L^{\widetilde h}(j)=-c_e\varepsilon$. Therefore $\bar g_{\widetilde h}(\varphi')=\bar g_\star-\lambda(0,j)c_e\varepsilon$. By \eqref{eq:lb-gain-optimal} for the transformed problem,
\begin{equation} \label{eq:lb-aeps} \bar g_{\widetilde h} \le \bar g_\star-\lambda(0,j)c_e\varepsilon, \qquad a_\varepsilon:=\bar g_\star-\bar g_{\widetilde h}\ge\lambda(0,j)c_e\varepsilon>0. \end{equation}

\paragraph{Regret of under-reviewing at level $j$ under $\widetilde h$.} Let $r_\varepsilon:=L^{\widetilde h}(j)-L^{h^\star}(j)=c_e(h_j^\star-\widetilde h_j)>0$, $w:=v_{\widetilde h}-v_\star$, and $s_w:=\operatorname{sp}(w)$. Subtracting \eqref{eq:lb-state-disadvantage} under the two conditional accuracy vectors gives
\begin{equation} \label{eq:lb-disadvantage-difference} \mathfrak d_{\widetilde h}(s,a)-\mathfrak d_\star(s,a) =-r_\varepsilon a\mathbf 1\{k=j\}+a_\varepsilon +\mathbb E[w(S_{t+1})\mid S_t=s,A_t=a]-w(s). \end{equation}
Generate $J_{T+1}\sim\omega$ independently after the last decision
and set $S_{T+1}=(B_{T+1},J_{T+1})$. By Lemma~\ref{lem:lb-transformed-disadvantage},
\[\mathcal M_t:=\sum_{s=1}^t\left(w(S_{s+1})-\mathbb E[w(S_{s+1})\mid S_s,A_s]\right)\]
is a martingale under $\widetilde h$, with
$|\mathcal M_t-\mathcal M_{t-1}|\le s_w$. Summing \eqref{eq:lb-disadvantage-difference}, telescoping, and using $\mathfrak d_\star\ge0$ gives
\begin{equation} \label{eq:lb-pathwise-instance} r_\varepsilon D_{j,T}+\sum_{t=1}^T\mathfrak d_{\widetilde h}(S_t,A_t) \ge Ta_\varepsilon-s_w-\mathcal M_T. \end{equation}

Define $E_T:=\{D_{j,T}<a_\varepsilon T/(4r_\varepsilon)\}$.
For $T$ such that $s_w\le a_\varepsilon T/4$, \eqref{eq:lb-pathwise-instance} implies
\[ E_T \subseteq \left\{\sum_{t=1}^T\mathfrak d_{\widetilde h}(S_t,A_t)\ge\frac{a_\varepsilon T}{4}\right\} \cup \left\{\mathcal M_T\ge\frac{a_\varepsilon T}{4}\right\}. \]

By \eqref{eq:lb-regret-average} and uniform efficiency at $\widetilde h$, $\mathbb E_{\widetilde h}[\sum_{t=1}^T\mathfrak d_{\widetilde h}(S_t,A_t)]\le R_T(\pi;\widetilde h)+\operatorname{sp}(D-u_{\widetilde h})=o(T^\alpha)$. Markov's inequality and the Hoeffding--Azuma inequality \citep[Exercise~20.6]{lattimore2020bandit} therefore give, for every $\alpha\in(0,1)$,
\begin{equation} \label{eq:lb-event-alt} \mathbb P_{\widetilde h}(E_T) \le \frac{4\mathbb E_{\widetilde h}[\sum_{t=1}^T\mathfrak d_{\widetilde h}(S_t,A_t)]}{a_\varepsilon T} + \exp\!\left\{-\frac{a_\varepsilon^2T}{32(1+s_w)^2}\right\} =o(T^{\alpha-1}). \end{equation}
Under the true instance,
\begin{equation} \label{eq:lb-event-true} \mathbb P_\star(E_T^c) \le \frac{4r_\varepsilon\mathbb E_\star [D_{j,T}]}{a_\varepsilon T} \longrightarrow0. \end{equation}

\paragraph{Information from completed reviews.} Lemma~\ref{lem:lower-bound-information} gives
\[ \mathrm{KL}(P_{h^\star}^{\pi,T}\|P_{\widetilde h}^{\pi,T}) \le \mathbb E_\star [D_{j,T}]\kl(h_j^\star,\widetilde h_j). \]
Since $E_T\in\mathcal F_T$, the data-processing inequality \citep[Exercise~14.10]{lattimore2020bandit} applies. Moreover, $\kl(p,q)\ge p\log(1/q)-\log2$ because the binary entropy is at most $\log2$. Hence
\[ \mathbb E_\star [D_{j,T}]\kl(h_j^\star,\widetilde h_j) \ge \mathbb P_\star(E_T)\log\frac1{\mathbb P_{\widetilde h}(E_T)}-\log2. \]
By \eqref{eq:lb-event-true}, $\mathbb P_\star(E_T)\to1$, while \eqref{eq:lb-event-alt} gives $\log(1/\mathbb P_{\widetilde h}(E_T))\ge(1-\alpha-o(1))\log T$. Therefore, for every $\alpha\in(0,1)$,
\[ \liminf_{T\to\infty}\frac{\mathbb E_\star [D_{j,T}]}{\log T} \ge\frac{1-\alpha}{\kl(h_j^\star,\overline\tau^\star-\varepsilon)}. \]
Letting $\alpha\downarrow0$ and then $\varepsilon\downarrow0$, using continuity of Bernoulli relative entropy in its second argument on $(0,1)$, yields
\[ \liminf_{T\to\infty}\frac{\mathbb E_\star [D_{j,T}]}{\log T} \ge\frac1{\kl(h_j^\star,\overline\tau^\star)}. \]
Substituting this into \eqref{eq:lb-instance-count} and summing over $j\in I$ proves \eqref{eq:main-instance}.
\end{proof}

\begin{remark}
When $\mu=1$, $\overline\tau^\star=1-c_{\rm exp}/c_e$, and Theorem~\ref{thm:instance_dependent_lower_bound} can alternatively be derived from the Graves--Lai lower bound for structured MDPs \citep[Theorem~1]{ok2018exploration}.
\end{remark}

\subsection{\texttt{KL-PBD} is asymptotically optimal under immediate-service}
\label{app:kl-noqueue}
In this section, we show that \texttt{KL-PBD} matches Theorem \ref{thm:instance_dependent_lower_bound} when service is immediate $(\mu=1)$.

Throughout this subsection, $C_t\ge1$ almost surely and
$0<c_{\rm exp}<c_e$. Thus $\mu=1$ and
$\tau=1-c_{\rm exp}/c_e\in(0,1)$.
We consider \texttt{KL-PBD} with a fixed budget $\phi(T)>0$.
All probabilities and expectations in this subsection are under
$h^\star$ and $\pi^{\texttt{KL-PBD}}$.

By \eqref{eq:lb-finite-empty-marginal} with $\mu=1$, $M_t^h(0)=c_{\mathrm{exp}}$ for every $h$ and $t$.
Thus, after $n$ completed level $j$ reviews, the algorithm
defers exactly when $g_{j,n}<\tau$, with ties resolved in
favor of acceptance.

Every deferral yields a label before the next round.
The sample counts before successive level $j$ deferrals are
therefore $0,1,\ldots,D_{j,T}-1$.
We use the infinite i.i.d.\ ground-truth label sequences for each self-confidence level from Lemma~\ref{lem:regional-concentration}.

\begin{lemma}[Erroneous acceptance under immediate service]
\label{lem:kl-noqueue-acceptance}
For a level $j$ with $0<h_j^\star<\tau$,
\begin{equation} \E\left[\sum_{t=1}^T\mathbf 1\{J_t=j,A_t=0\}\right] \le \frac{T e^{-\phi(T)}} {\exp\{\kl(\tau,h_j^\star)\}-1}. \label{eq:kl-noqueue-acceptance} \end{equation}
If $h_j^\star=0$, the left-hand side is zero.
\end{lemma}

\begin{proof}
Fix $j$ with $0<h_j^\star<\tau$.
The algorithm accepts only when $g_{j,n}\ge\tau$ and hence $n\ge1$,
since $g_{j,0}=0$.
By the definition of the lower confidence bound and $\phi(T)>0$, $g_{j,n}\ge\tau$ implies $\widehat h_{j,n}>\tau$ and $n\kl(\widehat h_{j,n},\tau)\ge\phi(T)$.
For every $z\in[\tau,1]$,
\[
\begin{aligned}
\kl(z,h_j^\star)-\kl(z,\tau) &= \kl(\tau,h_j^\star) +(z-\tau) \log\left( \frac{\tau(1-h_j^\star)} {h_j^\star(1-\tau)} \right)\\
&\ge\kl(\tau,h_j^\star).
\end{aligned}
\]
If $\phi(T)/n\le\kl(1,\tau)$, let $z_n\in(\tau,1]$
solve $\kl(z_n,\tau)=\phi(T)/n$.
The Bernoulli Chernoff bound gives
\[ \mathbb P(g_{j,n}\ge\tau) \le \mathbb P(\widehat h_{j,n}\ge z_n) \le e^{-n\kl(z_n,h_j^\star)} \le e^{-\phi(T)-n\kl(\tau,h_j^\star)}. \]
If $\phi(T)/n>\kl(1,\tau)$, the event is empty and
the same bound holds.
Taking a union bound over sample counts,
\[ \mathbb P(\exists n\ge1:g_{j,n}\ge\tau) \le \sum_{n=1}^{\infty} e^{-\phi(T)-n\kl(\tau,h_j^\star)} = \frac{e^{-\phi(T)}} {\exp\{\kl(\tau,h_j^\star)\}-1}. \]
Since
\[ \sum_{t=1}^T\mathbf 1\{J_t=j,A_t=0\} \le T\mathbf 1\{\exists n\ge1:g_{j,n}\ge\tau\}, \]
taking expectations proves the bound.
If $h_j^\star=0$, all sample means and lower confidence bounds
are zero, so every level $j$ query is deferred.
\end{proof}

\begin{proposition}[Asymptotic optimality under immediate service]
\label{prop:kl-noqueue}
For fixed $K$, any budget sequence satisfying
$\phi(T)/\log T\to1$ makes \texttt{KL-PBD} uniformly efficient on the
immediate-service class. If, in addition,
$T e^{-\phi(T)}/\log T\to0$, then, for every fixed immediate-service
instance,
\[ \lim_{T\to\infty} \frac{R_T(\pi^{\texttt{KL-PBD}})}{\log T} = c_e\sum_{\substack{j\in[K]\\ h_j^\star>\tau,\ \omega_j>0}} \frac{h_j^\star-\tau}{\kl(h_j^\star,\tau)}. \]

In particular, the equality holds for
$\phi(T)=\log T+\log\log T$ and $\phi(T)=\log T$
for $T\ge3$, with arbitrary positive budgets for smaller horizons.
If $h_j^\star\ge\tau$ for every $j\in[K]$ with $\omega_j>0$, the
condition $T e^{-\phi(T)}/\log T\to0$ is unnecessary.
\end{proposition}
\begin{proof}
Fix a level $j$ with $h_j^\star>\tau$ and
$\varepsilon\in(0,h_j^\star-\tau)$, and write
$d:=\kl(h_j^\star-\varepsilon,\tau)>0$.
If $n\ge\phi(T)/d$ and
$\widehat h_{j,n}\ge h_j^\star-\varepsilon$, then
\[ n\kl(\widehat h_{j,n},\tau) \ge n\kl(h_j^\star-\varepsilon,\tau) \ge\phi(T). \]
Since $\widehat h_{j,n}>\tau$, this implies $g_{j,n}\ge\tau$, so the
algorithm accepts. Hence each deferral at a sample count
$n\ge\phi(T)/d$ requires
$\widehat h_{j,n}<h_j^\star-\varepsilon$. Since each sample count
occurs at most once before a deferral,
\[ D_{j,T} \le \left\lceil\frac{\phi(T)}{d}\right\rceil + \sum_{n=1}^{\infty} \mathbf 1\{\widehat h_{j,n}<h_j^\star-\varepsilon\}. \]
Taking expectations and applying the Bernoulli Chernoff bound gives
\begin{align}\label{eq:kl-noqueue-finite}
\mathbb E [D_{j,T}] &\le \left\lceil\frac{\phi(T)}{d}\right\rceil + \sum_{n=1}^{\infty} e^{-n\kl(h_j^\star-\varepsilon,h_j^\star)} \nonumber \\
&\le 1+\frac{\phi(T)}{\kl(h_j^\star-\varepsilon,\tau)} +\frac{1} {\exp\{\kl(h_j^\star-\varepsilon,h_j^\star)\}-1}
\end{align}

where the final term is interpreted as zero when the divergence is
infinite. In particular, this covers $h_j^\star=1$.

We next prove uniform efficiency. By Lemma~\ref{lem:regret-decomposition}, under immediate service,
\[
\begin{aligned}
R_T(\pi^{\texttt{KL-PBD}}) ={}& c_e\sum_{\substack{j\in[K]\\
h_j^\star>\tau,\ \omega_j>0}} (h_j^\star-\tau)\E [D_{j,T}]\\
&+ c_e\sum_{\substack{j\in[K]\\
h_j^\star<\tau,\ \omega_j>0}} (\tau-h_j^\star) \E\!\left[\sum_{t=1}^T\mathbf1\{J_t=j,A_t=0\}\right].
\end{aligned}
\]
Levels $j$ with $h_j^\star=\tau$ or $\omega_j=0$ contribute zero regret. If $h_j^\star=0$, \texttt{KL-PBD} never accepts at level $j$.
By Lemma~\ref{lem:kl-noqueue-acceptance}, the second term is at most
\begin{equation} c_eT e^{-\phi(T)} \sum_{\substack{j\in[K]\\0<h_j^\star<\tau,\ \omega_j>0}} \frac{\tau-h_j^\star} {\exp\{\kl(\tau,h_j^\star)\}-1}. \label{eq:kl-noqueue-acceptance-regret} \end{equation}

If $\phi(T)/\log T\to1$, then
$T e^{-\phi(T)}=\exp\{o(\log T)\}=o(T^\alpha)$ for every
$\alpha>0$, while \eqref{eq:kl-noqueue-finite} gives
$\E [D_{j,T}]=O(\log T)$ for every $j$ with $h_j^\star>\tau$.
Thus \texttt{KL-PBD} is uniformly efficient on the immediate-service
class.

It remains to identify the asymptotic coefficient. Dividing
\eqref{eq:kl-noqueue-finite} by $\log T$, using
$\phi(T)/\log T\to1$, and then letting $\varepsilon\downarrow0$
gives, for every $j$ with $h_j^\star>\tau$,
\[\limsup_{T\to\infty}\frac{\E [D_{j,T}]}{\log T}\le\frac{1}{\kl(h_j^\star,\tau)}.\]
If, in addition, $T e^{-\phi(T)}/\log T\to0$, then
\eqref{eq:kl-noqueue-acceptance-regret} is $o(\log T)$.
Since $K$ is finite,
\[\limsup_{T\to\infty}\frac{R_T(\pi^{\texttt{KL-PBD}})}{\log T}\le c_e\sum_{\substack{j\in[K]\\ h_j^\star>\tau,\ \omega_j>0}}\frac{h_j^\star-\tau}{\kl(h_j^\star,\tau)}. \]
Uniform efficiency and Theorem~\ref{thm:instance_dependent_lower_bound} give the matching lower bound, proving the claimed limit.

If $h_j^\star\ge\tau$ for every $j\in[K]$ with $\omega_j>0$, the
acceptance contribution is identically zero. Hence the same upper
coefficient follows from \eqref{eq:kl-noqueue-finite} under
$\phi(T)/\log T\to1$ alone, and the lower bound again gives equality.

Finally, both $\phi(T)=\log T+\log\log T$ and
$\phi(T)=\log T$ satisfy
$\phi(T)/\log T\to1$ and
$T e^{-\phi(T)}/\log T\to0$.
\end{proof}

\subsection{Worst-case regret lower bound}\label{app:worst-case-lb-proof}

The main idea of the proof is to construct statistically indistinguishable instances with empty-queue states at which the optimal actions differ, while controlling how often congestion prevents the process from visiting these informative states.

\begin{proof}
Fix an algorithm $\pi$ and write $m:=\min\{\tau,1-\tau\}$. If $\sum_{j=1}^{K}\omega_j=0$, both sides of \eqref{eq:main-minimax} are zero, so assume otherwise.

We first prove a sharper lower bound for $\mu=1$. In this setting, the queue stays empty and every deferred label is observed before the next decision. Define the hard family
\begin{equation} \label{eq:lb-immediate-family} \delta_j:=\frac m4\min\!\left\{1,\frac1{\sqrt{\omega_jT}}\right\}, \qquad h_j^\sigma:=\tau+\sigma_j\delta_j, \end{equation}
with $\delta_j=0$ when $\omega_j=0$. 
Since $B_t=0$ for every $t$, the next-state law is the same under both actions, so $M_t^h(0)=c_{\rm exp}=c_e(1-\tau)$ for every $h$ and $t\le T$, and $c_{\rm hold}B_t=D(B_{T+1})=0$ along every trajectory. Resolving ties in favor of acceptance, the oracle accepts at level $j$ when $\sigma_j=+1$ and defers when $\sigma_j=-1$, and Lemma~\ref{lem:regret-decomposition} gives
\[
\begin{aligned}
R_T(\pi;h^\sigma)&=\sum_{j=1}^KR_{T,j}(\pi;h^\sigma),\\
R_{T,j}(\pi;h^\sigma)&:=c_e\delta_j\sum_{t=1}^T\mathbb P^\pi_{h^\sigma}\big(J_t=j,\ A_t\ne\mathbf1\{\sigma_j=-1\}\big).
\end{aligned}
\]

Fix $j$ with $\omega_j>0$ and neighboring sign vectors
$\sigma,\sigma^{(j)}$ with $\sigma_j=+1$ and
$\sigma^{(j)}_j=-1$. Since the two conditional accuracy vectors differ only in coordinate $j$, Lemma~\ref{lem:lower-bound-information} gives, for every $t\le T$,
\[ \mathrm{KL}\!\left(P_{h^\sigma}^{\pi,t-1}\middle\|P_{h^{\sigma^{(j)}}}^{\pi,t-1}\right) \le \omega_j(t-1)\kl(\tau+\delta_j,\tau-\delta_j) \le \frac{32\delta_j^2}{9m^2}\,\omega_jT\le\frac29. \]
For the middle inequality, every $u\in[\tau-\delta_j,\tau+\delta_j]$ satisfies $\min\{u,1-u\}\ge3m/4$, so Taylor's theorem in the first argument of the Bernoulli relative entropy gives
\[ \kl(\tau+\delta_j,\tau-\delta_j) \le\frac{(2\delta_j)^2}{2}\sup_{u\in[\tau-\delta_j,\tau+\delta_j]}\frac1{u(1-u)} \le\frac{32\delta_j^2}{9m^2}, \]
and $\delta_j^2\omega_jT\le m^2/16$ by \eqref{eq:lb-immediate-family}. Pinsker's inequality \citep[Eq.~(14.12)]{lattimore2020bandit} therefore gives $\operatorname{TV}(P_{h^\sigma}^{\pi,t-1},P_{h^{\sigma^{(j)}}}^{\pi,t-1})\le\sqrt{(2/9)/2}=1/3$.

Let $\alpha_t:=\pi_t(A_t=1\mid\mathcal F_{t-1},J_t=j)$ be the algorithm's conditional probability of deferral after the arrival of a query at level $j$. Since $J_t$ is independent of $\mathcal F_{t-1}$, for $\sigma_j=+1$,
\begin{align*}
\mathbb P_{h^\sigma}^\pi(J_t=j,A_t=1)+\mathbb P_{h^{\sigma^{(j)}}}^\pi(J_t=j,A_t=0) &=\omega_j\big[1+\mathbb E_{h^\sigma}^\pi[\alpha_t]-\mathbb E_{h^{\sigma^{(j)}}}^\pi[\alpha_t]\big] \\
&\ge\omega_j\big(1-\operatorname{TV}(P_{h^\sigma}^{\pi,t-1},P_{h^{\sigma^{(j)}}}^{\pi,t-1})\big)\\
&\ge\frac{2\omega_j}3,
\end{align*}
using $0\le\alpha_t\le1$ and that $\alpha_t$ is $\mathcal F_{t-1}$-measurable. Summing over $t$ gives $R_{T,j}(\pi;h^\sigma)+R_{T,j}(\pi;h^{\sigma^{(j)}})\ge2c_e\delta_j\omega_jT/3$. Pairing the $2^{K-1}$ sign vectors with identical $\sigma_{-j}$ and summing over $j$ gives
\[ 2^{-K}\sum_{\sigma}R_T(\pi;h^\sigma) \ge \frac{c_e}{3}\sum_{j=1}^K\delta_j\omega_jT =\frac{c_em}{12}\sum_{j=1}^K\min\{\omega_jT,\sqrt{\omega_jT}\}. \]
The supremum over $h$ is at least this average, proving the $\mu=1$ sharpening stated after Theorem~\ref{thm:known-M-lower}.

Assume from now on that $\mu<1$. If $c_{\rm hold}=0$, the right-hand side of \eqref{eq:main-minimax} is zero, so there is nothing to prove. We therefore take $c_{\rm hold}>0$.

\paragraph{Self-centered heterogeneous family.}
We next recenter each sign pattern so that its empty-queue action gaps remain exactly controlled despite the queue-dependent threshold.
For $\omega_j>0$, set
\begin{equation} \label{eq:lb-delta-sharp} \delta_j := \frac1{24} \min\!\left\{ \mu^2m\min\!\left\{1,\frac1{\sqrt{\omega_jT}}\right\}, \frac{\mu c_{\rm hold}}{c_e} \right\}, \end{equation}
and set $\delta_j=0$ if $\omega_j=0$. Write $\delta_{\max}:=\max_j\delta_j$.

For each $\sigma\in\{-1,+1\}^K$, Lemma~\ref{lem:lb-self-centered-family} provides a unique center $a_\sigma$, with $S^-_\sigma:=\sum_{j:\sigma_j=-1}\omega_j\delta_j$. Set
\begin{equation} \label{eq:lb-hard-family-sharp} h_j^\sigma:=\tau-a_\sigma+\sigma_j\delta_j. \end{equation}

Write $u_\sigma := u_{h^\sigma}$, $\mathbb{P}^s_\sigma := \mathbb{P}^{\pi,s}_{h^\sigma}$, $\mathbb{P}_\sigma := \mathbb{P}^\pi_{h^\sigma}$ and $\mathbb{E}_\sigma := \mathbb{E}^\pi_{h^\sigma}$. By Lemma~\ref{lem:lb-self-centered-family},
\begin{equation} \label{eq:lb-self-center-sharp} \overline\tau(h^\sigma)=\tau-a_\sigma, \end{equation}
so $\overline M^{h^\sigma}(0)-L^{h^\sigma}(j)=\sigma_jc_e\delta_j$. Moreover,
\begin{equation} \label{eq:lb-center-size} 0\le a_\sigma \le \frac{1-\mu}{\mu}S^-_\sigma. \end{equation}

\paragraph{Testing neighboring instances.}
We now show that flipping one sign changes the observations too little to identify the optimal empty-queue action reliably.
By \eqref{eq:lb-center-size} and \eqref{eq:lb-delta-sharp}, every hard parameter lies within $m/12$ of $\tau$. The inequality $\log x\le x-1$ gives $\kl(p,q)\le(p-q)^2/[q(1-q)]$; since $q(1-q)\ge(3m/4)^2$ on the hard family, $\kl(p,q)\le16(p-q)^2/(9m^2)$. Using the neighboring-center bound \eqref{eq:lb-neighbor-sharp} from Lemma~\ref{lem:lb-self-centered-family}, for neighboring sign vectors with $\sigma_j=+1$ and $\sigma^{(j)}_j=-1$, write $\Delta a:=a_{\sigma^{(j)}}-a_\sigma$ and $\Omega:=\sum_{k=1}^K\omega_k$. Then
\begin{align*}
\sum_{k=1}^K\omega_k|h_k^\sigma-h_k^{\sigma^{(j)}}|^2 &= 4\omega_j\delta_j^2 +4\omega_j\delta_j\Delta a +\Omega(\Delta a)^2\\
&\le \frac{4}{\mu^2}\omega_j\delta_j^2.
\end{align*}
Lemma~\ref{lem:lower-bound-information} and \eqref{eq:lb-delta-sharp} therefore give $\mathrm{KL}(P_\sigma^{t-1}\|P_{\sigma^{(j)}}^{t-1})\le\mu^2/81$ uniformly over $t\le T$. Pinsker's inequality \citep[Eq.~(14.12)]{lattimore2020bandit} then yields
\begin{equation} \label{eq:lb-tv-sharp} \operatorname{TV}(P_\sigma^{t-1},P_{\sigma^{(j)}}^{t-1})\le\frac\mu{12}, \qquad t\le T. \end{equation}

\paragraph{Empty-queue occupancy and average-cost disadvantage.}
We combine \eqref{eq:lb-hard-barrier} and \eqref{eq:lb-tv-sharp}: either the queue is often empty, forcing decision errors, or keeping it occupied requires costly deferrals. Let $N_0(T):=\sum_{t=1}^T\mathbf1\{B_t=0\}$. For each hard instance and level $j$, define
\[ \mathcal D_{\sigma,j}:=\mathbb E_\sigma\!\left[\sum_{t=1}^T \mathfrak d_{h^\sigma}(B_t,J_t,A_t)\mathbf1\{J_t=j\}\right], \qquad \mathcal D_\sigma:=\sum_{j=1}^K\mathcal D_{\sigma,j}. \]

Let $r$ be the conditional service median in Lemma~\ref{lem:lb-local-queue}. By \eqref{eq:lb-general-barrier}, \eqref{eq:lb-self-center-sharp}, and \eqref{eq:lb-delta-sharp}, for $j\in[K]$ and $r\le b<\bar B$,
\begin{equation} \label{eq:lb-hard-barrier} \mathfrak d_{h^\sigma}(b,j,1) \ge \frac{c_{\rm hold}}2-c_e\delta_{\max} \ge \frac{11}{24}c_{\rm hold} =:g_{\rm bar}. \end{equation}

Let $H_T:=\sum_{t=1}^T\mathbf1\{B_t\ge r,A_t=1\}$. Lemma~\ref{lem:lb-empty-occupancy} and \eqref{eq:lb-hard-barrier} give
\begin{equation} \label{eq:lb-empty-occupancy-sharp} \mathbb E_\sigma [N_0(T)] \ge \frac{\mu T}{2}-2\mathbb E_\sigma [H_T] \ge \frac{\mu T}{2} - \frac{48}{11c_{\rm hold}}\mathcal D_\sigma. \end{equation}

Pair $\sigma$ and $\sigma^{(j)}$. For the positive sign, every level $j$ deferral has disadvantage at least $c_e\delta_j$; for the negative sign, accepting a level-$j$ query at an empty queue has disadvantage exactly $c_e\delta_j$. With $\alpha_t:=\pi_t(A_t=1\mid\mathcal F_{t-1},J_t=j)$ as in the case $\mu=1$, and since $\mathbf1\{B_t=0\}\alpha_t$ and $\mathbf1\{B_t=0\}(1-\alpha_t)$ are $\mathcal F_{t-1}$-measurable and $[0,1]$-valued, the variational characterization of total variation gives
\[
\begin{aligned}
&\mathbb E_\sigma[\mathbf1\{B_t=0\}\alpha_t] +\mathbb E_{\sigma^{(j)}}[\mathbf1\{B_t=0\}(1-\alpha_t)]\\
&\qquad\ge \frac{\mathbb P_\sigma(B_t=0)+\mathbb P_{\sigma^{(j)}}(B_t=0)}2 -\operatorname{TV}(P_\sigma^{t-1},P_{\sigma^{(j)}}^{t-1}).
\end{aligned}
\]
Summing this inequality and using \eqref{eq:lb-tv-sharp} yields
\begin{align*}
\frac12\bigl(\mathcal D_{\sigma,j}+\mathcal D_{\sigma^{(j)},j}\bigr) \ge \frac{c_e\delta_j\omega_j}{2} \left[ \frac{\mathbb E_\sigma [N_0(T)]+\mathbb E_{\sigma^{(j)}}[N_0(T)]}2 - \frac{\mu T}{12} \right].
\end{align*}
Averaging over all neighboring pairs and summing over $j$, with $S_\delta:=\sum_{j=1}^K\omega_j\delta_j$ and $\overline{\mathcal D}:=2^{-K}\sum_\sigma\mathcal D_\sigma$, gives
\[ \overline{\mathcal D} \ge \frac{5\mu T}{24}c_eS_\delta - \frac{24c_eS_\delta}{11c_{\rm hold}}\overline{\mathcal D}. \]
Since $c_eS_\delta\le\mu c_{\rm hold}/24$, we conclude that
\begin{equation} \label{eq:lb-Dbar-sharp} \overline{\mathcal D} \ge \frac{55}{288}\,\mu T c_eS_\delta. \end{equation}

\paragraph{From disadvantage to finite-horizon regret.}
It remains to control the terminal bias correction without introducing any dependence on the capacity $\bar{B}$. For a general conditional accuracy vector $h$, write $\mathcal D_h(\pi):=\mathbb E_h^\pi[\sum_{t=1}^T\mathfrak d_h(B_t,J_t,A_t)]$. We only need a one-sided terminal comparison. By Lemma~\ref{lem:lb-average-cost}, $u_h\ge D$.
Because $B_1=0$, \eqref{eq:lb-cost-identity} shows that a stationary average-cost minimizer has finite-horizon cost at most $Tg_h$, while an arbitrary algorithm has cost $Tg_h+\mathcal D_h(\pi)+\mathbb E_h^\pi[(D-u_h)(B_{T+1})]$. Consequently,
\begin{equation} \label{eq:lb-one-sided-regret} R_T(\pi;h) \ge \mathcal D_h(\pi) - \mathbb E_h^\pi[(u_h-D)(B_{T+1})]. \end{equation}

For the hard instance $h^\sigma$, Lemma~\ref{lem:lb-terminal-comparison} gives
\[ 0\le u_\sigma(b)-D(b) \le\frac{2c_eS^-_\sigma}{\mu}\left\lceil\frac br\right\rceil. \]
Pathwise,
$\lceil B_{T+1}/r\rceil\le1+H_T$, because a deferral below $r$ cannot raise the queue above $r$, while each deferral at or above $r$ increases this coarse queue state by at most one. Together with \eqref{eq:lb-hard-barrier}, \eqref{eq:lb-one-sided-regret} yields
\[ R_T(\pi;h^\sigma) \ge \frac9{11}\mathcal D_\sigma - \frac{2c_eS_\delta}{\mu}. \]
Averaging over the hard family and using \eqref{eq:lb-Dbar-sharp},
\begin{equation} \label{eq:lb-regret-sharp-large} \sup_hR_T(\pi;h) \ge c_eS_\delta \left(\frac{5\mu T}{32}-\frac2\mu\right). \end{equation}

Independently, the first decision gives a complementary bound valid for every $T$. Consider the family $h_j^\sigma=\tau+\sigma_jm/4$. By \eqref{eq:lb-finite-empty-marginal}, Lemma~\ref{lem:lb-one-job-floor} at $h^0$, and the adjacent-increment estimate \eqref{eq:two-sided-increment-stability} from the proof of Proposition~\ref{prop:bellman-stability}, $|M_1^{h^\sigma}(0)-M_{\tau}|\le(1-\mu)c_em/4$. Since $|L^{h^\sigma}(j)-M_{\tau}|=c_em/4$, opposite signs in coordinate $j$ induce opposite optimal actions at $B_1=0$, each with gap at least $\mu c_em/4$. Before the first action the algorithm has no accuracy-dependent observations, so averaging each sign pair forces an error probability at least $1/2$ conditional on the arrival of a query at level $j$. Summing over $j$ gives
\begin{equation} \label{eq:lb-first-round} \sup_hR_T(\pi;h) \ge \frac{\mu c_em}{8}\sum_{j=1}^K\omega_j. \end{equation}
Finally, define
\[ \mathcal L_T := \sum_{j=1}^K \min\!\left\{ \mu^3c_em\min\{\omega_jT,\sqrt{\omega_jT}\}, \mu^2c_{\rm hold}\omega_jT \right\}. \]
By \eqref{eq:lb-delta-sharp}, $\mathcal L_T=24\mu Tc_eS_\delta$. Therefore \eqref{eq:lb-regret-sharp-large} gives
\begin{equation} \label{eq:lb-assembly-large} \sup_hR_T(\pi;h)\ge\mathcal L_T\left(\frac5{768}-\frac1{12\mu^2T}\right), \end{equation}
while \eqref{eq:lb-first-round} and $\mathcal L_T\le\mu^3c_emT\sum_j\omega_j$ give
\begin{equation} \label{eq:lb-assembly-small} \sup_hR_T(\pi;h)\ge\frac{\mathcal L_T}{8\mu^2T}. \end{equation}
If $\mu^2T\ge32$, \eqref{eq:lb-assembly-large} is at least $\mathcal L_T/256$; if $\mu^2T\le32$, \eqref{eq:lb-assembly-small} is at least $\mathcal L_T/256$. Since $\pi$ was arbitrary, taking the infimum over algorithms proves \eqref{eq:main-minimax}.
\end{proof}

\begin{remark}[Worst case over $\omega$]
\label{rem:worst_case_distribution}
The function $f(x)=\min\{x,\sqrt x\}$ is nondecreasing and concave on $[0,\infty)$, so Jensen's inequality gives $\sum_{j=1}^Kf(\omega_jT)\le Kf(T/K)=\min\{T,\sqrt{KT}\}$, with equality when $\omega_0=0$ and $\omega_j=1/K$ for every $j\in[K]$.
\end{remark}

\subsection{Technical lemmas}
\label{app:lb-technical}

\subsubsection{Proofs of the common ingredients}

\begin{proof}[Proof of Lemma~\ref{lem:lb-average-cost}]
By Lemma~\ref{lem:service-queue-increments}, each $\xi_s^h$ is
nonnegative and nondecreasing in $q$. Its update is given by \eqref{eq:service-smoothed-increment}, \eqref{eq:service-increment-recursion}, and \eqref{eq:service-increment-boundary}.

Write $g_D(q):=D(q)-D(q-1)$ and $d_0(b):=\sum_{c=0}^b\rho_C(c)g_D(b+1-c)$.
Differencing \eqref{eq:lb-drain} gives $g_D(q)=c_{\rm hold}+d_0(q-1)$.
Since $D$ is discrete convex, $d_0$ is nondecreasing. The median or maximum in the increment update is therefore at least $d_0(q-1)$, giving $\xi_1^h(q)\ge g_D(q)=\xi_0^h(q)$. The update is coordinatewise nondecreasing, so $\xi_{s+1}^h\ge\xi_s^h$ for every $s\ge0$.

By \eqref{eq:service-adjacent-value-bound}, 
$\xi_s^h(q)\le qc_{\rm hold}/\mu+(c_e-c_{\rm exp})_+$, uniformly in $s$.
Hence $\xi_s^h(q)\uparrow x_h(q)$, and $W_s^h(b)-W_s^h(0)\to u_h(b):=\sum_{q=1}^b x_h(q)$. In particular, $u_h\ge D$ because $\xi_0^h=g_D$.

Let $\mathcal T_h$ be the queue Bellman operator. With $v_s^h:=W_s^h-W_s^h(0)$ and $g_s^h:=W_{s+1}^h(0)-W_s^h(0)$, translation invariance gives
$g_s^h+v_{s+1}^h=\mathcal T_hv_s^h$ and $g_s^h=(\mathcal T_hv_s^h)(0)$.
The operator is continuous because the state and action spaces
are finite. Thus $g_s^h\to g_h:=(\mathcal T_hu_h)(0)$, and
passing to the limit yields \eqref{eq:lb-acoe}.
Passing the increment identities to the limit proves (i) and
preserves monotonicity and discrete convexity.

For (ii), \eqref{eq:lb-smoothed} shows that $d_h(b)$ is
nonnegative and nondecreasing in $b$. At the empty queue,
$d_h(0)=\rho_C(0)x_h(1)=(1-\mu)u_h(1)$, proving
\eqref{eq:lb-empty-marginal}.

For (iii), using \eqref{eq:lb-acoe}, independence of $J_t$ from $\mathcal F_{t-1}$, and the tower property,
\begin{align}\label{eq:lb-cost-identity}
\mathbb E_h^\pi\!\left[\sum_{t=1}^T\ell_h(B_t,J_t,A_t)+D(B_{T+1})\right] &={}Tg_h+\mathbb E[u_h(B_1)] +\mathbb E_h^\pi[(D-u_h)(B_{T+1})] \nonumber \\
&+\mathbb E_h^\pi\!\left[\sum_{t=1}^T\mathfrak d_h(B_t,J_t,A_t)\right].
\end{align}

A stationary policy attaining the minima in \eqref{eq:lb-acoe} has zero disadvantage. Therefore every algorithm has cost at least $Tg_h+\mathbb E[u_h(B_1)]+\min_b(D-u_h)(b)$, while that stationary policy has cost at most the same expression with the minimum replaced by the maximum. The finite-horizon optimum lies between these two quantities, which proves \eqref{eq:lb-regret-average}.

For (iv), evaluating the minimum in \eqref{eq:lb-acoe} at the action distribution of $\varphi$ gives $g_h+u_h(b)\le\sum_{j=0}^K\omega_j\sum_{a}\varphi(a\mid b,j)\{\ell_h(b,j,a)+(P_au_h)(b)\}$ for every $b$. Averaging under $\lambda$, stationarity cancels the two $u_h$ terms and leaves \eqref{eq:lb-gain-optimal}.
\end{proof}

\begin{proof}[Proof of Lemma~\ref{lem:lower-bound-information}]
For $s\ge1$, set $\Delta N_{j,s}:=N_{j,s+1}-N_{j,s}$, order the labels completed in round $s$ by submission time, and expose them one at a time. Conditional on $\mathcal F_{s-1}$, $J_s$, $A_s$, $C_s$, and the labels already exposed in round $s$, each newly exposed level $j$ label has law $\operatorname{Bernoulli}(h_j)$ under $h$ and $\operatorname{Bernoulli}(\widetilde h_j)$ under $\widetilde h$; the remaining observations have the same conditional law under the two instances. The KL chain rule \citep[Exercise~14.12, Eq.~(14.17)]{lattimore2020bandit} therefore gives
\begin{align*}
\mathrm{KL}(P_h^{\pi,t}\|P_{\widetilde h}^{\pi,t}) &=\sum_{s=1}^t\mathbb E_h^\pi\!\left[\sum_{j=1}^K\Delta N_{j,s}\kl(h_j,\widetilde h_j)\right]\\
&=\sum_{j=1}^K\mathbb E_h^\pi[N_{j,t+1}]\kl(h_j,\widetilde h_j).
\end{align*}
Since $N_{j,t+1}\le D_{j,t}\le\sum_{s=1}^t\mathbf1\{J_s=j\}$ pathwise, $\mathbb E_h^\pi[N_{j,t+1}]\le\mathbb E_h^\pi[D_{j,t}]\le t\omega_j$, proving \eqref{eq:lb-information} and \eqref{eq:lb-information-arrivals}.
\end{proof}

\begin{proof}[Proof of Lemma~\ref{lem:lb-one-job-floor}]
\Eqref{eq:lb-finite-empty-marginal} follows from \eqref{eq:service-smoothed-increment} at $b=0$. Since $x_{t+1}^h(1)=\xi_{T-t}^h(1)$, the monotone convergence proved in Lemma~\ref{lem:lb-average-cost} gives both the time monotonicity and the limit, using \eqref{eq:lb-empty-marginal}. Also, \eqref{eq:lb-drain} gives $D(1)=c_{\rm hold}+(1-\mu)D(1)$, so $\xi_0^h(1)=c_{\rm hold}/\mu$ and $M_T^h(0)=M_\tau$.

Suppose now that $\tau\in[0,1]$. At $h^0$, $L^{h^0}(j)=M_\tau$ for every $j\in[K]$. By Proposition~\ref{prop:queue-threshold} and
\eqref{eq:lb-threshold-sandwich}, $M_t^{h^0}(b)\ge M_\tau$ whenever deferral is feasible. Hence always accepting is optimal for every horizon.
Under this policy, the difference in expected cost between starting from queue state $b$ and from zero is exactly $D(b)$. Thus $W_s^{h^0}(b)-W_s^{h^0}(0)=D(b)$ for every $s$. Taking the limit gives $u_{h^0}=D$, and the finite-horizon empty-queue marginals all equal $M_\tau$.
\end{proof}

\subsubsection{Lemmas for the worst-case lower bound}

For the following four lemmas, assume \(0<\mu<1\), \(c_{\rm hold}>0\), and \(\tau\in(0,1)\), and write \(m:=\min\{\tau,1-\tau\}\).
Choose an integer $r\ge1$ that is a median of $C_t$ conditional on $C_t>0$, which exists since $\mu>0$, so that
\begin{equation} \label{eq:lb-service-median} p_-:=\mathbb P(1\le C_t\le r)\ge\frac\mu2, \qquad p_+:=\mathbb P(C_t\ge r)\ge\frac\mu2. \end{equation}
Let
\[ \mathcal N := \left\{h:\|h-h^0\|_\infty< \min\!\left\{\frac m2,\frac{c_{\rm hold}}{4c_e}\right\}\right\}. \]
Starting from $B_1=b$ under a stationary policy $\varphi$,
write $\mathbb E_b^\varphi$ for expectation and let
\(\tau_0:=\inf\{s\ge0:B_{s+1}=0\}\) be the number of rounds until the queue first becomes empty.

\begin{lemma}[Local queue estimates]
\label{lem:lb-local-queue}
Fix \(h\in\mathcal N\). For \(r\le b<\bar B\),
\begin{equation} \label{eq:lb-general-barrier} \overline M^h(b)-\overline M^h(0) \ge\frac{c_{\rm hold}}2, \end{equation}
every stationary minimizing policy accepts at \(b\), and \(\mathfrak d_h(b,j,1)\ge c_{\rm hold}/4\) for \(j\in[K]\). Under every stationary minimizing policy, and under the policy that always accepts,
\begin{equation} \label{eq:lb-hit-general} \mathbb E_b[\tau_0]\le\frac{\lceil b/r\rceil}{p_+}\le\frac{2\lceil b/r\rceil}{\mu}, \qquad b\ge1. \end{equation}
For every stationary policy $\varphi$ with $\mathbb E_b^\varphi[\tau_0]<\infty$ for every $b$,
\begin{equation} \label{eq:lb-hitting} u_h(b) \le \mathbb E_b^\varphi\!\left[ \sum_{t=1}^{\tau_0} \bigl(\ell_h(B_t,J_t,A_t)-g_h\bigr) \right], \end{equation}
with equality when $\varphi$ selects an ACOE-minimizing action at every state.
\end{lemma}

\begin{proof}
For the solution $u_h$ from Lemma~\ref{lem:lb-average-cost}, \eqref{eq:lb-smoothed}, monotonicity of $(x_h(q))_q$, and \eqref{eq:lb-threshold-sandwich} give, for every $b\ge r$ with $b<\bar B$,
\[ \overline M^h(b)-\overline M^h(0) \ge \sum_{c=1}^b\rho_C(c)x_h(1) \ge\frac{c_{\rm hold}}\mu\mathbb P(1\le C_t\le b) \ge\frac{c_{\rm hold}}2, \]
which is \eqref{eq:lb-general-barrier}. Because $\overline M^h(0)\ge M_{\tau}$ and $\|h-h^0\|_\infty<c_{\rm hold}/(4c_e)$, we have $L^h(j)\le M_{\tau}+c_{\rm hold}/4$. 
Thus every stationary minimizing policy for $h\in\mathcal N$
accepts every query at $b\ge r$. For $j\in[K]$ and
$r\le b<\bar B$, we also have $\overline M^h(b)-L^h(j)\ge c_{\rm hold}/4$.

From queue state one the queue then stays at most $r$ until it hits zero, so under any stationary minimizing policy
\begin{equation} \label{eq:lb-hit-one} \mathbb E_1[\tau_0]\le\frac1{p_+}\le\frac2\mu. \end{equation}
More generally, positive-service events of size at least $r$ reduce any queue state by at least $r$ until it enters $\{0,\ldots,r\}$. This proves \eqref{eq:lb-hit-general} under every stationary minimizing policy for $h\in\mathcal N$. The same drainage bounds hold for the policy that always accepts.

Let $\varphi$ be a stationary policy with
$\mathbb E_b^\varphi[\tau_0]<\infty$ for every $b$.
Evaluating the minimum in \eqref{eq:lb-acoe} at $\varphi$,
summing over rounds $1,\ldots,\tau_0\wedge s$, and letting
$s\to\infty$ by dominated convergence gives \eqref{eq:lb-hitting}, with equality when $\varphi$ selects an ACOE-minimizing
action at every state.
\end{proof}

\begin{lemma}[Self-centered family]
\label{lem:lb-self-centered-family}
Fix \(T\ge1\), let \(\delta_j\) be given by \eqref{eq:lb-delta-sharp}, with \(\delta_j=0\) when \(\omega_j=0\), and write \(\delta_{\max}:=\max_j\delta_j\). Let \(r\) and \(\mathcal N\) be as in Lemma~\ref{lem:lb-local-queue}.

For $\sigma\in\{-1,+1\}^K$, let $S^-_\sigma:=\sum_{j:\sigma_j=-1}\omega_j\delta_j$ and $A_\sigma:=(1-\mu)S^-_\sigma/\mu$. There is a unique \(a_\sigma\in[0,A_\sigma]\) for which the vector \(h^\sigma_j:=\tau-a_\sigma+\sigma_j\delta_j\) satisfies \(\overline\tau(h^\sigma)=\tau-a_\sigma\). It belongs to \(\mathcal N\). If \(\sigma_j=+1\) and \(\sigma^{(j)}\) is obtained by flipping the \(j\)th sign, then \(0\le a_{\sigma^{(j)}}-a_\sigma\le2(1-\mu)\omega_j\delta_j/\mu\).
\end{lemma}

\begin{proof}
We need the ordered stability from Proposition~\ref{prop:bellman-stability}, applied to the values with $s$ rounds remaining: for $h'\le h$ coordinatewise and every $s\ge0$,
\begin{equation} \label{eq:lb-long-horizon-ordered} 0\le \bigl[W_s^{h'}(1)-W_s^{h'}(0)\bigr] - \bigl[W_s^h(1)-W_s^h(0)\bigr] \le c_e\|h-h'\|_\infty. \end{equation}
For $a\ge0$ such that $h-a\mathbf1\in[0,1]^K$, let
$s\to\infty$ with $h':=h-a\mathbf1$.
Using \eqref{eq:lb-empty-marginal} yields
\begin{equation} \label{eq:lb-shift-contraction-sharp} 0 \le \overline\tau(h)-\overline\tau(h-a\mathbf1) \le (1-\mu)a. \end{equation}

Define, for $a\in[0,A_\sigma]$,
\( F_\sigma(a) := \tau - \overline\tau(\tau\mathbf1-a\mathbf1+\sigma\odot\delta) \)
where $\odot$ denotes componentwise multiplication. Since $A_\sigma\le\delta_{\max}/\mu$, for $a\in[0,A_\sigma]$ we have
\[ a+\delta_{\max} \le \frac{\mu m}{24}+\frac{\mu^2m}{24} \le\frac m{12}, \qquad a+\delta_{\max} \le \frac{c_{\rm hold}}{24c_e}+\frac{\mu c_{\rm hold}}{24c_e} \le\frac{c_{\rm hold}}{12c_e}, \]
so every conditional accuracy vector in this interval belongs to $\mathcal N$. \Eqref{eq:lb-shift-contraction-sharp} gives $|F_\sigma(a)-F_\sigma(a')|\le(1-\mu)|a-a'|$, so $F_\sigma$ is a contraction.

We next check that $F_\sigma$ maps $[0,A_\sigma]$ into itself. Fix $a$ in this interval, write $h=\tau\mathbf1-a\mathbf1+\sigma\odot\delta$, and set $z:=F_\sigma(a)$. Set $w_h:=(d_h(1)-d_h(0))/c_e\ge0$ when $\bar B\ge2$ and $w_h:=+\infty$ when $\bar B=1$. By \eqref{eq:lb-empty-marginal} and the definition of $z$, $d_h(0)=(1-\mu)u_h(1)=(1-\mu)c_{\rm hold}/\mu+c_ez$, hence $r_j^h-d_h(0)=c_e(a-\sigma_j\delta_j-z)$. Since $\operatorname{med}\{d_h(0),d_h(1),r\}-d_h(0)=\min\{(r-d_h(0))_+,d_h(1)-d_h(0)\}$, \eqref{eq:lb-increment} at $q=1$ (\eqref{eq:lb-increment-boundary} when $\bar B=1$) gives
\begin{equation} \label{eq:lb-q1-map} \mu z = (1-\mu) \sum_{j=1}^K\omega_j \min\!\left\{(a-\sigma_j\delta_j-z)_+,w_h\right\}. \end{equation} In particular $z\ge0$. If $z>A_\sigma$, then $z>a$ and only negative coordinates can contribute to the right-hand side, each by at most $\delta_j$. Thus $\mu z\le(1-\mu)S^-_\sigma=\mu A_\sigma$, a contradiction. Since $F_\sigma$ maps the closed interval $[0,A_\sigma]$ into itself and is a contraction, it has a unique fixed point $a_\sigma\in[0,A_\sigma]$.

Set \(h_j^\sigma:=\tau-a_\sigma+\sigma_j\delta_j\), as in \eqref{eq:lb-hard-family-sharp}, and write \(g_\sigma:=g_{h^\sigma}\), \(u_\sigma:=u_{h^\sigma}\). The fixed-point identity gives \eqref{eq:lb-self-center-sharp}, so \(\overline M^{h^\sigma}(0)-L^{h^\sigma}(j)=\sigma_jc_e\delta_j\). Moreover, setting \(z=a_\sigma\) in \eqref{eq:lb-q1-map} yields \eqref{eq:lb-center-size}.

We need one sharper comparison between neighboring centers. Let $\sigma_j=+1$ and let $\sigma^{(j)}$ be obtained by flipping the $j$th sign. At a common value of $a$, the conditional accuracy vector for $\sigma^{(j)}$ is coordinatewise no larger than that for $\sigma$. 
Applying \eqref{eq:lb-long-horizon-ordered} and letting $s\to\infty$ gives $F_{\sigma^{(j)}}(a)\ge F_\sigma(a)$. Both maps are nondecreasing contractions. Starting the iteration of $F_{\sigma^{(j)}}$ at $a_\sigma$ therefore produces a nondecreasing sequence whose limit is its fixed point, so $a_{\sigma^{(j)}}\ge a_\sigma$. Write $\Delta a:=a_{\sigma^{(j)}}-a_\sigma$. Evaluating \eqref{eq:lb-acoe} at $b=0$, where $u_\sigma(0)=P_0u_\sigma(0)=0$ and $P_1u_\sigma(0)=\overline M^{h^\sigma}(0)-c_{\rm exp}$, and using \eqref{eq:lb-self-center-sharp} gives
\begin{equation} \label{eq:lb-gain-formula} g_\sigma = \Omega(M_{\tau}+c_ea_\sigma) -c_e\sum_{k:\sigma_k=+1}\omega_k\delta_k, \qquad \Omega:=\sum_{k=1}^K\omega_k. \end{equation}
Let $\varphi_\sigma$ be a stationary minimizing policy under $h^\sigma$. Starting from queue state one, the queue stays at most $r$ until its first hitting time $\tau_0$ of zero, so \eqref{eq:lb-hit-one} gives $\mathbb E_1^{\varphi_\sigma}[\tau_0]\le2/\mu$. From \eqref{eq:lb-gain-formula},
$g_{\sigma^{(j)}}-g_\sigma=c_e\Omega\Delta a+c_e\omega_j\delta_j$.
For the same queue state and action rule $\varphi_\sigma$, the conditional one-step cost difference satisfies
\[ \mathbb E\!\left[ \ell_{h^{\sigma^{(j)}}}(B_t,J_t,A_t) - \ell_{h^\sigma}(B_t,J_t,A_t) \,\middle|\,B_t \right] \le c_e\Omega\Delta a+2c_e\omega_j\delta_j, \]
because acceptance losses increase by $c_e\Delta a$ at every self-confidence level and by an additional $2c_e\delta_j$ at level $j$, while deferral costs do not change. Hence
\[ \mathbb E\!\left[ (\ell_{h^{\sigma^{(j)}}}-g_{\sigma^{(j)}}) - (\ell_{h^\sigma}-g_\sigma) \,\middle|\,B_t \right] \le c_e\omega_j\delta_j. \]
By \eqref{eq:lb-hitting} at $b=1$, with equality for $h^\sigma$ under $\varphi_\sigma$ and as an inequality for $h^{\sigma^{(j)}}$ under the same $\varphi_\sigma$, subtracting the two relations yields
\[ u_{\sigma^{(j)}}(1)-u_\sigma(1) \le c_e\omega_j\delta_j\,\mathbb E_1^{\varphi_\sigma}[\tau_0] \le\frac{2c_e}{\mu}\omega_j\delta_j. \]

Since \eqref{eq:lb-self-center-sharp} and \eqref{eq:lb-empty-marginal} give
$c_e\Delta a=(1-\mu)\{u_{\sigma^{(j)}}(1)-u_\sigma(1)\}$, we obtain
\begin{equation} \label{eq:lb-neighbor-sharp} 0\le \Delta a \le \frac{2(1-\mu)}\mu\,\omega_j\delta_j. \end{equation}

\end{proof}

\begin{lemma}[Empty-queue occupancy]
\label{lem:lb-empty-occupancy}
Fix a conditional accuracy vector \(h\) and an algorithm \(\pi\), and let \(r\) be as in Lemma~\ref{lem:lb-local-queue}. Write \(N_0(T):=\sum_{t=1}^T\mathbf1\{B_t=0\}\) and \(H_T:=\sum_{t=1}^T\mathbf1\{B_t\ge r,A_t=1\}\). Then \(\mathbb E_h^\pi[N_0(T)]\ge\mu T/2-2\mathbb E_h^\pi[H_T]\).
\end{lemma}

\begin{proof}
Write \(p_+:=\mathbb P(C_t\ge r)\), as in \eqref{eq:lb-service-median}, and drop the sub- and superscripts of \(\mathbb E_h^\pi\).
With $W_t:=\lceil(B_t-r)_+/r\rceil$, pathwise $W_{t+1}-W_t\le\mathbf1\{B_t\ge r,A_t=1\}-\mathbf1\{B_t>r,A_t=0,C_t\ge r\}$. Thus
\[ p_+\,\mathbb E\!\left[\sum_{t=1}^T\mathbf1\{B_t>r,A_t=0\}\right] \le\mathbb E [H_T]. \]
A round with $B_t+A_t\le r$ and $C_t\ge r$ leaves the next queue empty; the number of rounds with $B_t+A_t>r$ is at most the count in the preceding display plus $H_T$. Since $B_1=0$, the number of empty next-round queues $B_{t+1}=0$ is at most $N_0(T)$. Using \eqref{eq:lb-service-median} therefore gives
\(\mathbb E[N_0(T)]\ge\mu T/2-2\mathbb E[H_T]\), as claimed.
\end{proof}

\begin{lemma}[Terminal comparison for the hard family]
\label{lem:lb-terminal-comparison}
For every \(h^\sigma\) from Lemma~\ref{lem:lb-self-centered-family} and every \(b\in\mathcal B\), write \(u_\sigma:=u_{h^\sigma}\) and let \(r\) be as in Lemma~\ref{lem:lb-local-queue}. Then \(0\le u_\sigma(b)-D(b)\le(2c_eS^-_\sigma/\mu)\lceil b/r\rceil\).
\end{lemma}

\begin{proof}

Let $\varphi^0$ be the stationary policy that always accepts and $\tau_0$ its first hitting time of queue state zero. By \eqref{eq:lb-hit-general}, $\mathbb E_b^{\varphi^0}[\tau_0]\le2\lceil b/r\rceil/\mu$. Moreover, \eqref{eq:lb-hitting} gives
\[
u_\sigma(b)\le\mathbb E_b^{\varphi^0}\!\left[\sum_{t=1}^{\tau_0}(\ell_{h^\sigma}(B_t,J_t,0)-g_{h^\sigma})\right],\qquad
D(b)=\mathbb E_b^{\varphi^0}\!\left[\sum_{t=1}^{\tau_0}(\ell_{h^0}(B_t,J_t,0)-g_{h^0})\right],
\]
where the second identity uses Lemma~\ref{lem:lb-one-job-floor}, since $\varphi^0$ is a stationary minimizing policy at $h^0$.
Moreover, \eqref{eq:lb-gain-formula} gives, at every queue state $b$,
\[ \sum_{j=0}^K\omega_j \Big[ \ell_{h^\sigma}(b,j,0)-g_\sigma - \bigl(\ell_{h^0}(b,j,0)-g_{h^0}\bigr) \Big] = c_eS^-_\sigma. \]
Subtracting the two stopped-cost representations and using the last display gives $u_\sigma(b)-D(b)\le c_eS^-_\sigma\,\mathbb E_b^{\varphi^0}[\tau_0]$; together with the bound on \(\mathbb E_b^{\varphi^0}[\tau_0]\) and $u_\sigma\ge D$ from Lemma~\ref{lem:lb-average-cost}, this proves the claim.
\end{proof}

\subsubsection{Lemmas for the instance-dependent lower bound}

\begin{lemma}[Transformed disadvantages]
\label{lem:lb-transformed-disadvantage}
For $h\in[0,1]^K$, set $k_h(b,j,a):=c_{\rm hold}b+a\{c_{\rm exp}-L^h(j)\}$ and $\bar g_h:=g_h-\sum_{k=0}^K\omega_kL^h(k)$. Then $(\bar g_h,u_h)$ solves \eqref{eq:lb-acoe} with $\ell_h$ replaced by $k_h$. The disadvantages values in \eqref{eq:lb-disadvantage} are unchanged, and \eqref{eq:lb-gain-optimal} holds with the transformed gains. Define $v_h(b,j):=\min_{a\in\mathcal A(b,j)}\{k_h(b,j,a)+(P_au_h)(b)\}-\bar g_h$. For a state $s=(b,k)$ and an action $a\in\mathcal A(b,k)$,
\begin{equation} \label{eq:lb-state-disadvantage} \mathfrak d_h(s,a)=k_h(s,a)-\bar g_h+\mathbb E[v_h(S_{t+1})\mid S_t=s,A_t=a]-v_h(s). \end{equation}

Moreover, for any instance, any algorithm, and any bounded $w:\mathcal B\times\{0,\ldots,K\}\to\mathbb R$, with $J_{T+1}\sim\omega$ independently after the last decision and $S_{T+1}:=(B_{T+1},J_{T+1})$, the process $\mathcal M_t:=\sum_{s=1}^t(w(S_{s+1})-\mathbb E[w(S_{s+1})\mid S_s,A_s])$, $0\le t\le T$, is a martingale with respect to $\mathcal H_t:=\mathcal F_t\vee\sigma(J_{t+1})$, with $|\mathcal M_t-\mathcal M_{t-1}|\le\operatorname{sp}(w)$ for $1\le t\le T$.
\end{lemma}

\begin{proof}

Since $\ell_h(b,j,a)=k_h(b,j,a)+L^h(j)$ and $L^h(j)$ does not depend on the action, subtracting $\sum_{k=0}^K\omega_kL^h(k)$ from \eqref{eq:lb-acoe} gives the transformed ACOE for $(\bar g_h,u_h)$. The minimizing actions and disadvantages are unchanged, and \eqref{eq:lb-gain-optimal} holds with the transformed gains.

For the martingale property, $\mathcal M_t$ is $\mathcal H_t$-measurable. Since $B_{t+1}=(B_t+A_t-C_t)^+$ with $C_t$ independent of $(\mathcal H_{t-1},A_t)$, and $J_{t+1}\sim\omega$ is independent of $(\mathcal H_{t-1},A_t,C_t)$, the conditional law of $S_{t+1}$ given $\mathcal H_{t-1}$ and $A_t$ depends only on $(S_t,A_t)$; the tower property therefore gives $\mathbb E[\mathcal M_t-\mathcal M_{t-1}\mid\mathcal H_{t-1}]=0$. The increment bound follows from $|w(S_{t+1})-\mathbb E[w(S_{t+1})\mid S_t,A_t]|\le\operatorname{sp}(w)$.
\end{proof}
\clearpage
\section{Additional related work}\label{app:additional-related}

\subsection{Alternative self-confidence signals}
\label{app:uncertainty-scores}

Our experiments use the self-consistency score, i.e. the empirical mass of the largest answer class. This choice is natural when multiple responses are sampled, and answer consistency has previously been used as a signal to decide whether to escalate queries in LLM cascades \citep{yue2024large}. However, the deferral and learning components of our model do not depend on this particular definition of the score. They only require that queries be partitioned into levels \(J_t\in[K]\), each with an unknown conditional accuracy
$h_j^\star=\Pr(Y_t=1\mid J_t=j)$ that can be learned from selectively observed expert feedback. Thus, other scalar self-confidence measures can be substituted for self-consistency and discretized into levels before applying the same deferral framework. In particular, neither the deferral rule nor our analysis requires the self-confidence signal itself to be calibrated or the conditional accuracy to be monotone in that signal.

\paragraph{Sample-based self-confidence.}
Let \(\mathcal C_t\) denote the set of answer classes represented
among the \(G\) sampled responses, with empirical frequencies
\[
q_{t,c}:=\frac{1}{G}\sum_{g=1}^G
\mathbf{1}\{\widehat R_{t,g}\in c\},
\qquad c\in\mathcal C_t.
\]
Writing \(q_{t,(1)} \ge q_{t,(2)} \ge \cdots\)
for the ordered frequencies, with \(q_{t,(2)}=0\) if \(|\mathcal C_t|=1\), the self-consistency score is simply
\(q_{t,(1)}.\) It therefore retains only the empirical mass of the most frequent answer class.

A natural alternative is to use the entropy of the empirical answer distribution,
\[H_t:=-\sum_{c\in\mathcal C_t}q_{t,c}\log q_{t,c}.\]
which takes all sampled answer classes into account. For tasks with canonical answers, the classes can be defined directly by the extracted answer, yielding an entropy over answer labels. For open-ended tasks, responses can instead be grouped according to semantic equivalence, giving the discrete variant of semantic entropy \citep{kuhn2023semantic,Farquhar2024}. Semantic entropy has also been considered directly as an uncertainty signal for LLM routing \citep{zhang2025leveraging}.

Another simple alternative is the margin between the two most frequent answer classes,
\[
q_{t,(1)} - q_{t,(2)},
\]
which measures how strongly the dominant answer is separated from its closest competitor. Margin-based self-confidence scores have been used to control escalation in language-model cascades \citep{ramirez2024optimising}. The statistic above applies the same principle directly to the empirical distribution of sampled answer classes.

For open-ended generation, self-confidence can also be estimated without first assigning responses to hard equivalence classes. Kernel Language Entropy (KLE) \citep{nikitin2024kernel}, for example, constructs a kernel from pairwise semantic similarities between sampled responses and measures uncertainty through its von Neumann entropy. Semantic voting \citep{jiang2026semantic} also aggregates pairwise semantic similarities, using them to select responses for model self-improvement.

\paragraph{Model-internal self-confidence.}
When model logits are available, self-confidence can instead be estimated from the predictive distributions produced during generation. For a generated response \(\widehat R_t=(\widehat r_{t,1},\ldots,\widehat r_{t,L_t})\), let
\[
p_{t,\ell}
:=
\Pr(\widehat r_{t,\ell}\mid X_t,\widehat r_{t,<\ell})
\]
denote the probability assigned by the model to the generated token at position \(\ell\). For example, \citet{gupta2024language} consider the sequence log-probability, which they refer to as the \emph{Chow-Sum},
\[
S_t^{\mathrm{sum}}
:=
\sum_{\ell=1}^{L_t}\log p_{t,\ell},
\]
as a measure of model confidence. Since this score is affected by sequence length, they also consider the length-normalized \emph{Chow-Average} score
\[S_t^{\mathrm{avg}}
:=
\frac{1}{L_t}\sum_{\ell=1}^{L_t}\log p_{t,\ell}.\]
These and other model-internal scores can be incorporated into our model by discretizing the resulting scalar self-confidence signal into levels and learning the corresponding conditional accuracies online.

These alternatives modify the self-confidence signal supplied to the deferral layer, but not the central learning problem studied in this work. In all cases, the system must learn from delayed and selectively observed feedback which score levels are sufficiently reliable to accept, while simultaneously accounting for the state-dependent cost of expert intervention.

\subsection{Probabilistic calibration}\label{app:related-calibration}

Calibration is the property that characterizes the relationship between a model's confidence scores and the probability that its predictions are correct: a confidence score is calibrated if predictions assigned confidence \(p\) are correct with probability \(p\).
Classical post-hoc calibration methods include histogram binning \citep{645530.655658}, Platt scaling \citep{platt1999probabilistic}, isotonic regression \citep{robertson1988order,zadrozny2002transforming}.
In our problem, the conditional accuracies are learned from delayed, selectively observed human feedback, and learning is evaluated through deferral regret rather than calibration error.

\subsection{Additional learning-to-defer formulations}\label{app:related-deferral-extensions}

\citet{pmlr-v162-charusaie22a} study learning predictors that complement humans and develop active learning schemes to reduce the number of expert predictions required during training. In our model, expert feedback is instead acquired during deployment to learn the conditional accuracies of a fixed LLM.

\paragraph{Learning to defer with multiple experts.}
Building on the single-expert setting, learning to defer has been generalized to multiple heterogeneous experts whose competencies may vary across inputs. In this setting, the learner must determine not only whether to defer, but also which expert should handle each instance. Early multi-expert formulations jointly learned a predictive model and an allocation rule while accounting for heterogeneous expert competencies and, in some cases, expert biases and fairness considerations~\citep{keswani2021towards,ijcai2022p344}. Subsequent work developed more principled formulations for expert selection, including calibrated and consistent surrogate objectives~\citep{verma2023learning}. A complementary line of work has focused on the theoretical foundations of multi-expert deferral. \citet{mao2023two} study the two-stage setting, in which a predictor is fixed and a deferral function is learned to select among experts, while \citet{NEURIPS2024_86a11cc0} develop surrogate losses with realizable $H$-consistency, $H$-consistency, and Bayes-consistency guarantees. More recently, \citet{pmlr-v267-mao25c} provide a unified treatment of single-stage and two-stage multi-expert deferral and establish stronger consistency guarantees for both settings. Other extensions consider learning with limited expert annotations~\citep{hemmer2023learning}, adaptation to previously unseen experts~\citep{pmlr-v238-tailor24a}, regression~\citep{mao2024regression}, imbalanced expert data~\citep{Cortes2026}, and Top-$k$ deferral, where multiple experts may be selected for the same instance~\citep{ICLR2026_6dc129f0}.

\subsection{LLM routing}\label{app:related-routing}

FrugalGPT constructs LLM cascades to reduce monetary cost while maintaining predictive quality \citep{chen2024frugalgpt}. RouteLLM learns routers between stronger and weaker models using preference data \citep{ong2025routellm}.
\citet{ICLR2026_113ae3a9} introduce UniRoute, a model-routing framework that represents LLMs through their prediction-error profiles on representative prompts, thus enabling cost-aware routing to previously unseen models without retraining the router.
These methods allocate computation among automated models. We first sample from a fixed LLM, then decide whether to defer to a capacity-constrained human expert whose completed reviews supply ground-truth labels.

\subsection{Structured admission control}
\label{app:structured-admission-control}
Queue-aware learning has also been studied through structured admission-control policies.
In particular, UCRL-AC \citep{pmlr-v235-weber24a} exploits the structure of multi-class admission control to obtain sharper worst-case regret guarantees than generic MDP approaches, while SALMUT \citep{roy2021online} assumes that the optimal policy has an ordered multi-threshold structure and directly learns the corresponding vector of class-specific thresholds using stochastic approximation.
The learning problems differ substantially from ours. In \citet{pmlr-v235-weber24a},
arrival rates are unknown and are learned passively from exogenously observed
arrivals, while the class-dependent reward model is known. \cite{roy2021online} allows
the transition dynamics to be unknown, but learns from
state, reward, and transition observations generated by interaction with
the system. In our model, by contrast, the confidence level-dependent cost of accepting the LLM answer is itself unknown since it depends on the unknown conditional accuracy vector, and it is learned only through queries deferred to the expert. Thus, obtaining information about prediction quality requires consuming expert capacity and increasing queue congestion, with the resulting
feedback additionally delayed by expert service.

\subsection{Active learning and oracle queries}\label{app:related-active-learning}

Active learning studies how a learner can selectively acquire costly labels while maintaining predictive performance
\citep{10.1016/j.tcs.2010.12.054,10.1561/2200000037}.
Online variants require these query decisions to be made sequentially under partial feedback \citep{10.1145/3201604}. \citet{reveillard2025minimizing} study online classification in which
costly expert consultations provide labeled examples used to improve
subsequent predictions. In our model, the LLM remains fixed, and expert
feedback instead estimates conditional accuracies
to guide deferral. \citet{russo2024online} study a limited budget of best-action queries, which reveal the identity of the best action before the current decision. This information differs from our completed-review labels, which reveal whether the LLM response was correct without revealing the optimal deferral decision.

\subsection{Human--AI decision support}\label{app:related-human-ai}

Related work studies how the relationship between human and AI confidence affects the usefulness of AI advice \citep{NEURIPS2023_2f1d1196,CorveloBenz2025}.
\citet{benz2026learning} study learning decision policies through repeated interactions under human-alignment and full feedback.
In our model, an accepted LLM response provides no correctness feedback, and expert review is limited by queue capacity.

\clearpage
\section{Datasets}\label{app:datasets}
\subsection{Synthetic datasets}\label{app:synthetic-datasets}

In this section, we describe the two synthetic instances used in Appendix~\ref{app:synthetic-experiments}. Both have $K=10$ self-confidence levels and are shown in Figure~\ref{fig:synthetic-instances}.

\begin{figure*}[htb!]
    \centering
    \includegraphics[width=\textwidth]{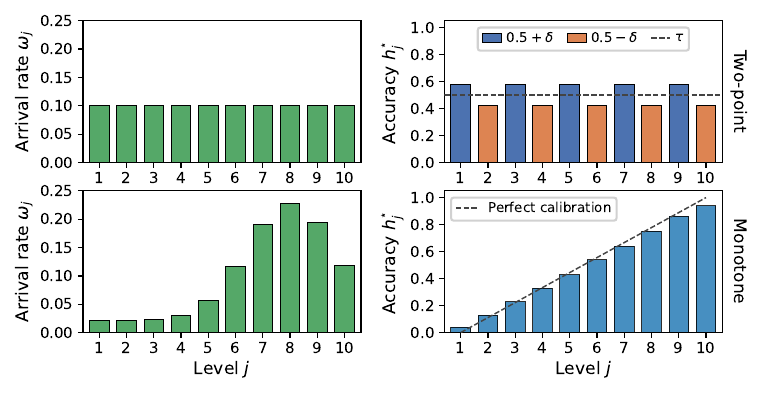}
    \caption{Top: two-point instance with $\delta=0.08$. Bottom: monotone instance (dashed line: perfect calibration, i.e., $h^\star_j=p_j$ with $p_j=(j-1)/9$).}
    \label{fig:synthetic-instances}
\end{figure*}

\paragraph{Two-point instance.}
The first instance uses the two-point family $h^\star_j\in\{0.5-\delta,0.5+\delta\}$, $j\in[K]$, with uniform arrival rates $\omega_j=1/K$ (top panel of Figure~\ref{fig:synthetic-instances}). This construction is motivated by the hard instances used in the proof of our minimax lower bound (Theorem~\ref{thm:known-M-lower}): under immediate service ($\mu=1$), uniform arrivals and $c_{\rm exp}/c_e=1/2$, the oracle threshold is $\tau=1/2$, and if $K \le T$, these instances take $h^\star_j\in\{\tau-\delta,\tau+\delta\}$ with $\delta=\frac18\sqrt{K/T}$. Appendix~\ref{app:exp-immediate-service} uses this instance in that setting only, with $\delta$ as a free parameter.

\paragraph{Monotone instance.}
The second instance combines increasing conditional accuracies $h^\star=(.04,.13,.23,.33,.43,.54,.64,.75,.86,.94)$ with arrival rates concentrated on the high levels (bottom panel of Figure~\ref{fig:synthetic-instances}). Since the scores are $p_j=(j-1)/9$, they are nearly calibrated: this instance models an LLM that is slightly underconfident at the three lowest levels and overconfident at the others, by at most $0.06$. Since the conditional accuracies span most of $[0,1]$, the oracle threshold falls between the accuracies of two levels for a wide range of service rates, costs and queue sizes. This ensures the optimal actions do not become trivial to learn when we vary these parameters, unless deferral is never worthwhile ($c_{\rm exp}\ge c_e$). We use this instance in Appendix~\ref{app:exp-service-cost}, where we vary the service rate $\mu$, the expert cost $c_{\rm exp}$ and the holding cost $c_{\rm hold}$, with the conditional accuracies held fixed.

\subsection{Real-world datasets}
\label{app:real-world-datasets}

In this section, we describe the datasets used in our real-world experiments. We start from three question-answering datasets with verifiable answers:
\begin{itemize}[leftmargin=*,itemsep=1pt,topsep=2pt,parsep=0pt]
    \item GSM8K \citep{cobbe2021gsm8k} contains grade-school math problems with numerical answers. We use all $8{,}792$ questions from its train and test splits.
    \item  NuminaMath-TIR \citep{numina2024} contains competition-style mathematics problems. We use its first $10{,}000$ training questions, from which we remove $157$ exact duplicates and $659$ questions that also appear in GSM8K, leaving $9{,}184$.
    \item  MedMCQA is a large-scale medical dataset derived primarily from medical entrance and professional examinations \citep{pal2022medmcqa}. It contains multiple-choice questions with four options. We use all $187{,}005$ questions of its train and validation splits.
\end{itemize} 
We use Qwen3-4B-Instruct-2507, Qwen3-0.6B, and Qwen2.5-0.5B-Instruct \citep{qwen3,qwen25}, and sample $20$ answers per question. Answers are sampled with vLLM at temperature $0.6$ and top-$p$ $0.95$, without top-$k$ truncation, with up to $2{,}048$ generated tokens per answer and a maximum model sequence length (prompt plus output) of $4{,}096$ tokens.
\paragraph{Answer extraction and grading.}
A final answer is extracted from each generation by taking the last complete \texttt{<answer>} block. If none is present, we use a terminal \texttt{\textbackslash boxed\{\}} expression or an explicit final-answer declaration. Otherwise, no answer is returned. Answers are grouped into equivalence classes, either by mathematical equivalence with Math-Verify \citep{mathverify} for the math datasets, which handles fractions, radicals, tuples, sets and symbolic expressions, or by option letter for MedMCQA. The same equivalence decides whether $R_t\equiv R^\star_t$. Ties between the largest classes are broken by first occurrence. Generations without an extractable answer count in the denominator of the score but cast no vote, so a question on which no generation yields an answer has score $0$ and $Y_t=0$. Questions whose ground-truth answer could not be parsed, or whose grading timed out, are removed: $58$ NuminaMath questions for the Qwen3 models, and $85$ NuminaMath and $12$ GSM8K questions for Qwen2.5-0.5B.

\paragraph{Instances.}
Figure~\ref{fig:real-data-distributions-new} shows $h^\star$ and $\omega$ for all (dataset, LLM) pairs, from which the instances are built as in Section~\ref{sec:experiments}. The regimes of Figure~\ref{fig:accuracy-2by3} extend to Qwen2.5-0.5B: its conditional accuracy is above the score on the math datasets, especially on GSM8K, and varies little with the score on MedMCQA. Across all three datasets, the self-consistency score is generally not calibrated, i.e. $h^\star \ne p$. Moreover, $h^\star$ generally is not monotone in the self-consistency level $j$. This motivates learning $h^\star$ rather than assuming the score is calibrated.

\begin{figure*}[htb!]
    \centering
    \includegraphics[width=\linewidth]{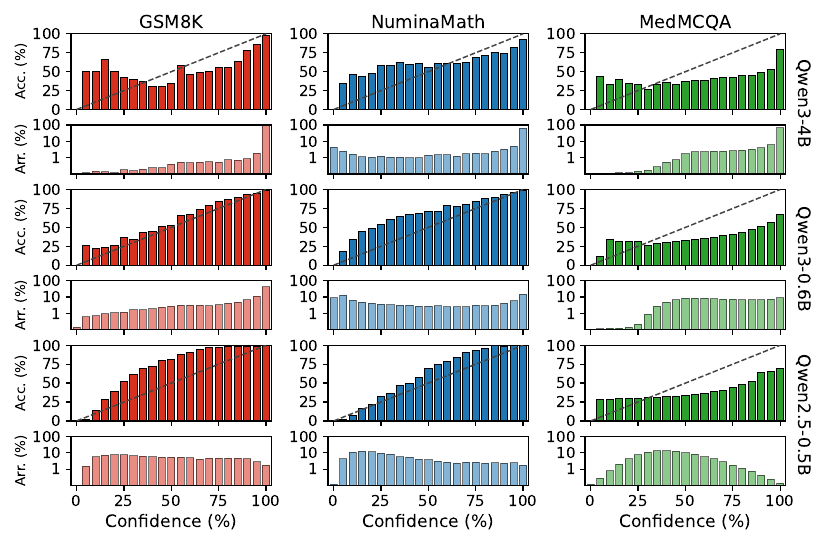}
    \caption{Conditional accuracy (top) and arrival rates (bottom) of the self-consistency score. The dashed line represents a perfectly calibrated score.}
    \label{fig:real-data-distributions-new}
\end{figure*}

\clearpage
\section{Baselines and implementation details}
\label{app:policy-implementations}

The costs $(c_e,c_{\rm exp},c_{\rm hold})$, the service distribution $\rho_C$ and the queue capacity $\bar B$ are known to all algorithms. In the synthetic experiments, $\omega$ is known. In the real-world experiments it is unknown, and every learning algorithm that needs it uses the empirical frequencies $\hat\omega_{j,t}=O_{j,t}/t$, where $O_{j,t}=\sum_{s\le t}\mathbf 1\{J_s=j\}$ counts the arrivals at self-confidence level $j$ up to round $t$. Let $N_{j,t}$ and $C_{j,t}$ be the number of level-$j$ labels received before the decision of round $t$ and the number of correct ones among them, and $\widehat h_{j,N_{j,t}}=(C_{j,t}/N_{j,t})\mathbf 1\{N_{j,t}>0\}$ their empirical mean.

\paragraph{Plug-in algorithms (\texttt{KL-PBD}, \texttt{KL-PBD+}, \texttt{GBD}, \texttt{NBD}).}
These four algorithms share the decision rule of \texttt{KL-PBD}, $A_t=\mathbf 1\{c_e(1-h_{J_t,t})>M^{h_t}_t(B_t)\}$, and differ only in the conditional accuracy vector $h_t$ plugged into the marginal cost of deferral $M^{h_t}_t$:
\begin{itemize}
\item \texttt{KL-PBD}: the KL lower confidence bound $\underline h_{j,t}$ with confidence budget $\phi(T)=\log T+\log\log T$;
\item \texttt{KL-PBD+}: the same bound with the level-dependent budget $\log(O_{j,t}/N_{j,t})$ in place of $\phi(T)$, in the spirit of the KL-UCB+ family of algorithms \citep{garivierkl2011,achab2018profitable};
\item \texttt{GBD} (Greedy Bellman Deferral): the empirical mean $\widehat h_{j,N_{j,t}}$;
\item \texttt{NBD} (Naive Bellman Deferral): the surrogate $h=p$, i.e.\ the score is taken as calibrated and nothing is learned about $h^\star$.
\end{itemize}
With known $\omega$, the finite-horizon recursion is re-solved whenever a new label arrives (never for \texttt{NBD}). With unknown $\omega$, it is re-solved at every round with the current $(h_t,\hat\omega_t)$, for \texttt{NBD} included. The marginal costs are computed exactly by backward induction over the remaining horizon, with the terminal draining cost $D(b)$ of Section~\ref{sec:problem-formulation}.

\paragraph{Fixed-threshold policies (\texttt{FCT}, \texttt{FQT}).}
Fixed Confidence Threshold defers iff $p_{J_t}<\tau^{\rm conf}$; Fixed Queue Threshold defers iff $B_t<\tau^{\rm queue}_{J_t}$ for a threshold vector $\tau^{\rm queue}\in\{0,\ldots,\bar B\}^K$. Both are tuned once before the run to minimize the expected cost under the surrogate $h=p$, computed exactly by dynamic programming. When $\omega$ is unknown the uniform distribution $\omega_j=1/K$ is used in place of $\omega$. Since only the ordering of $\{p_1,\ldots,p_K\}$ matters for \texttt{FCT}, its $K+1$ candidate policies are enumerated exhaustively. For \texttt{FQT}, exhaustive search over $\{0,\ldots,\bar B\}^K$ is infeasible, so we use coordinate descent: each threshold $\tau_j^{\rm queue}$ is optimized over $\{0,\ldots,\bar B\}$ with the others fixed, cycling over $j=1,\ldots,K$ until the vector no longer changes. We run this from several deterministic initializations and keep the best result, which need not be the global optimum.

\paragraph{\texttt{SALMUT-SF}.}
We adapt \texttt{SALMUT} \citep{roy2021online}, which defers a level-$j$ query iff $B_t<\tau_j$, with one queue threshold $\tau_j\in[0,\bar B]$ per level, initialized at $\tau_j=0$. In the setting of \citet{roy2021online}, $h$ would be known. In our case, we estimate it with $\widehat h_{j,N_{j,t}}$.
Let $b=B_t$, $j=J_t$, $a=A_t$ and $b'=B_{t+1}$. At the end of round $t$, after service and label processing, the relative-value critic, initialized at $V(b)=0$, is updated at the pre-decision queue state $b$, including when $b=\bar B$, with step size $\alpha_{n_b}=(\lfloor n_b/100\rfloor+2)^{-0.6}$ where $n_b$ counts the updates at state $b$:
\[
V(b)\leftarrow V(b)+\alpha_{n_b}\bigl[c_{\rm hold}b+c_{\rm exp}a+c_e(1-\widehat h_{j,N_{j,t+1}})(1-a)+V(b')-V(0)-V(b)\bigr].
\]
For the actor, we retain \texttt{SALMUT}'s sigmoid interpolation $f(b,\tau)=1/(1+e^{-(b-\tau-1/2)})$. Note that unlike \cite{roy2021online}, the transition probabilities are assumed known in our setting, so they do not need to be estimated. Therefore, after the critic update, for $b<\bar{B}$, we estimate the advantage function via the updated critic,
$\hat Q_1-\hat Q_0=c_{\rm exp}-c_e(1-\widehat h_{j,N_{j,t+1}})+(P_1V)(b)-(P_0V)(b)$, where $(P_aV)(b)=\sum_{b'}P_a(b,b')V(b')$, and for $b<\bar B$ we update the threshold of the arriving level as
\[
\tau_j\leftarrow\Pi_{[0,\bar B]}\Bigl[\tau_j-\tfrac{10}{t+100}\,f(b,\tau_j)\bigl(1-f(b,\tau_j)\bigr)(\hat Q_1-\hat Q_0)\Bigr].
\]
 We do not use SALMUT's ordering constraint on the thresholds, since conditional accuracies need not be ordered in the score. When the threshold rule accepts, we force a feasible deferral with probability $\varepsilon_t=\min\{.5,\,2/\sqrt t\}$; without it a level can stop producing the labels needed to revise an optimistic threshold due to the selective nature of the feedback in our problem (this failure mode is explored in Appendix \ref{app:gbd-pessimism}). The same hyperparameters are used in every instance.

\paragraph{Queue-unaware bandit algorithms (\texttt{KL-UCB-4P}, \texttt{KL-UCB+-4P}).}
Under immediate service ($\mu=1$), the queue stays empty and $M_t(0)=c_{\rm exp}$. The problem then reduces to a separate learning task for each self-confidence level: at each round the level $J_t=j$ is revealed and the learner faces a two-armed bandit with one known arm. Indeed, one can equivalently subtract the cost $c_e(1-Y_t)$ of accepting from both actions, so that accepting has reward $0$, while deferring has reward $c_e(1-Y_t)-c_{\rm exp}$, with unknown mean $c_e(1-h^\star_j)-c_{\rm exp}$, observed only when the query is deferred. This is close to the profitable-bandit problem of \citet{achab2018profitable}, played independently at each level, which motivates a comparison with their algorithms. We adapt their algorithms to our setting, so that they defer 
iff $c_e\bar q_{j,t}\ge c_{\rm exp}$, ignoring the queue except when it is full, where $\bar q_{j,t}$ is the KL upper confidence bound on the LLM error probability $1-h^\star_j$, computed from $1-\widehat h_{j,N_{j,t}}$ with budget $\log O_{j,t}+3\log\log O_{j,t}$ for \texttt{KL-UCB-4P} and $\log(O_{j,t}/N_{j,t})$ for \texttt{KL-UCB+-4P}, using the level clock $O_{j,t}$ in place of $t$, with unlabeled levels having $\bar q_{j,t}=1$ so that they are deferred. Equivalently, up to ties, they follow the decision rule of \texttt{KL-PBD} with $M^{h_t}_t(B_t)$ replaced by $c_{\rm exp}$, its value under immediate service: \texttt{KL-UCB+-4P} is then \texttt{KL-PBD+}, and \texttt{KL-UCB-4P} is \texttt{KL-PBD} with budget $\log O_{j,t}+3\log\log O_{j,t}$ in place of $\phi(T)$. When $\mu<1$, labels arrive only once a deferred query is served, and the algorithms use the labels available at each round.

\paragraph{Simulation protocol.}
For every run, all algorithms see the same arrivals $(J_t)$, latent labels $(Y_t)$ and service capacities $(C_t)$, and keep separate first-in, first-out (FIFO) queues. A label completed after the decision of round $t$ is available from round $t+1$. We compute the regret as $\hat R_T=\sum_{t\le T,\,B_t<\bar B}|L^\star(J_t)-M^\star_t(B_t)|\mathbf 1\{A_t\neq A^\star_t\}$ (Lemma~\ref{lem:regret-decomposition}), averaged over independent runs with pointwise $95\%$ normal confidence intervals. The code used for our experiments is provided in the supplementary files.

\clearpage
\section{Synthetic experiments}\label{app:synthetic-experiments}

\subsection{Effect of $K$ and of the oracle gap on \texttt{KL-PBD} regret}\label{app:exp-immediate-service}

In this experiment, we examine how the number of levels, the separation between the accuracies and the oracle threshold, and the arrival distribution $\omega$ affect \texttt{KL-PBD}'s regret. We use the two-point instance of Appendix~\ref{app:synthetic-datasets} (top row of Figure~\ref{fig:synthetic-instances}) with $\mu=1$, $c_e=1$ and $c_{\rm exp}=0.5$, so that the oracle threshold is $\tau=1/2$: the oracle defers the levels $j$ with $h^\star_j=0.5-\delta$ and accepts those with $h^\star_j=0.5+\delta$. All experiments use $T=2\cdot10^5$ and $100$ runs.

\begin{figure*}[htb!]
    \centering
    \includegraphics[width=\textwidth]{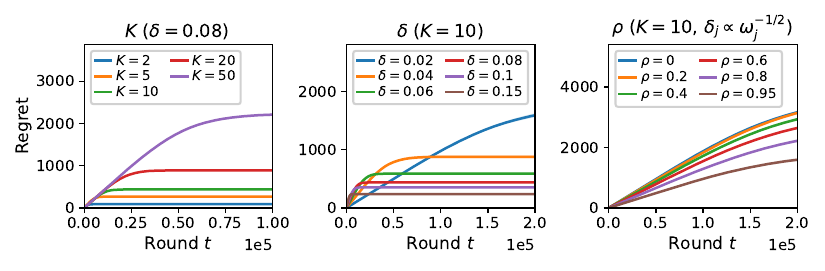}
    \caption{Regret of \texttt{KL-PBD} on the two-point instance for different $K$ with $\delta=0.08$ (left, shown up to $t=10^5$) and different $\delta$ with $K=10$ (center), and for different arrival distributions with $K=10$ and $h^\star_j=0.5+\sqrt{\phi(T)/(2\omega_jT)}$ (right). $95\%$ confidence intervals over 100 runs are displayed as bands but are too small to be visible.}
    \label{fig:KanddeltavsT}
\end{figure*}

Figure~\ref{fig:KanddeltavsT} shows that regret increases with $K$, since the learner must acquire more expert feedback to identify the correct action in a larger number of levels. The dependence on $\delta$ reflects the usual gap-dependent learning trade-off. Larger gaps make an incorrect deferral more costly initially, but they are statistically easier to distinguish from the oracle threshold, so the learner identifies the optimal action sooner and the regret curves flatten earlier. Conversely, when $\delta$ is small, the conditional accuracies lie closer to the threshold and require substantially more feedback to distinguish, leading to larger regret over the range of the figure.

Under immediate service, \texttt{KL-PBD} defers level $j$ until $N_{j,t}\,\mathrm{kl}(\widehat h_{j,N_{j,t}},1/2)\ge\phi(T)$. A level with $h^\star_j=0.5+\delta$ therefore needs about $\phi(T)/(2\delta^2)$ labels before it is accepted, each deferral costing $\delta$, while a level with $h^\star_j=0.5-\delta$ is deferred, as by the oracle. Each level above the threshold thus contributes about $\min\{\delta\,\omega_jT,\ \phi(T)/(2\delta)\}$ to the regret.

Table~\ref{tab:k_delta} complements Figure~\ref{fig:KanddeltavsT} by reporting the final regret at $T=2\cdot10^5$. The results show that regret increases significantly with $K$, confirming the additional learning burden associated with a larger number of levels. At $\delta=0.08$, every oracle decision is learned well before $T$, and each level above the threshold adds about $\phi(T)/(2\delta)\approx90$ to the regret. For every value of $K$ except $K=50$, regret decreases as $\delta$ increases; we explain the exception below.

\begin{table}[htb!]
\centering
\small
\begin{tabular}{|c|c|c|c|c|c|c|}
\hline
$K$ & $\delta=0.02$ & $\delta=0.04$ & $\delta=0.06$ & $\delta=0.08$ & $\delta=0.10$ & $\delta=0.15$ \\
\hline
2 & $334 \pm 23$ & $179 \pm 13$ & $121 \pm 9$ & $87 \pm 7$ & $71 \pm 5$ & $50 \pm 3$ \\
\hline
5 & $1085 \pm 51$ & $555 \pm 21$ & $357 \pm 13$ & $267 \pm 12$ & $209 \pm 9$ & $138 \pm 6$ \\
\hline
10 & $1592 \pm 36$ & $878 \pm 28$ & $589 \pm 20$ & $439 \pm 13$ & $353 \pm 12$ & $235 \pm 7$ \\
\hline
20 & $1948 \pm 11$ & $1757 \pm 44$ & $1181 \pm 29$ & $889 \pm 17$ & $706 \pm 15$ & $462 \pm 11$ \\
\hline
50 & $1997 \pm 2$ & $3564 \pm 30$ & $2966 \pm 42$ & $2225 \pm 34$ & $1789 \pm 27$ & $1168 \pm 14$ \\
\hline
\end{tabular}
\caption{Final regret $R_T$ of \texttt{KL-PBD} at $T=2\cdot10^5$ for different $K$ and $\delta$, mean $\pm$ 95\% confidence interval half-width over 100 runs.}
\label{tab:k_delta}
\end{table}

\begin{table}[htb!]
\centering
\small
\setlength{\tabcolsep}{4pt}
\begin{tabular}{|c|c|c|c|c|c|c|}
\hline
$\delta$ & 0.0025 & 0.005 & 0.01 & 0.015 & $0.0192^{\dagger}$ & 0.025 \\
\hline
$R_T$ & $250 \pm 0$ & $500 \pm 0$ & $998 \pm 3$ & $1442 \pm 16$ & $1596 \pm 32$ & $1396 \pm 43$ \\
\hline
\hline
$\delta$ & 0.03 & 0.04 & 0.06 & 0.08 & 0.10 & 0.15 \\
\hline
$R_T$ & $1189 \pm 41$ & $878 \pm 28$ & $589 \pm 20$ & $439 \pm 13$ & $353 \pm 12$ & $235 \pm 7$ \\
\hline
\end{tabular}
\caption{Final regret $R_T$ of \texttt{KL-PBD} at $T=2\cdot10^5$ for $K=10$ and different $\delta$, mean $\pm$ 95\% confidence interval half-width over 100 runs. $\dagger$: $\delta^\star=\sqrt{\phi(T)K/(2T)}\approx0.019$.}
\label{tab:delta_sweep2}
\end{table}

Table~\ref{tab:delta_sweep2} further illustrates the non-monotone dependence of finite-horizon regret on the oracle gap $\delta$. For very small $\delta$, distinguishing which side of the oracle threshold the accuracy is requires many observations, but the regret incurred by an incorrect decision is itself only $O(\delta)$, so the regret is at most of order $\delta T$. In this case, \texttt{KL-PBD} does not have enough time to learn the oracle decisions, and $R_T\approx\delta T/2$. For larger $\delta$, the levels are statistically easier to distinguish and the regret decreases roughly as $K\phi(T)/(4\delta)$. Consequently, the largest regret is observed near the crossover between these two regimes, $\delta^\star=\sqrt{\phi(T)K/(2T)}\approx0.019$ for $K=10$. Since $\delta^\star$ grows with $K$ ($\delta^\star\approx0.043$ for $K=50$), the regret for $K=50$ is lower at $\delta=0.02$ than at $\delta=0.04$ (Table~\ref{tab:k_delta}).

The right panel examines the effect of non-uniform arrivals with $K=10$. We parameterize
\[
\omega_1=\frac{1}{K}+\rho\left(1-\frac{1}{K}\right),
\qquad
\omega_j=\frac{1-\rho}{K},\quad j=2,\ldots,K,
\]
so that $\rho=0$ corresponds to uniform arrivals and increasing $\rho$ concentrates probability mass on a single level. We set $h^\star_j=0.5+\delta_j$ with $\delta_j=\sqrt{\phi(T)/(2\omega_jT)}$ for every level. Unlike in the two-point instance, all levels lie above the threshold: a level below it is deferred by both \texttt{KL-PBD} and the oracle and adds no regret, so this makes every level contribute. Level $j$ receives about $\omega_jT$ arrivals and needs about $\phi(T)/(2\delta_j^2)=\omega_jT$ labels before it is accepted, so with this gap \texttt{KL-PBD} starts accepting level $j$ at about the horizon. For uniform arrivals, $\delta_j=\delta^\star$.
Each level then contributes about $\delta_j\,\omega_jT=\sqrt{\phi(T)\,\omega_jT/2}$ to the regret. Regret is largest under the uniform distribution and decreases as the arrival mass becomes more concentrated. This ordering is consistent with the dependence on $\omega$ in Theorem~\ref{thm:pbd-regret}, whose leading term is $\widetilde O\bigl(\sqrt{T}\sum_{j=1}^K\sqrt{\omega_j}\bigr)$ and is maximized when $\omega_j=1/K$ for $j \in[K]$. Quantitatively, $R_T/(\sqrt T\sum_{j=1}^K\sqrt{\omega_j})$ stays between $2.21$ and $2.29$ for all $\rho$.

\subsection{Effect of service rate and cost parameters}\label{app:exp-service-cost}

In this experiment we use the monotone instance (bottom row of Figure~\ref{fig:synthetic-instances}) in order to evaluate how \texttt{KL-PBD} is affected by the expert service rate. We set the queue capacity to $\bar{B}=20$, the LLM error cost to $c_e=2$, the expert-review cost to $c_{\rm exp}=0.25$, and the holding cost to $c_{\rm hold}=0.1$. We vary the expert service rate over $\mu\in\{0.2,0.4,0.5,0.6,0.8,1.0\}$, where in each round the expert completes a pending query with probability $\mu$.

For each service rate, we run the system for $T=50{,}000$ rounds and average results over 20 independent runs. We report the average queue size, the cumulative regret relative to the oracle, the response accuracy, i.e., the fraction of queries answered correctly, counting deferred queries as correct, and the average threshold, i.e., the time average of $\widehat\tau_t(B_t)$.

\begin{table}[h]
\centering
\caption{Finite-horizon KL-PBD performance across service rates $\mu$ with $\bar B=20$. Results are mean $\pm$ 95\% confidence interval half-width over 20 runs.} \label{tab:KL-PBD-mu-sweep-finite}
\small
\begin{tabular}{|c|c|c|c|c|}
\hline
Service rate $\mu$ & Average queue size & Cumulative regret $R_T$ & Response accuracy & Average threshold \\
\hline
0.20 & $0.74 \pm 0.01$ & $194.06 \pm 14.96$ & $0.7710 \pm 0.0010$ & $0.2663 \pm 0.0024$ \\
\hline
0.40 & $0.70 \pm 0.01$ & $233.59 \pm 30.80$ & $0.8445 \pm 0.0008$ & $0.5130 \pm 0.0030$ \\
\hline
0.50 & $0.74 \pm 0.00$ & $115.49 \pm 7.42$ & $0.8856 \pm 0.0005$ & $0.5716 \pm 0.0013$ \\
\hline
0.60 & $0.55 \pm 0.01$ & $135.36 \pm 11.75$ & $0.9070 \pm 0.0006$ & $0.6579 \pm 0.0018$ \\
\hline
0.80 & $0.37 \pm 0.00$ & $301.17 \pm 19.26$ & $0.9603 \pm 0.0011$ & $0.7632 \pm 0.0010$ \\
\hline
1.00 & $0.00 \pm 0.00$ & $80.61 \pm 10.91$ & $0.9937 \pm 0.0001$ & $0.8750 \pm 0.0000$ \\
\hline
\end{tabular}
\end{table}

Table~\ref{tab:KL-PBD-mu-sweep-finite} shows that increasing the expert service rate generally makes deferral more attractive. As \(\mu\) increases, deferred queries are processed more quickly, reducing the congestion cost associated with sending additional jobs to the expert. Consistent with this behavior, the average threshold increases from \(0.2663\) at \(\mu=0.2\) to \(0.8750=1-c_{\rm exp}/c_e\) at \(\mu=1.0\), while response accuracy improves substantially from \(0.7710\) to \(0.9937\). The average queue size remains below one across all service rates, falls from \(0.74\) at \(\mu=0.5\) to \(0.37\) at \(\mu=0.8\), and reaches zero at \(\mu=1.0\), where deferred jobs are cleared immediately.

Cumulative regret remains non-monotone in \(\mu\), ranging from \(80.61\) to \(301.17\) across the tested service rates. In particular, regret is relatively high at \(\mu=0.4\) and reaches its maximum at \(\mu=0.8\), despite the corresponding improvements in response accuracy and queue length. This reflects the fact that each service rate induces a different control problem and oracle policy. At \(\mu=1.0\), the experiment attains both its highest response accuracy, \(0.9937\), and its lowest cumulative regret, \(80.61\).
\begin{figure}[htbp]
    \centering
    \includegraphics[width=\textwidth]{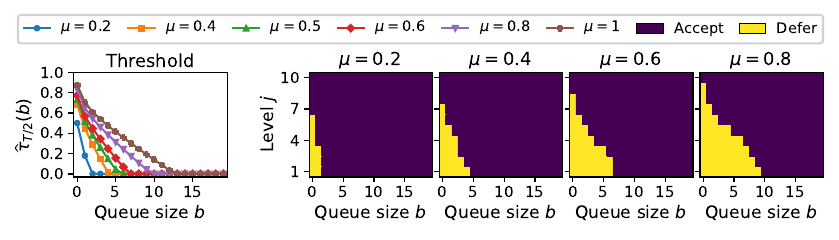}
    \caption{\texttt{KL-PBD} threshold $\widehat\tau_{T/2}(b)$ as a function of queue size $b$, averaged over the runs (left), and learned \texttt{KL-PBD} decisions, most frequent over the runs (right), at $t=T/2$ for different service rates $\mu$.}
    \label{fig:KL-PBD-policy-mu}\label{fig:KL-PBD-threshold-mu}
\end{figure}

Figure~\ref{fig:KL-PBD-policy-mu} (right) shows the policy learned by \texttt{KL-PBD}  at \(t=T/2\). It exhibits a clear threshold structure: within each level, the action changes from defer to accept once the queue size exceeds a level-dependent boundary, producing the staircase-shaped frontier visible across all panels. As the expert service rate \(\mu\) increases, the defer region expands toward larger queue sizes, since a faster expert clears deferred requests more quickly.
    
Figure~\ref{fig:KL-PBD-threshold-mu} (left) shows how the \texttt{KL-PBD} threshold $\widehat\tau_t(b)$ varies with the current queue size for different expert service rates. For every value of $\mu$, the threshold decreases as the queue grows, reflecting the increasing congestion cost of sending additional queries to the expert. Once the queue becomes sufficiently large, the threshold reaches zero, and no query is deferred. Increasing the service rate shifts the threshold curve upward and extends it over larger queue sizes. When the expert is slow, as for $\mu=0.2$, the threshold falls rapidly and reaches zero at a small queue size. In contrast, faster service allows the system to tolerate substantially more queued queries before the threshold becomes small.

We then fix $\mu=0.5$ and vary the holding cost $c_{\rm hold}\in\{0.005,0.01,0.05,0.1,0.5\}$ with $c_{\rm exp}=0.25$, and the expert-review cost $c_{\rm exp}\in\{0,0.25,1,3\}$ with $c_{\rm hold}=0.05$, over 10 independent runs of $T=50{,}000$ rounds. In these sweeps, the thresholds of \texttt{KL-PBD} are recomputed after every 50 changes of the lower bounds.

\begin{table}[htb!]
\centering
\caption{KL-PBD performance across holding and expert-review costs with $\mu=0.5$. Results are mean $\pm$ 95\% confidence interval half-width over 10 runs.}
\label{tab:klpbd-cost-ablation}
\small
\begin{tabular}{|c|c|c|c|c|c|c|}
\hline
$c_{\mathrm{hold}}$ & $c_{\mathrm{exp}}$ &
Mean queue & Defer (\%) & Accuracy (\%) &
Average threshold & Regret \\
\hline
0.005 & 0.25 & $2.840 \pm 0.057$ & $49.62 \pm 0.18$ & $91.27 \pm 0.11$ & $0.635 \pm 0.003$ & $138.20 \pm 8.81$ \\
\hline
0.010 & 0.25 & $2.106 \pm 0.059$ & $48.68 \pm 0.29$ & $90.87 \pm 0.13$ & $0.638 \pm 0.005$ & $183.55 \pm 14.20$ \\
\hline
0.050 & 0.25 & $1.042 \pm 0.027$ & $45.20 \pm 0.24$ & $89.48 \pm 0.12$ & $0.604 \pm 0.004$ & $145.86 \pm 18.98$ \\
\hline
0.100 & 0.25 & $0.751 \pm 0.007$ & $43.52 \pm 0.17$ & $88.56 \pm 0.10$ & $0.571 \pm 0.003$ & $118.58 \pm 13.78$ \\
\hline
0.500 & 0.25 & $0.229 \pm 0.005$ & $22.90 \pm 0.39$ & $81.06 \pm 0.11$ & $0.429 \pm 0.003$ & $222.55 \pm 33.59$ \\
\hline
0.050 & 0.00 & $1.221 \pm 0.010$ & $48.72 \pm 0.18$ & $89.97 \pm 0.10$ & $0.611 \pm 0.003$ & $71.57 \pm 10.97$ \\
\hline
0.050 & 1.00 & $0.264 \pm 0.011$ & $18.71 \pm 0.79$ & $80.20 \pm 0.36$ & $0.441 \pm 0.001$ & $432.32 \pm 57.30$ \\
\hline
0.050 & 3.00 & $0.000 \pm 0.000$ & $0.00 \pm 0.00$ & $67.83 \pm 0.12$ & $0.000 \pm 0.000$ & $0.00 \pm 0.00$ \\
\hline
\end{tabular}
\end{table}

Table~\ref{tab:klpbd-cost-ablation} shows that the holding cost and expert-review cost affect \texttt{KL-PBD} in distinct ways. With $c_{\mathrm{exp}}=0.25$, increasing $c_{\mathrm{hold}}$ from $0.005$ to $0.5$ reduces the mean queue from $2.84$ to $0.23$ and the deferral rate from $49.62\%$ to $22.90\%$. As fewer queries are routed to the expert, response accuracy decreases from $91.27\%$ to $81.06\%$. The average threshold decreases overall from $0.635$ to $0.429$. Regret is non-monotone because changing $c_{\mathrm{hold}}$ also changes the oracle policy.

The expert-review-cost sweep exhibits a sharper transition. With $c_{\mathrm{hold}}=0.05$, increasing $c_{\mathrm{exp}}$ from $0$ to $1$ reduces the mean queue from $1.22$ to $0.26$, the deferral rate from $48.72\%$ to $18.71\%$, and response accuracy from $89.97\%$ to $80.20\%$; the average threshold correspondingly decreases from $0.611$ to $0.441$. For $c_{\mathrm{exp}}\geq3$, \texttt{KL-PBD} does not defer, and the accuracy of $67.83\%$ is the LLM-only accuracy. This is due to the choice of cost parameters: the expected cost $c_e(1-h^\star_{J_t})\le c_e=2$ of accepting the LLM response is below the expert-review cost, so both \texttt{KL-PBD} and the oracle accept every query, yielding zero regret.


\begin{figure}[htb!]
          \includegraphics[width=\textwidth]{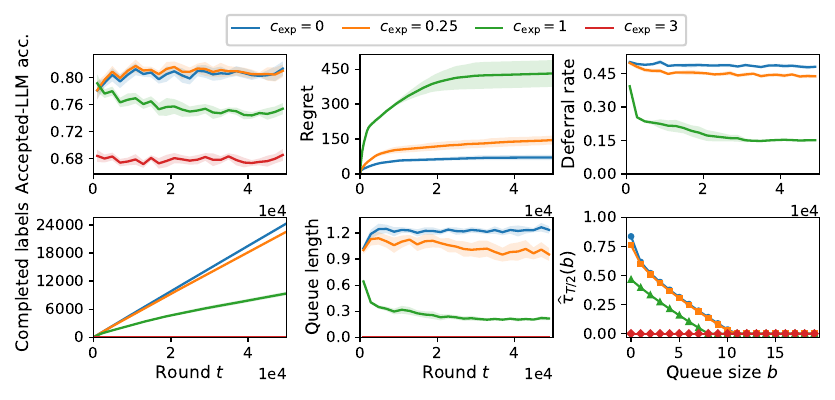}
          
    \caption{\textbf{Effect of the expert-review cost on finite-horizon \texttt{KL-PBD}.}
    We vary $c_{\mathrm{exp}}\in\{0,0.25,1,3\}$, with
    $T=50{,}000$, $\mu=0.5$, $\bar B=20$, $c_e=2$, and
    $c_{\mathrm{hold}}=0.05$.
    From the upper-left to the lower-right, the panels show accepted-LLM
    accuracy, cumulative regret, empirical deferral rate, cumulative completed
    expert labels, average queue size, and the deferral
    threshold $\widehat{\tau}_{T/2}(b)$ at $t=T/2$ as a function of queue size $b$.}
    \label{fig:expert_Queue_and_threshold_vs_T}
\end{figure}

Figure~\ref{fig:expert_Queue_and_threshold_vs_T} demonstrates how the expert-review cost
$c_{\mathrm{exp}}$ governs the operating regime of \texttt{KL-PBD}.
For $c_{\mathrm{exp}}\in\{0,0.25\}$, expert review remains sufficiently inexpensive that the policy
defers a large fraction of arriving queries, thereby maintaining a high accepted-LLM accuracy.
As $c_{\mathrm{exp}}$ increases to $1$, the policy becomes substantially more selective:
the deferral rate decreases, more queries are answered directly by the LLM, and the accuracy of the
accepted LLM responses decreases because LLM answers are also accepted at lower confidence levels. For $c_{\mathrm{exp}}\geq 3$, expert review is no
longer used, and the accepted-LLM accuracy equals the LLM accuracy.

The cumulative-regret trajectories show that the intermediate-cost regime is the most challenging. At $c_{\mathrm{exp}}=1$, expert review is still optimal in a substantial set of level--queue states, but less frequent deferrals slow the acquisition of informative expert feedback. Consequently, regret is
substantially larger than for $c_{\mathrm{exp}}\in\{0,0.25\}$.
The second row makes the underlying mechanism explicit. Increasing $c_{\mathrm{exp}}$ reduces the
number of completed expert labels, lowers the average queue size, and shifts the threshold $\widehat{\tau}_{T/2}(b)$ downward.

\begin{figure}[htb!]
          \includegraphics[width=\textwidth]{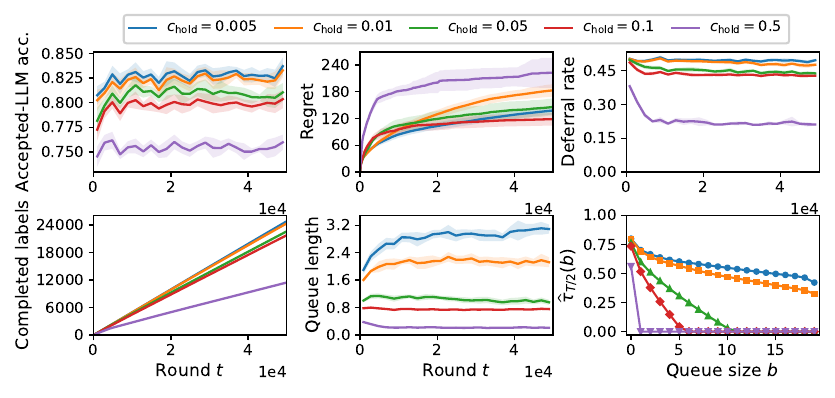}

    \caption{\textbf{Effect of the holding cost on
    \texttt{KL-PBD}.}
    We vary
    $c_{\mathrm{hold}}\in\{0.005,0.01,0.05,0.1,0.5\}$
    with $T=50{,}000$, $\mu=0.5$, $\bar B=20$, $c_e=2$, and
    $c_{\mathrm{exp}}=0.25$.
    From the upper-left to the lower-right, the panels show accepted-LLM
    accuracy, cumulative regret, empirical deferral rate, cumulative completed
    expert labels, average queue size, and the deferral
    threshold $\widehat{\tau}_{T/2}(b)$ at $t=T/2$ as a function of queue size $b$.}
    \label{fig:hold_Queue_and_threshold_vs_T}
\end{figure}

Figure~\ref{fig:hold_Queue_and_threshold_vs_T} illustrates the behavior shown in Table~\ref{tab:klpbd-cost-ablation} over time: as $c_{\mathrm{hold}}$ increases, maintaining backlog becomes more expensive, so the policy becomes increasingly selective and a larger fraction of queries is answered directly by the LLM. It also shows that the reduction in expert use limits the feedback available for learning: at low holding costs, roughly half of all queries, about $2.5\times10^4$, yield completed expert labels, compared with only about $1.1\times 10^4$ at $c_{\mathrm{hold}}=0.5$. The threshold curves provide the state-level view: for every $c_{\mathrm{hold}}$, $\widehat{\tau}_{T/2}(b)$ decreases with backlog, with the decline becoming markedly steeper as the holding cost increases. In particular, the threshold reaches zero at moderate backlog for $c_{\mathrm{hold}}=0.1$ and almost immediately for $c_{\mathrm{hold}}=0.5$, showing that higher holding costs make the deferral rule increasingly sensitive to congestion.
\begin{figure}[h!]
    \centering
      \includegraphics[width=\textwidth]{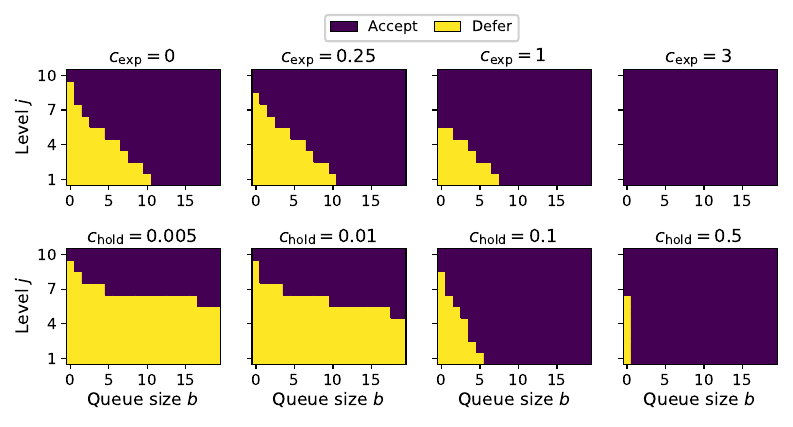}

    \caption{Learned \texttt{KL-PBD} decisions for varying $c_{\rm exp}$ and $c_{\rm hold}$, for $t=T/2$ (most frequent decision over the runs).}    \label{fig:policy-costs}
\end{figure}

Figure~\ref{fig:policy-costs} shows how the learned \texttt{KL-PBD} deferral policy changes as the expert-review and holding costs increase. In both sweeps, higher costs make deferral less attractive, so the deferral region progressively shrinks: the yellow defer region contracts while the purple accept region expands, indicating that the policy relies increasingly on the LLM and sends fewer queries to the expert.

\clearpage
\section{Additional real-world experiments and details}\label{app:real-world-details}

This section extends the experiments of Section~\ref{sec:experiments}. We first report the regret of all algorithms on the instances of Section~\ref{sec:experiments} and on additional Qwen2.5-0.5B instances, and in a separate experiment, we report how the final regret grows with the horizon. We then vary the service rate to show the benefit of queue awareness, and compare the policy learned by \texttt{KL-PBD} with the oracle policy. Finally, we show how selective feedback makes \texttt{GBD} fail, which illustrates the need for pessimism. Unless stated otherwise, the setting is that of Section~\ref{sec:experiments}. Implementation details are given in Appendix~\ref{app:policy-implementations}.

\subsection{Regret of all algorithms}
\label{app:rw-all-regret}
Table~\ref{tab:rw-final-regret} and Figure~\ref{fig:rw-all-policies} report the regret of all algorithms described in Appendix~\ref{app:policy-implementations} on all instances. \texttt{KL-PBD+} has the lowest regret in $8$ of the $9$ instances. The exception is Qwen2.5-0.5B on NuminaMath, where the score is nearly calibrated and \texttt{NBD} ($102$) is still below \texttt{KL-PBD+} ($115$) at $T=10^5$; a longer horizon is needed for \texttt{KL-PBD+} to overtake \texttt{NBD} in this instance, as we show in Table~\ref{tab:rw-horizon}. The queue-unaware algorithms \texttt{KL-UCB-4P} and \texttt{KL-UCB+-4P} incur $24$ to $450$ times the regret of \texttt{KL-PBD+} on eight instances. \texttt{KL-UCB+-4P} is only competitive on Qwen3-4B GSM8K, where the LLM accuracy averages $95\%$, which makes deferral rare and the queue almost always empty. \texttt{FCT} is the worst of the three $h=p$ surrogate algorithms on every instance.
\begin{table}[htb!]
\centering\scriptsize
\caption{Final regret at $T=10^5$, mean $\pm$ 95\% confidence half-width over 100 runs.}
\label{tab:rw-final-regret}
\resizebox{\textwidth}{!}{%
\begin{tabular}{llrrrrrrrrr}
\toprule
Model & Dataset & \texttt{KL-PBD} & \texttt{KL-PBD+} & \texttt{GBD} & \texttt{SALMUT-SF} & \texttt{KL-UCB-4P} & \texttt{KL-UCB+-4P} & \texttt{NBD} & \texttt{FCT} & \texttt{FQT} \\
\midrule
Qwen3-4B & GSM8K & 100.4 $\pm$ 3.3 & 35.4 $\pm$ 2.7 & 425.5 $\pm$ 20.5 & 531.0 $\pm$ 26.7 & 223.6 $\pm$ 5.9 & 71.3 $\pm$ 3.8 & 255.6 $\pm$ 1.3 & 756.3 $\pm$ 2.6 & 477.2 $\pm$ 1.9 \\
Qwen3-4B & NuminaMath & 424.4 $\pm$ 6.6 & 108.0 $\pm$ 4.2 & 1057.3 $\pm$ 82.0 & 1460.2 $\pm$ 41.0 & 7646.6 $\pm$ 101.8 & 4985.1 $\pm$ 156.7 & 231.9 $\pm$ 0.9 & 542.3 $\pm$ 1.2 & 195.4 $\pm$ 0.6 \\
Qwen3-4B & MedMCQA & 99.9 $\pm$ 2.9 & 40.2 $\pm$ 2.5 & 2236.7 $\pm$ 128.8 & 678.8 $\pm$ 13.8 & 7610.5 $\pm$ 191.7 & 5247.8 $\pm$ 87.7 & 2579.0 $\pm$ 4.0 & 5307.1 $\pm$ 5.4 & 4349.8 $\pm$ 5.3 \\
\midrule
Qwen3-0.6B & GSM8K & 287.9 $\pm$ 7.0 & 93.5 $\pm$ 3.5 & 1468.3 $\pm$ 110.7 & 1173.0 $\pm$ 61.8 & 4654.6 $\pm$ 60.3 & 2227.5 $\pm$ 63.9 & 146.1 $\pm$ 2.6 & 2228.8 $\pm$ 3.8 & 185.5 $\pm$ 1.0 \\
Qwen3-0.6B & NuminaMath & 451.9 $\pm$ 7.3 & 159.4 $\pm$ 5.0 & 1264.1 $\pm$ 287.7 & 3007.0 $\pm$ 68.9 & 23600.1 $\pm$ 40.1 & 21539.7 $\pm$ 57.8 & 528.7 $\pm$ 1.3 & 8414.0 $\pm$ 27.7 & 1040.3 $\pm$ 2.3 \\
Qwen3-0.6B & MedMCQA & 106.0 $\pm$ 4.5 & 75.0 $\pm$ 5.8 & 2443.2 $\pm$ 218.0 & 306.0 $\pm$ 15.4 & 21834.8 $\pm$ 18.9 & 21831.9 $\pm$ 18.8 & 1501.1 $\pm$ 3.2 & 10439.5 $\pm$ 4.1 & 4070.2 $\pm$ 4.7 \\
\midrule
Qwen2.5-0.5B & GSM8K & 314.1 $\pm$ 6.3 & 116.0 $\pm$ 3.8 & 1931.4 $\pm$ 317.9 & 1961.7 $\pm$ 43.2 & 20567.7 $\pm$ 109.6 & 17884.8 $\pm$ 107.5 & 1108.8 $\pm$ 2.4 & 3349.7 $\pm$ 20.2 & 1423.1 $\pm$ 2.4 \\
Qwen2.5-0.5B & NuminaMath & 298.5 $\pm$ 5.3 & 115.4 $\pm$ 3.6 & 1141.0 $\pm$ 296.9 & 1973.4 $\pm$ 34.0 & 23363.2 $\pm$ 34.0 & 22291.6 $\pm$ 54.5 & 102.3 $\pm$ 0.3 & 9785.0 $\pm$ 24.3 & 588.7 $\pm$ 1.7 \\
Qwen2.5-0.5B & MedMCQA & 100.6 $\pm$ 1.9 & 46.5 $\pm$ 2.2 & 2520.2 $\pm$ 168.7 & 161.6 $\pm$ 9.0 & 21041.6 $\pm$ 18.2 & 21039.9 $\pm$ 18.2 & 1365.7 $\pm$ 1.7 & 10027.7 $\pm$ 7.3 & 1003.3 $\pm$ 1.6 \\
\bottomrule
\end{tabular}}
\end{table}

\begin{figure}[htb!]
  \centering
  \includegraphics[width=\textwidth]{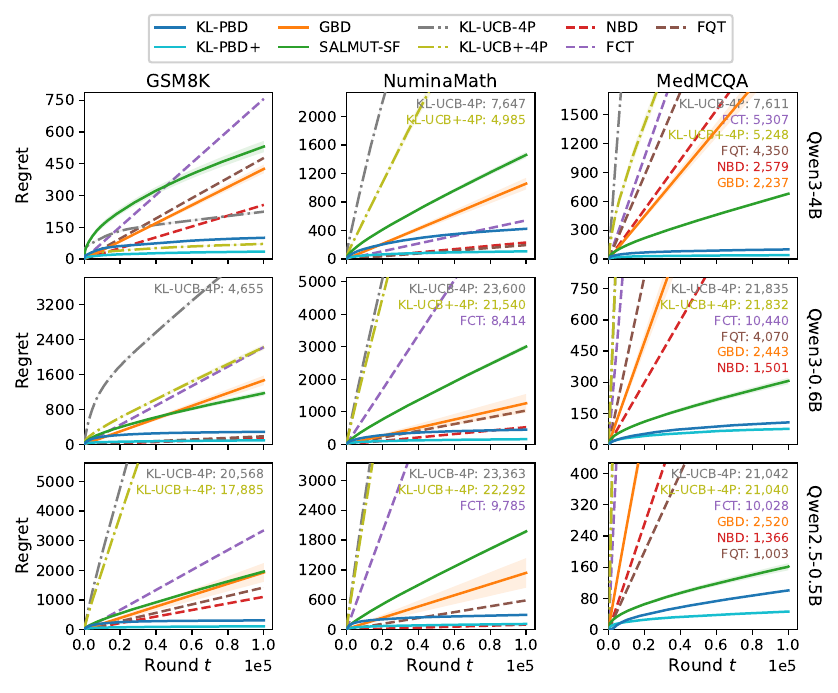}
  \caption{Regret of all algorithms, $T=10^5$, 100 runs. Truncated curves are listed with their final regret.}
  \label{fig:rw-all-policies}
\end{figure}

\subsection{Regret as a function of the horizon}
\label{app:horizon-sweep}
Figure~\ref{fig:horizon-sweep} and Table~\ref{tab:rw-horizon} report the final regret when the horizon varies in $T\in\{10^4,5\cdot10^4,10^5,2\cdot10^5\}$, averaged over 24 runs per horizon. The regret of \texttt{GBD} and of the algorithms using the $h=p$ surrogate grows linearly in $T$, by a factor of about $10$ from $T=10^4$ to $10^5$, whereas that of \texttt{KL-PBD} and \texttt{KL-PBD+} grows by a factor between $1.6$ and $4.4$. \texttt{SALMUT-SF} lies in between, with a factor between $3.4$ and $6.8$.

\begin{figure}[htb!]
  \centering
  \includegraphics[width=\textwidth]{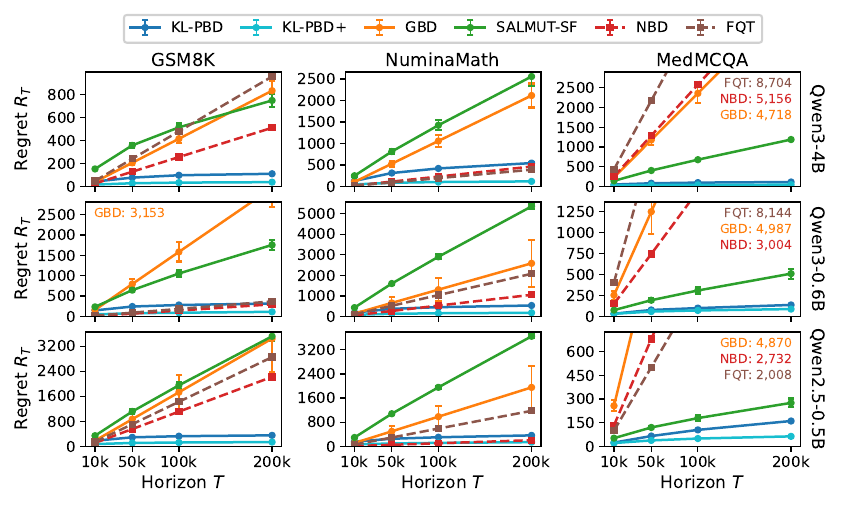}
  \caption{Final regret versus horizon $T$, 24 runs per point. Truncated curves are listed with their final regret.}
  \label{fig:horizon-sweep}
\end{figure}

\begin{table}[htb!]
\centering\scriptsize
\caption{Final regret of the horizon sweep, mean $\pm$ 95\% confidence half-width over 24 runs.}
\label{tab:rw-horizon}
\resizebox{\textwidth}{!}{%
\begin{tabular}{llrrrrrrrrrr}
\toprule
Model & Dataset & $T$ & \texttt{KL-PBD} & \texttt{KL-PBD+} & \texttt{GBD} & \texttt{SALMUT-SF} & \texttt{KL-UCB-4P} & \texttt{KL-UCB+-4P} & \texttt{NBD} & \texttt{FCT} & \texttt{FQT} \\
\midrule
Qwen3-4B & GSM8K & 10k & 43.4 $\pm$ 3.0 & 17.1 $\pm$ 1.8 & 39.9 $\pm$ 3.8 & 152.8 $\pm$ 5.5 & 97.6 $\pm$ 5.6 & 35.0 $\pm$ 3.9 & 26.4 $\pm$ 0.8 & 76.0 $\pm$ 1.6 & 48.9 $\pm$ 0.9 \\
 & & 50k & 77.5 $\pm$ 4.8 & 28.7 $\pm$ 3.4 & 206.3 $\pm$ 20.2 & 357.4 $\pm$ 20.8 & 168.8 $\pm$ 10.8 & 60.4 $\pm$ 5.0 & 128.9 $\pm$ 1.9 & 379.2 $\pm$ 3.7 & 240.8 $\pm$ 2.9 \\
 & & 100k & 98.2 $\pm$ 6.9 & 34.2 $\pm$ 4.3 & 414.6 $\pm$ 40.4 & 513.6 $\pm$ 33.7 & 221.6 $\pm$ 10.9 & 79.1 $\pm$ 6.7 & 255.7 $\pm$ 2.8 & 754.8 $\pm$ 4.1 & 478.5 $\pm$ 3.9 \\
 & & 200k & 110.8 $\pm$ 9.6 & 38.3 $\pm$ 5.2 & 831.4 $\pm$ 81.0 & 746.8 $\pm$ 52.9 & 312.0 $\pm$ 8.9 & 99.4 $\pm$ 10.0 & 510.0 $\pm$ 3.3 & 1510.5 $\pm$ 5.7 & 954.4 $\pm$ 5.2 \\
\addlinespace
Qwen3-4B & NuminaMath & 10k & 130.4 $\pm$ 4.3 & 48.8 $\pm$ 3.5 & 108.0 $\pm$ 13.6 & 250.6 $\pm$ 18.7 & 1201.9 $\pm$ 25.1 & 560.0 $\pm$ 38.4 & 22.2 $\pm$ 0.8 & 53.7 $\pm$ 0.8 & 19.3 $\pm$ 0.4 \\
 & & 50k & 314.2 $\pm$ 9.8 & 88.6 $\pm$ 7.5 & 530.3 $\pm$ 69.4 & 813.0 $\pm$ 63.7 & 4416.2 $\pm$ 167.6 & 2504.3 $\pm$ 183.6 & 115.2 $\pm$ 1.2 & 270.2 $\pm$ 1.9 & 97.0 $\pm$ 0.8 \\
 & & 100k & 420.9 $\pm$ 14.4 & 105.2 $\pm$ 9.5 & 1058.1 $\pm$ 138.5 & 1423.0 $\pm$ 112.6 & 7610.0 $\pm$ 187.2 & 4847.4 $\pm$ 311.0 & 231.9 $\pm$ 2.1 & 541.8 $\pm$ 2.6 & 195.3 $\pm$ 1.2 \\
 & & 200k & 545.1 $\pm$ 17.4 & 120.2 $\pm$ 12.2 & 2113.5 $\pm$ 278.3 & 2554.9 $\pm$ 210.0 & 13963.5 $\pm$ 193.3 & 9537.6 $\pm$ 470.1 & 465.0 $\pm$ 3.1 & 1085.5 $\pm$ 2.8 & 391.6 $\pm$ 1.8 \\
\addlinespace
Qwen3-4B & MedMCQA & 10k & 52.1 $\pm$ 3.6 & 27.8 $\pm$ 3.7 & 233.5 $\pm$ 23.9 & 144.4 $\pm$ 17.1 & 2512.0 $\pm$ 96.9 & 1000.8 $\pm$ 137.4 & 257.9 $\pm$ 2.9 & 531.2 $\pm$ 4.1 & 434.8 $\pm$ 3.5 \\
 & & 50k & 83.5 $\pm$ 5.4 & 38.8 $\pm$ 4.5 & 1180.6 $\pm$ 123.7 & 407.1 $\pm$ 20.3 & 5484.8 $\pm$ 371.2 & 3073.6 $\pm$ 185.5 & 1290.9 $\pm$ 5.0 & 2656.6 $\pm$ 8.3 & 2176.5 $\pm$ 9.0 \\
 & & 100k & 100.2 $\pm$ 5.6 & 43.5 $\pm$ 5.5 & 2360.2 $\pm$ 247.8 & 680.5 $\pm$ 20.9 & 7911.0 $\pm$ 395.9 & 5541.9 $\pm$ 205.0 & 2579.9 $\pm$ 8.1 & 5311.9 $\pm$ 13.6 & 4352.6 $\pm$ 13.7 \\
 & & 200k & 115.2 $\pm$ 7.5 & 46.8 $\pm$ 6.4 & 4717.8 $\pm$ 492.6 & 1190.8 $\pm$ 20.8 & 12638.8 $\pm$ 424.3 & 10127.5 $\pm$ 206.9 & 5155.7 $\pm$ 11.9 & 10622.1 $\pm$ 16.4 & 8703.8 $\pm$ 18.7 \\
\midrule
Qwen3-0.6B & GSM8K & 10k & 148.9 $\pm$ 4.4 & 51.6 $\pm$ 4.1 & 161.7 $\pm$ 22.6 & 235.9 $\pm$ 16.0 & 1272.3 $\pm$ 56.5 & 367.5 $\pm$ 32.6 & 10.8 $\pm$ 1.7 & 225.2 $\pm$ 2.8 & 18.6 $\pm$ 0.5 \\
 & & 50k & 243.6 $\pm$ 9.7 & 77.3 $\pm$ 5.2 & 796.3 $\pm$ 117.2 & 645.9 $\pm$ 50.8 & 2951.3 $\pm$ 93.1 & 1217.7 $\pm$ 86.3 & 68.5 $\pm$ 4.7 & 1115.9 $\pm$ 5.6 & 93.0 $\pm$ 1.6 \\
 & & 100k & 281.0 $\pm$ 13.3 & 92.2 $\pm$ 6.4 & 1583.5 $\pm$ 235.0 & 1053.5 $\pm$ 81.6 & 4689.9 $\pm$ 144.5 & 2174.1 $\pm$ 138.4 & 144.4 $\pm$ 6.2 & 2227.4 $\pm$ 5.5 & 184.7 $\pm$ 1.8 \\
 & & 200k & 317.6 $\pm$ 13.1 & 113.6 $\pm$ 8.6 & 3153.0 $\pm$ 471.6 & 1755.3 $\pm$ 130.5 & 7439.0 $\pm$ 290.3 & 4039.5 $\pm$ 220.3 & 297.8 $\pm$ 6.7 & 4453.4 $\pm$ 9.6 & 369.6 $\pm$ 3.1 \\
\addlinespace
Qwen3-0.6B & NuminaMath & 10k & 188.2 $\pm$ 5.3 & 80.1 $\pm$ 4.5 & 140.7 $\pm$ 56.9 & 429.8 $\pm$ 17.2 & 2638.2 $\pm$ 21.5 & 2297.8 $\pm$ 31.9 & 54.4 $\pm$ 1.4 & 840.0 $\pm$ 19.2 & 104.1 $\pm$ 1.5 \\
 & & 50k & 372.5 $\pm$ 10.8 & 130.1 $\pm$ 8.3 & 657.6 $\pm$ 284.0 & 1600.3 $\pm$ 75.5 & 12100.9 $\pm$ 62.0 & 10916.9 $\pm$ 95.9 & 265.6 $\pm$ 2.5 & 4212.3 $\pm$ 49.2 & 518.8 $\pm$ 3.4 \\
 & & 100k & 451.3 $\pm$ 14.7 & 154.9 $\pm$ 9.5 & 1302.7 $\pm$ 568.8 & 2901.5 $\pm$ 121.0 & 23607.2 $\pm$ 76.8 & 21467.3 $\pm$ 128.5 & 529.2 $\pm$ 2.5 & 8430.3 $\pm$ 72.0 & 1039.3 $\pm$ 4.3 \\
 & & 200k & 533.0 $\pm$ 16.7 & 179.5 $\pm$ 8.9 & 2583.8 $\pm$ 1132.9 & 5352.0 $\pm$ 140.2 & 45998.7 $\pm$ 119.2 & 42326.4 $\pm$ 161.1 & 1052.1 $\pm$ 3.6 & 16824.7 $\pm$ 83.6 & 2075.5 $\pm$ 4.8 \\
\addlinespace
Qwen3-0.6B & MedMCQA & 10k & 33.0 $\pm$ 2.2 & 31.8 $\pm$ 3.1 & 254.1 $\pm$ 52.3 & 75.7 $\pm$ 5.8 & 2177.1 $\pm$ 14.4 & 2176.2 $\pm$ 14.5 & 151.1 $\pm$ 2.2 & 1042.4 $\pm$ 3.3 & 407.0 $\pm$ 3.8 \\
 & & 50k & 77.3 $\pm$ 7.8 & 58.6 $\pm$ 5.7 & 1251.6 $\pm$ 263.9 & 194.0 $\pm$ 23.8 & 10903.5 $\pm$ 32.8 & 10902.6 $\pm$ 33.0 & 746.2 $\pm$ 5.1 & 5214.8 $\pm$ 7.0 & 2028.9 $\pm$ 8.6 \\
 & & 100k & 100.0 $\pm$ 7.3 & 71.9 $\pm$ 7.9 & 2500.1 $\pm$ 527.2 & 310.2 $\pm$ 37.4 & 21829.0 $\pm$ 43.5 & 21828.1 $\pm$ 43.8 & 1499.8 $\pm$ 6.2 & 10436.3 $\pm$ 10.7 & 4068.7 $\pm$ 11.0 \\
 & & 200k & 137.5 $\pm$ 8.0 & 88.3 $\pm$ 9.6 & 4986.9 $\pm$ 1060.1 & 511.4 $\pm$ 60.3 & 43671.7 $\pm$ 54.5 & 43670.7 $\pm$ 54.7 & 3004.1 $\pm$ 9.2 & 20877.6 $\pm$ 15.3 & 8144.1 $\pm$ 14.6 \\
\midrule
Qwen2.5-0.5B & GSM8K & 10k & 179.5 $\pm$ 4.3 & 72.3 $\pm$ 3.8 & 181.7 $\pm$ 55.0 & 339.6 $\pm$ 25.5 & 2576.5 $\pm$ 21.3 & 2008.0 $\pm$ 37.7 & 111.4 $\pm$ 1.6 & 334.3 $\pm$ 13.2 & 141.2 $\pm$ 2.3 \\
 & & 50k & 288.3 $\pm$ 13.4 & 108.2 $\pm$ 9.1 & 867.6 $\pm$ 270.1 & 1113.6 $\pm$ 70.8 & 11065.3 $\pm$ 71.7 & 9117.7 $\pm$ 159.8 & 552.6 $\pm$ 4.0 & 1660.7 $\pm$ 32.3 & 708.4 $\pm$ 4.6 \\
 & & 100k & 324.3 $\pm$ 15.4 & 122.4 $\pm$ 10.3 & 1727.4 $\pm$ 539.4 & 1949.2 $\pm$ 104.2 & 20651.6 $\pm$ 189.5 & 17814.5 $\pm$ 290.7 & 1108.4 $\pm$ 6.0 & 3339.7 $\pm$ 40.0 & 1422.7 $\pm$ 5.7 \\
 & & 200k & 351.4 $\pm$ 16.3 & 137.1 $\pm$ 11.0 & 3442.9 $\pm$ 1078.1 & 3515.2 $\pm$ 144.5 & 38398.9 $\pm$ 281.5 & 35172.1 $\pm$ 529.0 & 2218.8 $\pm$ 7.5 & 6670.6 $\pm$ 57.6 & 2850.4 $\pm$ 7.4 \\
\addlinespace
Qwen2.5-0.5B & NuminaMath & 10k & 156.4 $\pm$ 6.3 & 62.3 $\pm$ 3.5 & 111.8 $\pm$ 34.8 & 292.8 $\pm$ 11.0 & 2574.9 $\pm$ 18.1 & 2321.4 $\pm$ 20.7 & 10.2 $\pm$ 0.2 & 992.4 $\pm$ 9.1 & 58.6 $\pm$ 1.1 \\
 & & 50k & 251.7 $\pm$ 9.8 & 96.0 $\pm$ 4.8 & 497.9 $\pm$ 178.8 & 1081.2 $\pm$ 33.7 & 11909.0 $\pm$ 44.4 & 11218.1 $\pm$ 64.4 & 51.2 $\pm$ 0.4 & 4919.7 $\pm$ 37.2 & 293.1 $\pm$ 2.9 \\
 & & 100k & 301.2 $\pm$ 9.5 & 115.3 $\pm$ 6.9 & 983.2 $\pm$ 357.0 & 1956.1 $\pm$ 47.5 & 23287.7 $\pm$ 72.8 & 22275.3 $\pm$ 117.0 & 102.3 $\pm$ 0.5 & 9813.7 $\pm$ 47.3 & 589.1 $\pm$ 4.1 \\
 & & 200k & 363.9 $\pm$ 10.4 & 138.4 $\pm$ 9.6 & 1954.2 $\pm$ 715.0 & 3648.0 $\pm$ 66.5 & 45885.3 $\pm$ 140.9 & 44312.7 $\pm$ 181.2 & 204.3 $\pm$ 0.8 & 19570.7 $\pm$ 60.9 & 1177.4 $\pm$ 5.0 \\
\addlinespace
Qwen2.5-0.5B & MedMCQA & 10k & 23.9 $\pm$ 2.3 & 21.0 $\pm$ 3.2 & 256.9 $\pm$ 34.9 & 51.9 $\pm$ 4.3 & 2096.5 $\pm$ 13.5 & 2096.1 $\pm$ 13.5 & 133.7 $\pm$ 1.4 & 1002.7 $\pm$ 4.0 & 99.8 $\pm$ 0.9 \\
 & & 50k & 65.2 $\pm$ 3.2 & 37.5 $\pm$ 4.9 & 1230.6 $\pm$ 176.7 & 119.6 $\pm$ 11.9 & 10502.6 $\pm$ 32.6 & 10501.6 $\pm$ 32.8 & 679.3 $\pm$ 2.8 & 5012.3 $\pm$ 10.4 & 500.6 $\pm$ 2.1 \\
 & & 100k & 104.3 $\pm$ 3.6 & 49.5 $\pm$ 5.8 & 2444.1 $\pm$ 351.8 & 178.9 $\pm$ 18.1 & 21033.0 $\pm$ 41.9 & 21031.6 $\pm$ 42.3 & 1362.7 $\pm$ 3.8 & 10027.7 $\pm$ 15.4 & 1001.6 $\pm$ 3.4 \\
 & & 200k & 160.5 $\pm$ 6.2 & 63.2 $\pm$ 6.9 & 4870.1 $\pm$ 706.8 & 275.3 $\pm$ 27.7 & 42084.4 $\pm$ 54.6 & 42082.6 $\pm$ 55.1 & 2732.4 $\pm$ 6.3 & 20040.8 $\pm$ 19.8 & 2008.3 $\pm$ 6.4 \\
\bottomrule
\end{tabular}}
\end{table}

\subsection{Service rate and queue-unaware algorithms}
\label{app:service-rate}

We rerun the Qwen3-4B MedMCQA instance with service rates $\mu\in\{.2,.35,.5,.7,1\}$, $T=5\cdot10^4$ and $24$ runs, with \texttt{FCT} and \texttt{FQT} retuned for each $\mu$ (Figure~\ref{fig:rw-mu} and Table~\ref{tab:rw-mu}). At $\mu=1$ the queue stays empty and $M_t(0)=c_{\rm exp}$. The regret of queue-aware and queue-unaware algorithms then coincides pairwise, \texttt{KL-PBD+} with \texttt{KL-UCB+-4P} ($41.9$) and \texttt{NBD} with \texttt{FCT} and \texttt{FQT} ($2{,}511$), while \texttt{KL-PBD} and \texttt{KL-UCB-4P} differ in regret only due to the difference in their confidence budgets ($113$ against $158$). For $\mu<1$, \texttt{KL-UCB-4P} and \texttt{KL-UCB+-4P} keep deferring regardless of the queue: at $\mu=.2$ it is full in $41$ to $56\%$ of the rounds, and their regret exceeds $200$ times that of \texttt{KL-PBD+}. The regret of \texttt{KL-PBD} and \texttt{KL-PBD+} varies by less than a factor $2$ across $\mu$. The regret of \texttt{NBD} grows with $\mu$, since a faster expert makes deferral optimal at more levels, which the surrogate $h=p$ misses. \texttt{FCT} and \texttt{FQT} follow the same trend up to $\mu=.7$, with an additional cost from their restricted form, which vanishes at $\mu=1$: the optimal rule under $h=p$ is then a fixed confidence cutoff, which both can represent, so they recover the decisions of \texttt{NBD}.

\begin{figure}[htb!]
  \centering
  \includegraphics[width=\textwidth]{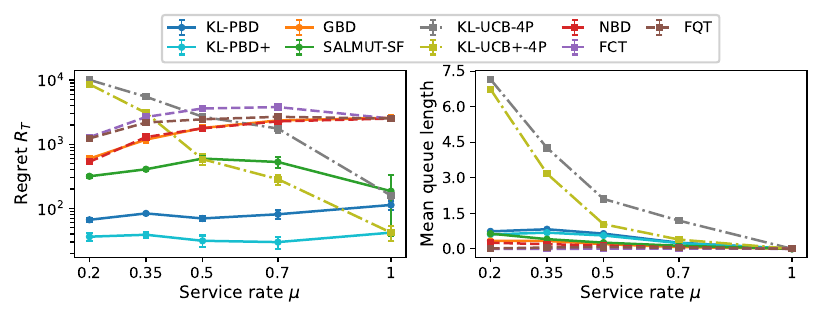}
  \caption{Service-rate sweep, $T=5\cdot10^4$, 24 runs: final regret (left) and mean queue length (right).}
  \label{fig:rw-mu}
\end{figure}

\begin{table}[htb!]
\centering\scriptsize
\caption{Final regret of the service-rate sweep, mean $\pm$ 95\% confidence half-width over 24 runs.}
\label{tab:rw-mu}
\begin{tabular}{lrrrrr}
\toprule
Algorithm & $\mu=.2$ & $\mu=.35$ & $\mu=.5$ & $\mu=.7$ & $\mu=1$ \\
\midrule
\texttt{KL-PBD} & 66.4 $\pm$ 3.9 & 83.5 $\pm$ 5.4 & 69.7 $\pm$ 7.0 & 80.9 $\pm$ 12.2 & 113.0 $\pm$ 17.2 \\
\texttt{KL-PBD+} & 36.2 $\pm$ 4.9 & 38.8 $\pm$ 4.5 & 31.2 $\pm$ 6.5 & 29.8 $\pm$ 6.4 & 41.9 $\pm$ 10.9 \\
\texttt{GBD} & 595.6 $\pm$ 57.9 & 1180.6 $\pm$ 123.7 & 1788.5 $\pm$ 207.5 & 2347.8 $\pm$ 249.4 & 2584.3 $\pm$ 259.7 \\
\texttt{SALMUT-SF} & 316.4 $\pm$ 13.8 & 407.1 $\pm$ 20.3 & 597.1 $\pm$ 58.6 & 528.1 $\pm$ 100.4 & 185.4 $\pm$ 148.1 \\
\texttt{KL-UCB-4P} & 10060.2 $\pm$ 228.5 & 5484.8 $\pm$ 371.2 & 2679.8 $\pm$ 302.3 & 1754.2 $\pm$ 252.2 & 158.2 $\pm$ 16.9 \\
\texttt{KL-UCB+-4P} & 8605.3 $\pm$ 110.2 & 3073.6 $\pm$ 185.5 & 584.5 $\pm$ 109.6 & 285.3 $\pm$ 50.8 & 41.9 $\pm$ 10.9 \\
\texttt{NBD} & 531.8 $\pm$ 3.7 & 1290.9 $\pm$ 5.0 & 1767.2 $\pm$ 5.5 & 2268.8 $\pm$ 6.6 & 2510.5 $\pm$ 7.2 \\
\texttt{FCT} & 1285.3 $\pm$ 4.7 & 2656.6 $\pm$ 8.3 & 3604.9 $\pm$ 11.1 & 3793.1 $\pm$ 13.9 & 2510.6 $\pm$ 7.2 \\
\texttt{FQT} & 1230.9 $\pm$ 4.7 & 2176.5 $\pm$ 9.0 & 2444.1 $\pm$ 8.7 & 2671.8 $\pm$ 9.1 & 2510.6 $\pm$ 7.2 \\
\bottomrule
\end{tabular}
\end{table}

\clearpage
\subsection{Learned threshold and decisions}
\label{app:policy-diagnostics}
The top row of Figure~\ref{fig:policy-maps} shows the threshold and decisions learned by \texttt{KL-PBD} for Qwen3-0.6B at $t=T/2$, from one run for each $\mu$, in the format of Figure~\ref{fig:policy-2by3}: the threshold $\widehat\tau_t(b)=1-M^{\underline h_t}_t(b)/c_e$ (top) and the decisions $\mathbf 1\{\underline h_{j,t}<\widehat\tau_t(b)\}$ for $\mu=0.35$ (bottom). The learned thresholds are within $0.09$ of the oracle ones for every $\mu$. The learned decisions differ from the oracle's due to extra deferrals caused by \texttt{KL-PBD}'s pessimism. For example, at $b=2$ on MedMCQA, \texttt{KL-PBD} defers queries with scores up to $30\%$ whereas the oracle accepts those with scores higher than $5\%$. The difference is larger for Qwen3-4B (middle and bottom rows of Figure~\ref{fig:policy-maps}). Since this LLM is more accurate, it is deferred less, so by $t=T/2$ \texttt{KL-PBD} has collected fewer labels ($2{,}649$ on GSM8K against $9{,}901$ for Qwen3-0.6B). Its lower confidence bounds are consequently looser, and it keeps deferring levels at intermediate queue sizes where the oracle no longer does ($b=3$ to $6$ on GSM8K, $b=2$ to $5$ on MedMCQA). On GSM8K, the levels still deferred at $b=7$ are those with at most $20$ labels, which produces the isolated cells in the bottom row.
\begin{figure}[htb!]
  \centering
  \includegraphics[width=\textwidth]{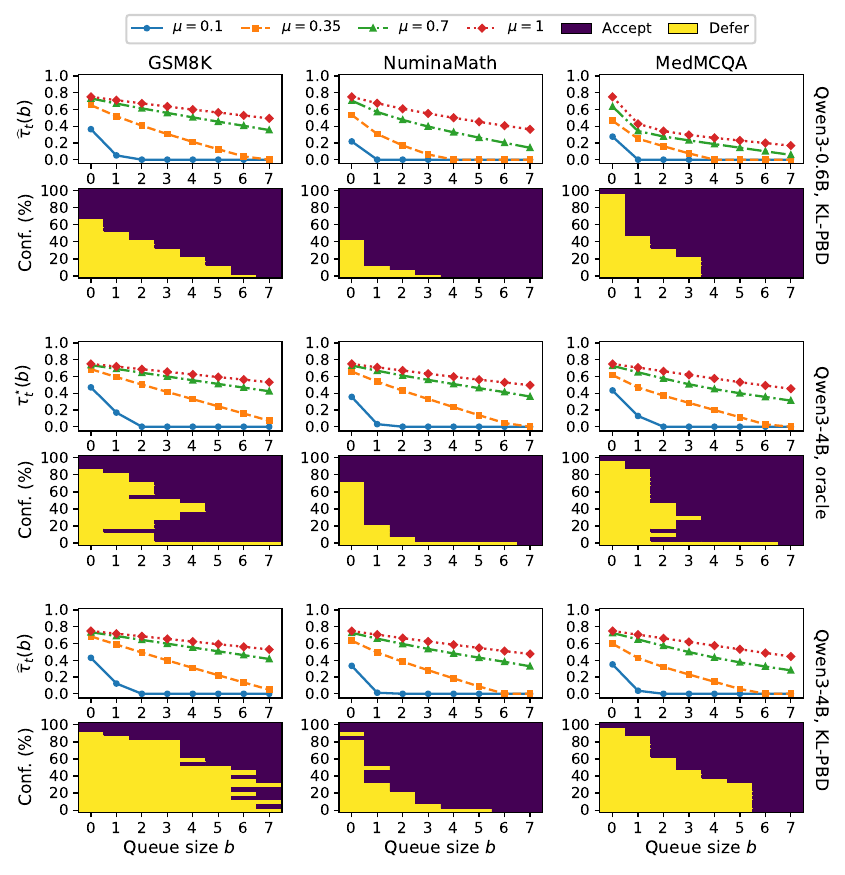}
  \caption{Thresholds for varying $\mu$ and decisions for $\mu=0.35$ at $t=T/2$: \texttt{KL-PBD} for Qwen3-0.6B (top), oracle (middle) and \texttt{KL-PBD} (bottom) for Qwen3-4B, one run per $\mu$.}
  \label{fig:policy-maps}
\end{figure}

\clearpage
\subsection{The importance of pessimism}
\label{app:gbd-pessimism}

\texttt{GBD} uses the same decision rule as \texttt{KL-PBD}, except that it uses the empirical mean $\widehat h_{j,N_{j,t}}$ instead of the lower confidence bound $\underline h_{j,t}$.
Figure~\ref{fig:rw-lockin} shows, for every instance, the final regret of each of the $100$ runs against the run's average deferral rate. For \texttt{GBD} the two are strongly negatively correlated, with correlation $\rho\in[-.95,-.85]$ across the nine instances: runs that defer less accumulate more regret. This is the selective-feedback failure of greedy learning: an early optimistic estimate of the conditional accuracy of a level stops its deferrals, hence no more feedback for this level is received, and the estimate is never corrected. On Qwen2.5-0.5B NuminaMath, for instance, the deferral rate of \texttt{GBD} ranges from $.15$ to $.31$ and its regret from $11$ to $10{,}658$. \texttt{KL-PBD+}, which keeps deferring a level while its lower confidence bound is below the current threshold, shows neither the spread nor the correlation.

\begin{figure}[htb!]
  \centering
  \includegraphics[width=\textwidth]{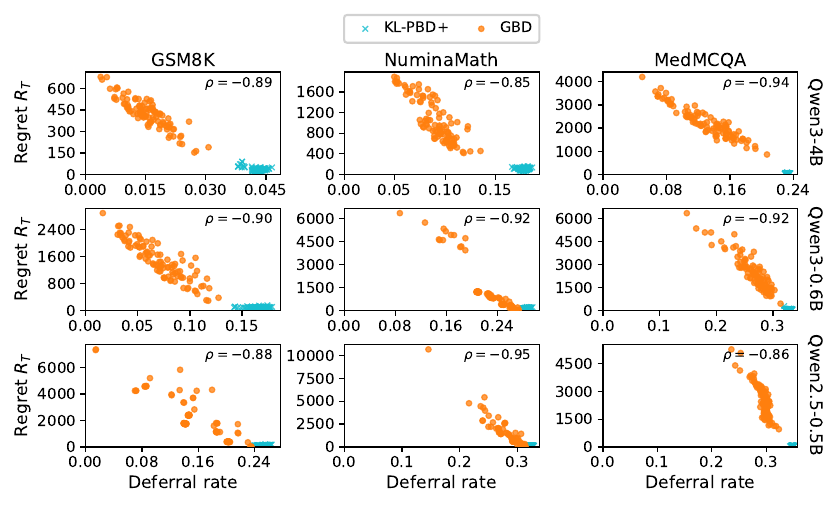}
  \caption{Final regret against average deferral rate, one point per run, $T=10^5$, and $\rho$ the correlation for \texttt{GBD}.}
  \label{fig:rw-lockin}
\end{figure}

\clearpage

\end{document}